\documentclass[runningheads]{llncs}
\usepackage{graphicx}
\usepackage[width=122mm,left=15mm,paperwidth=152mm,height=193mm,top=16mm,paperheight=225mm]{geometry}
\usepackage{tikz}
\usepackage{comment} 
\usepackage{amsmath,amssymb} % define this before the line numbering.
\usepackage{color}

\usepackage{booktabs}
\usepackage{multirow}
\usepackage{tabularx}
\usepackage{microtype}
\usepackage[normalem]{ulem}
\usepackage{ dsfont }
\usepackage{xcolor}
\usepackage{ulem}
\usepackage{arydshln}

\usepackage[shortlabels]{enumitem}
\usepackage{algorithmic}
\usepackage{algorithm}

\usepackage[pagebackref=true,breaklinks=true,colorlinks,bookmarks=false]{hyperref}
\long\def\ignorethis#1{}

\newcommand{\whitetxt}[1]{{\color{white}#1}\normalfont}

\newbox\jsavebox
\newcommand{\jsubfig}[2]{%
	\sbox\jsavebox{#1}%
	\parbox[t]{\wd\jsavebox}{\centering\usebox\jsavebox\\#2}%
	}

\usepackage[percent]{overpic}

\newcommand{\score}{g_\theta}
\newcommand{\pdata}{P_S}
\newcommand{\qdata}{Q_{\sigma,S}}
\newcommand{\qnoise}{q_{\sigma}}

\newcommand{\etal}{\textit{et al.}}

\begin{document}
\pagestyle{headings}
\mainmatter
\def\ECCVSubNumber{462}  % Insert your submission number here

\title{Learning Gradient Fields for Shape Generation} % Replace with your title

\begin{comment}
\titlerunning{ECCV-20 submission ID \ECCVSubNumber} 
\authorrunning{ECCV-20 submission ID \ECCVSubNumber} 
\author{Anonymous ECCV submission}
\institute{Paper ID \ECCVSubNumber}
\end{comment}
%******************

\titlerunning{Learning Gradient Fields for Shape Generation}
\makeatletter
\newcommand{\printfnsymbol}[1]{%
  \textsuperscript{\@fnsymbol{#1}}%
}
\makeatother
\author{
Ruojin Cai\thanks{Equal contribution.} \and
Guandao Yang\printfnsymbol{1}\and % todo
Hadar Averbuch-Elor \and
Zekun Hao \and \\
Serge Belongie \and
Noah Snavely \and
Bharath Hariharan
}

\authorrunning{Cai et al.}

\institute{
Cornell University
}
\maketitle
\begin{figure}
    \begin{center}
        \includegraphics[width=\linewidth, trim={3cm 2.5cm 3cm 1cm}, clip]{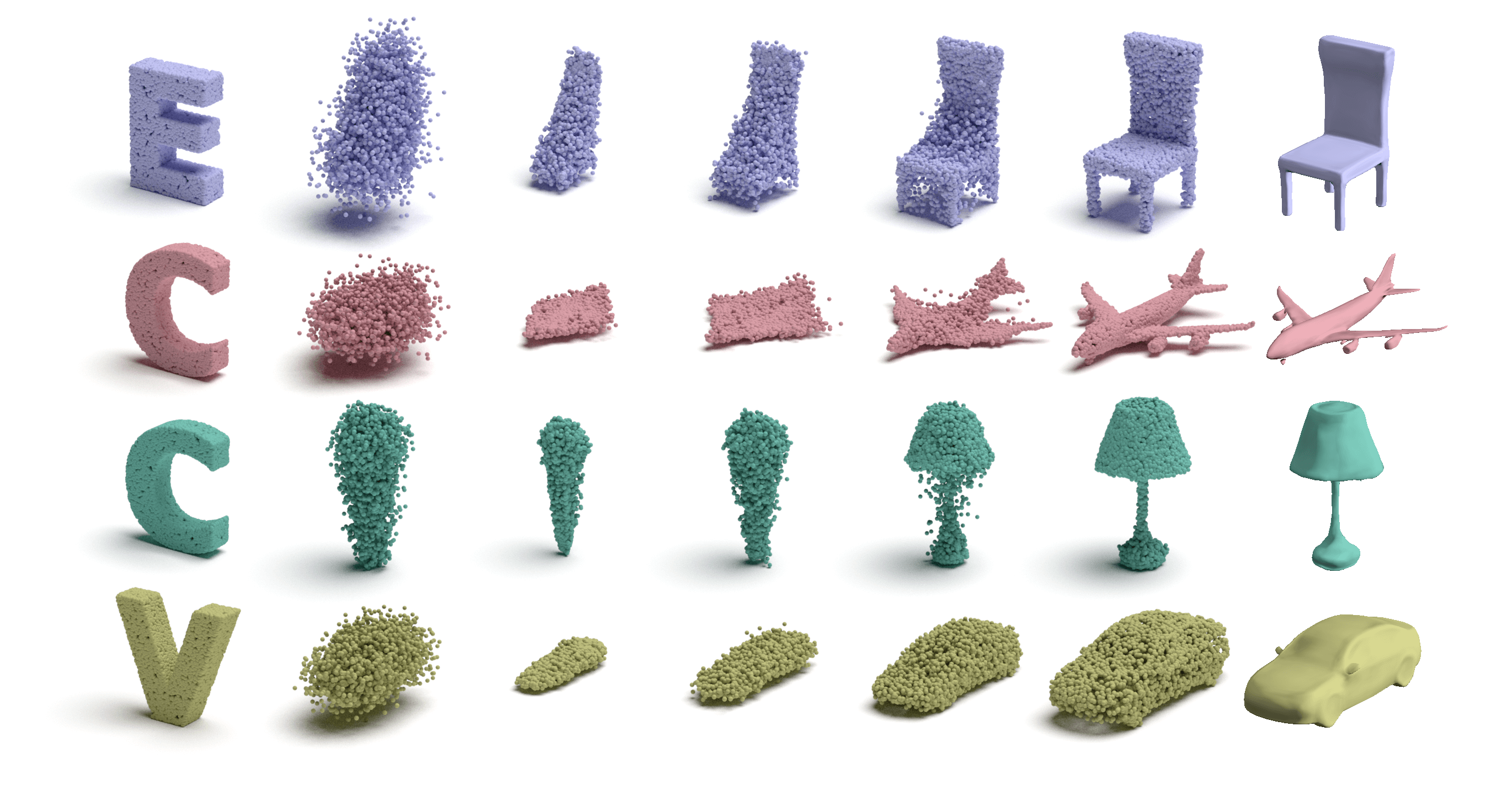}
    \end{center}
    \caption{To generate shapes, we sample points from an arbitrary prior (depicting the letters ``E", ``C", ``C", ``V" in the examples above) and move them stochastically along a learned gradient field, ultimately reaching the shape's surface. Our learned fields also enable extracting the surface of the shape, as demonstrated on the right. 
    }
    \label{fig:teaser}
\end{figure}
\vspace{-2em}
\begin{abstract}
In this work, we propose a novel technique to generate shapes from point cloud data.
A point cloud can be viewed as samples from a distribution of 3D points whose density is concentrated near the surface of the shape.
Point cloud generation thus amounts to moving randomly sampled points to high-density areas.
We generate point clouds by performing stochastic gradient ascent on an unnormalized probability density, thereby moving sampled points toward the high-likelihood regions.
Our model directly predicts the gradient of the log density field and can be trained with a simple objective adapted from score-based generative models.
We show that our method can reach state-of-the-art performance for point cloud auto-encoding and generation, while also allowing for extraction of a high-quality implicit surface.
Code is available at \url{https://github.com/RuojinCai/ShapeGF}.
\end{abstract}

\section{Introduction}\label{sec:intro}

Point clouds are becoming increasingly popular for modeling shapes, as many modern 3D scanning devices process and output point clouds.
As such, an increasing number of applications rely on the recognition, manipulation, and synthesis of point clouds.
For example, an autonomous vehicle might need to detect cars in sparse LiDAR point clouds.
An augmented reality application might need to scan in the environment.
Artists may want to further manipulate scanned objects to create new objects and designs.
A \emph{prior} for point clouds would be useful for these applications as it can densify LiDAR clouds, create additional training data for recognition, complete scanned objects or synthesize new ones.
Such a prior requires a powerful generative model for point clouds.

In this work, we are interested in learning a generative model that can sample shapes represented as point clouds.
A key challenge here is that point clouds are sets of arbitrary size.
Prior work often generates a fixed number of points instead~\cite{achlioptas_L3DP,gadelha2018multiresolution,3D-AAE,shu20193d,pointsetgen}.
This number, however, may be insufficient for some applications and shapes, or too computationally expensive for others.
Instead, following recent works~\cite{li2018point,yang2019pointflow,sun2018pointgrow}, we consider a point cloud as a set of \emph{samples} from an underlying distribution of 3D points.
This new perspective not only allows one to generate an arbitrary number of points from a shape, but also makes it possible to model shapes with varying topologies.
However, it is not clear how to best parameterize such a distribution of points, and how to learn it using \emph{only} a limited number of sampled points.

Prior research has explored modeling the distribution of points that represent the shape using generative adversarial networks (GANs)~\cite{li2018point}, flow-based models~\cite{yang2019pointflow}, and autoregressive models~\cite{sun2018pointgrow}. 
While substantial progress has been made, these methods have some inherent limitations for modeling the distribution representing a 3D shape. 
The training procedure can be unstable for GANs or prohibitively slow for invertible models, while autoregressive models assume an ordering, restricting their flexibility for point cloud generation. 
Implicit representations such as DeepSDF~\cite{park2019deepsdf} and OccupancyNet~\cite{mescheder2019occupancy} can be viewed as modeling this probability density of the 3D points directly, but these models require ground truth signed distance fields or occupancy fields, which are difficult to obtain from point cloud data alone without corresponding meshes.

In this paper, we take a different approach and focus on the end goal -- being able to draw an arbitrary number of samples from the distribution of points.
Working backward from this goal, we observe that the sampling procedure can be viewed as moving points from a generic prior distribution towards high likelihood regions of the shape (i.e., the surface of the shape).
One way to achieve that is to move points gradually, following the gradient direction, which indicates where the density grows the most~\cite{welling2011bayesian}.
To perform such sampling, one only needs to model the gradient of log-density (known as the \textit{Stein score function}~\cite{liu2016kernelized}).
In this paper, we propose to model a shape by learning the gradient field of its log-density.
To learn such a gradient field from a set of sampled points from the shape, we build upon a \textit{denoising score matching} framework~\cite{hyvarinen2005estimation,song2019generative}.
Once we learn a model that outputs the gradient field, the sampling procedure can be done using a variant of stochastic gradient ascent (i.e. \emph{Langevin dynamics}~\cite{welling2011bayesian,song2019generative}).

Our method offers several advantages.
First, our model is trained using a simple $L_2$ loss between the predicted and a ``ground-truth'' gradient field estimated from the input point cloud.
This objective is much simpler to optimize than adversarial losses used in GAN-based techniques.
Second, because it models the gradient directly and does not need to produce a normalized distribution, it imposes minimal restrictions on the model architecture in comparison to flow-based or autoregressive models.
This allows us to leverage more expressive networks to model complicated distributions.
Because the partition function need not be estimated, our model is also much \emph{faster} to train.
Finally, our model is able to furnish an implicit surface of the shape, as shown in Figure~\ref{fig:teaser}, without requiring ground truth surfaces during training.
We demonstrate that our technique can achieve state-of-the-art performance in both point cloud auto-encoding and generation. 
Moreover, our method can retain the same performance when trained with much sparser point clouds.

Our key contributions can be summarized as follows:

\begin{itemize}
    \item We propose a novel point cloud generation method by extending score-based generative models to learn conditional distributions. 
    \item We propose a novel algorithm to extract high-quality implicit surfaces from the learned model without the supervision from ground truth meshes.
    \item We show that our model can achieve state-of-the-art performance for point cloud auto-encoding and generation.
\end{itemize}

\section{Related work}
\paragraph{Point cloud generative modeling.}
Point clouds are widely used for representing and generating 3D shapes due to their simplicity and direct relation to common data acquisition techniques (LiDARs, depth cameras, etc.). 
Earlier generative models either treat point clouds as a fixed-dimensional matrix (i.e. $N\times3$ where $N$ is predefined)~\cite{achlioptas_L3DP,gadelha2018multiresolution,zamorski2018adversarial,sun2018pointgrow,3D-AAE,pointsetgen,shu20193d,valsesia2018learning}, or relies on heuristic set distance functions such as Chamfer distance and Earth Mover Distance~\cite{atlasnet,foldingnet,mrt18,deprelle2019learning,ben2018multi}.
As pointed out in Yang~\etal~\cite{yang2019pointflow} and Section~\ref{sec:intro}, both of these approaches lead to several drawbacks.
Alternatively, we can model the point cloud as samples from a distribution of 3D points.
Toward this end, Sun~\etal~\cite{sun2018pointgrow} applies an autoregressive model to model the distribution of points, but it requires assuming an ordering while generating points.
Li~\etal~\cite{li2018point} applies a GAN~\cite{WGAN,iWGAN} on both this distribution of 3D points as well as the distribution of shapes.
PointFlow~\cite{yang2019pointflow} applies normalizing flow~\cite{papamakarios2019normalizing} to model such distribution, so sampling points amounts to moving them to the surface according to a learned vector field.
In addition to modeling the movement of points, PointFlow also tracks the change of volume in order to normalize the learned distribution, which is computationally expensive~\cite{neuralODE}.
While our work applies a GAN to learn the distribution of latent code similar to Li~\etal~and Achilioptas~\etal, we take a different approach to model the distribution of 3D points.
Specifically, we predict the gradient of log-density field to model the non-normalized probability density, thus circumventing the need to compute the partition function and achieves faster training time with a simple L2 loss.

\paragraph{Generating other 3D representations.}
Common representations emerged for deep generative 3D modeling include voxel-based \cite{girdhar2016learning,wu2016learning}, mesh-based~\cite{anguelov2005scape,pons2015dyna,gao2019sdm,hanocka2019meshcnn,litany2018deformable,tan2018variational}, and assembly-based techniques~\cite{li2017grass,mo2019structurenet}.
Recently, implicit representations are gaining increasing popularity, as they are capable of representing shapes with high level of detail~\cite{park2019deepsdf,mescheder2019occupancy,chen2019learning,michalkiewicz2019deep}. 
They also allow for learning a structured decomposition of shapes, representing local regions with Gaussian functions \cite{genova2019deep,genova2019learning} or other primitives \cite{tulsiani2017learning,smirnov2019deep,hao2020dualsdf}.   
In order to reconstruct the mesh surface from the learned implicit field, these methods require finding the zero iso-surface of the learned occupancy field (e.g. using the Marching Cubes algorithm \cite{lorensen1987marching}). 
Our learned gradient field also allows for high-quality surface reconstruction using similar methods. 
However, we do not require prior information on the shape (e.g., signed distance values) for training, which typically requires a watertight input mesh.
Recently, SAL~\cite{atzmon2020sal} learns a signed distance field using only point cloud as supervision.
Different from SAL, our model directly outputs the gradients of the log-density instead field of the signed distance, which allows our model to use arbitrary network architecture without any constraints.
As a result, our method can scale to more difficult settings such as train on larger dataset (e.g. ShapeNet~\cite{shapenet}) or train with sparse scanned point clouds.

\paragraph{Energy-based modeling.}
In contrast to flow-based models \cite{NormalizingFlow,Nice,GLOW,neuralODE,ffjord,yang2019pointflow} and auto-regressive models \cite{sun2018pointgrow,pixelCNN,PixelRNN,oord2016wavenet}, 
energy-based models learn a non-normalized probability distribution~\cite{lecun2006tutorial}, thus avoid computation to estimate the partition function. 
It has been successfully applied to tasks such as image segmentation \cite{fan2007mcmc,fan2008mcmc}, where a normalized probability density function is hard to define.
Score matching was first proposed for modeling energy-based models in \cite{hyvarinen2005estimation} and 
deals with ``matching" the model and the observed data log-density gradients, by minimizing the squared distance between them. 
To improve its performance and scalability, various extensions have been proposed, including denoising score matching \cite{vincent2011connection} and sliced score matching \cite{song2019sliced}. 
Most recently, Song and Ermon \cite{song2019generative} introduced data perturbation and annealed Langevin dynamics to the original denoising score matching method, providing an effective way to model data embedded on a low dimensional manifold. 
Their method was applied to the image generation task, achieving performance comparable to GANs. In this work, we extend this method to model conditional distributions and demonstrate its suitability to the task of point cloud generation, viewing point clouds as samples from the 2D manifold (shape surface) in 3D space.

\section{Method}\label{sec:method}

In this work, we are interested in learning a generative model that can sample shapes represented as point clouds.
Therefore, we need to model two distributions.
First, we need to model the distribution of shapes, which encode how shapes vary across an entire collection of shapes.
Once we can sample a particular shape of interest, then we need a mechanism to sample a point clouds from its surface.
As previously discussed, a point cloud is best viewed as samples from a distribution of 3D (or 2D) points, which encode a particular shape.
To sample point clouds of arbitrary size for this shape, we also need to model this distribution of points.

Specifically, we assume a set of shapes $\mathcal{X} = \{X^{(i)}\}^N_{i=1}$ are provided as input.
Each shape in $\mathcal{X}$ is represented as a point cloud sampled from its surface, defined by $X^{(i)} = \{x_j^{i}\}_{j=1}^{M_i}$.
Our goal is to learn both the distribution of shapes and the distribution of points, conditioned on a particular shape from the data.
To achieve that, we first propose a model to learn the distribution of points encoding a shape from a set of points $X^{(i)}$
(Section~\ref{sec:single-shape-prob} - \ref{sec:implicit-surface}).
Then we describe how to model the distribution of shapes from the set of point clouds (i.e. $\mathcal{X}$) in Section~\ref{sec:multiple-shapes}.

\subsection{Shapes as a distribution of 3D points}\label{sec:single-shape-prob}

We would like to define a distribution of 3D points $P(x)$ such that sampling from this distribution will provide us with a surface point cloud of the object.
Thus, the probability density encoding the shape should concentrate on the shape surface.
Let $S$ be the set of points on the surface and $\pdata(x)$ be the uniform distribution over the surface.
Sampling from $\pdata(x)$ will create a point cloud uniformly sampled from the surface of interest.
However, this distribution is hard to work with: for all points that are not in the surface $x\notin S$, $\pdata(x)=0$. As a result, $\pdata(x)$ is discontinuous and has usually zero support over its ambient space (i.e. $\mathbb{R}^3$), which impose challenges in learning and modeling.
Instead, we approximate $\pdata(x)$ by smoothing the distribution with a Gaussian kernel:
\begin{equation}
    \qdata(x) = \int_{s \in \mathbb{R}^3} \pdata(s)\mathcal{N}(x; s, \sigma^2I)ds .\label{eq:shape_density}
\end{equation}
As long as the standard deviation $\sigma$ is small enough, $\qdata(x)$ will approximate the true data distribution $\pdata(x)$ whose density concentrates near the surface.
Therefore, sampling from $\qdata(x)$ will yield points near the surface $S$.

As discussed in Section~\ref{sec:intro}, instead of modeling $\qdata$ directly, we will model the gradient of the logarithmic density (i.e. $\nabla_x \log{\qdata(x)}$).
Sampling can then be performed by starting from a prior distribution and performing gradient ascent on this field, thus moving points to high probability regions.

In particular, we will model the gradient of the log-density using a neural network $\score(x, \sigma)$, where $x$ is a location in 3D (or 2D) space.
We will first analyze several properties of this gradient field $\nabla_x \log{\qdata(x)}$.
Then we describe how we train this neural network and how we sample points using the trained network.

\subsection{Analyzing the gradient field}\label{sec:interepting}

In this section we provide an interpretation of how $\nabla_x \log{\qdata(x)}$ behaves with different $\sigma$'s.
Computing a Monte Carlo approximation of $\qdata(x)$ using a set of observations $\{x_i\}_{i=1}^m$, we obtain a mixture of Gaussians with modes centered at $x_i$ and radially-symmetric kernels:
\begin{align*}
    \qdata(x) &= \mathbb{E}_{s \sim \pdata}\left[\mathcal{N}(x; s, \sigma^2I)\right] 
    \approx \frac{1}{m}\sum_{i=1}^m \mathcal{N}(x; x_i, \sigma^2I) \triangleq A_\sigma(x, \{x_i\}_{i=1}^m).
\end{align*}
The gradient field can thus be approximated by the gradient of the logarithmic of this Gaussian mixture:
\begin{equation}
    \nabla_x \log{A_\sigma(x, \{x_i\}_{i=1}^m)} = \frac{1}{\sigma^2} \left(-x + \sum_{i=1}^m  x_i w_{i}(x, \sigma)\right),
\end{equation}
where the weight $w_{ij}(x, \sigma)$ is computed from a softmax with temperature $2\sigma^2$:
\begin{equation}
    w_{i}(x, \sigma) = \frac{\exp{\left(-\frac{1}{2\sigma^2}\|x - x_i\|^2\right)}}{\sum_{j=1}^m \exp{\left(-\frac{1}{2\sigma^2}\|x - x_j\|^2\right)}}. \label{eq:weight_ij}
\end{equation}

Since $\sum_i w_i(x, \sigma) = 1$, $\sum_i x_i w_i(x, \sigma)$ falls within the convex hull created by the sampled surface points $\{x_i\}_{i=1}^m$.
Therefore, the direction of this gradient of the logarithmic density field points from the sampled location towards a point inside the convex hull of the shape.
When the temperature is high (i.e. $\sigma$ is large), then the weights $w_{i}(x, \sigma)$ will be roughly the same and $\sum_i x_i w_i(x,\sigma)$ behaves like averaging all the $x_i$'s.
Therefore, the gradient field will point to a coarse shape that resembles an average of the surface points.
When the temperature is low (i.e. $\sigma$ is small), then $w_i(x,\sigma)$ will be close to $0$ except when $x_i$ is the closest to $x$.
As a result, $\sum_i x_i w_i(x,\sigma)$ will behave like an $\operatorname{argmin}_{x_i} \|x_i - x\|$.
The gradient direction will thus point to the nearest point on the surface.
In this case, the norm of the gradient field approximates a distance field of the surface up to a constant $\sigma^{-2}$.
This allows the gradient field to encode fine details of the shape and move points to the shape surface more precisely.
Figure~\ref{fig:inference_mnist_samples} shows a visualization of the field in the 2D case for a series of different $\sigma$'s.

\subsection{Training objective}\label{sec:loss}

As mentioned in Section~\ref{sec:single-shape-prob}, we would like to train a deep neural network $\score(x, \sigma)$ to model the gradient of log-density: $\nabla_x \log{\qdata(x)}$.
One simple objective achieving this is minimizing the L2 loss between them~\cite{hyvarinen2005estimation}:
\begin{equation}
    \ell_\mathsf{direct}(\sigma, S) =  
    \mathbb{E}_{x \sim \qdata(x)}
    \left[\frac{1}{2}\left\|
    \score(x,\sigma) - \nabla_x \log{\qdata(x)}
    \right\|_2^2
    \right].\label{eq:loss-direct}
\end{equation}
%%%%%%%%%%%%%%%%
However, optimizing such an objective is difficult as it is generally hard to compute $\nabla_x \log \qdata(x)$ from a finite set of observations.

Inspired by \textit{denoising score matching} methods~\cite{vincent2011connection,song2019generative}, we can write $\qdata(x)$ as a perturbation of the data distribution $\pdata(x)$, produced with a Gaussian noise with standard deviation $\sigma$.
Specifically, $\qdata (x) = \int \pdata(s) \qnoise(\tilde{x}|s) dx$, where $\qnoise(\tilde{x}|s)=\mathcal{N}(\tilde{x}|s,\sigma^2I)$.
As such, optimizing the objective in Equation~\ref{eq:loss-direct} can be shown to be equivalent to optimizing the following~\cite{vincent2011connection}:
\begin{align}
    \ell_\mathsf{denoising}(\sigma,S) 
    = \mathbb{E}_{s\sim \pdata,\tilde{x}\sim \qnoise(\tilde{x}|s)}
    \left[\frac{1}{2} \|
    \score(\tilde{x}, \sigma) - \nabla_{\tilde{x}} \log{\qnoise(\tilde{x}|s)}
    \|_2^2\right]. \label{eq:loss}
\end{align}
Since $\nabla_{\tilde{x}} \log{\qnoise(\tilde{x}|s)} = \frac{s - \tilde{x}}{\sigma^2}$, this loss can be easily computed using the observed point cloud $X=\{x_j\}_{j=1}^m$ as following:
\begin{equation}
    \ell(\sigma, X) = 
    \frac{1}{|X|}\sum_{x_i \in X} 
    \|\score(\tilde{x_i}, \sigma) - \frac{x_i-\tilde{x}_i}{\sigma^2}\|_2^2,
    \quad \tilde{x_i} \sim \mathcal{N}(x_i, \sigma^2I).
\label{eq:lossfun}
\end{equation}

\noindent\textbf{Multiple noise levels.}
One problem with the abovementioned objective is that most $\tilde{x_i}$ will concentrate near the surface if $\sigma$ is small.
Thus, points far away from the surface will not be supervised. 
This can adversely affect the sampling quality, especially when the prior distribution puts points to be far away from the surface.
To alleviate this issue, we follow Song and Ermon~\cite{song2019generative} and train $\score$ for multiple $\sigma$'s, with $\sigma_1 \geq \dots \geq \sigma_k$.
Our final model is trained by jointly optimizing $\ell(\sigma_i, X)$ for all $\sigma_i$.
The final objective is computed empirically as:
\begin{equation}
    \mathcal{L}(\{\sigma_i\}_{i=1}^k, X) = \sum_{i=1}^k \lambda(\sigma_i) \ell(\sigma_i, X),
\end{equation}
where $\lambda(\sigma_i)$ are parameters weighing the losses $\ell(\sigma_i, X)$.
$\lambda(\sigma_i)$ is chosen so that the weighted losses roughly have the same magnitude during training.

\subsection{Point cloud sampling}\label{sec:sample-point-clouds}
Sampling a point cloud from the distribution is equivalent to moving points from a prior distribution to the surface (i.e. the high-density region).
Therefore, we can perform stochastic gradient ascent on the logarithmic density field.
Since $\score(x, \sigma)$ approximates the gradient of the log-density field (i.e. $\nabla_x \log{\qdata(x)}$), we could thus use $\score(x, \sigma)$ to update the point location $x$.
In order for the points to reach all the local maxima, we also need to inject random noise into this process.
This amounts to using Langevin dynamics to perform sampling~\cite{welling2011bayesian}.

Specifically, we first sample a point $x_0$ from a prior distribution $\pi$.
The prior is usually chosen to be simple distribution such as a uniform or a Gaussian distribution.
We empirically demonstrate that the sampling performance won't be affected as long as the prior points are sampled from places where the perturbed points would reach during training.
We then perform the following recursive update with step size $\alpha > 0$:
\begin{equation}
    x_{t+1} = x_t + \frac{\alpha}{2} \score(x_t, \sigma) + \sqrt{\alpha} \epsilon_t,\quad \epsilon_t \sim \mathcal{N}(0, I).\label{eq:langevin-dynamic}
\end{equation}
Under mild conditions, $p(x_T)$ converges to the data distribution $\qdata(x)$ as $T \to \infty$ and $\epsilon \to 0$~\cite{welling2011bayesian}.
Even when such conditions fail to hold, the error in Equation~\ref{eq:langevin-dynamic} is usually negligible when $\alpha$ is small and $T$ is large~\cite{song2019generative,chen2014stochastic,du2019implicit,nijkamp2019anatomy}.

\begin{figure}[t]
\begin{center}
\includegraphics[width=\textwidth]{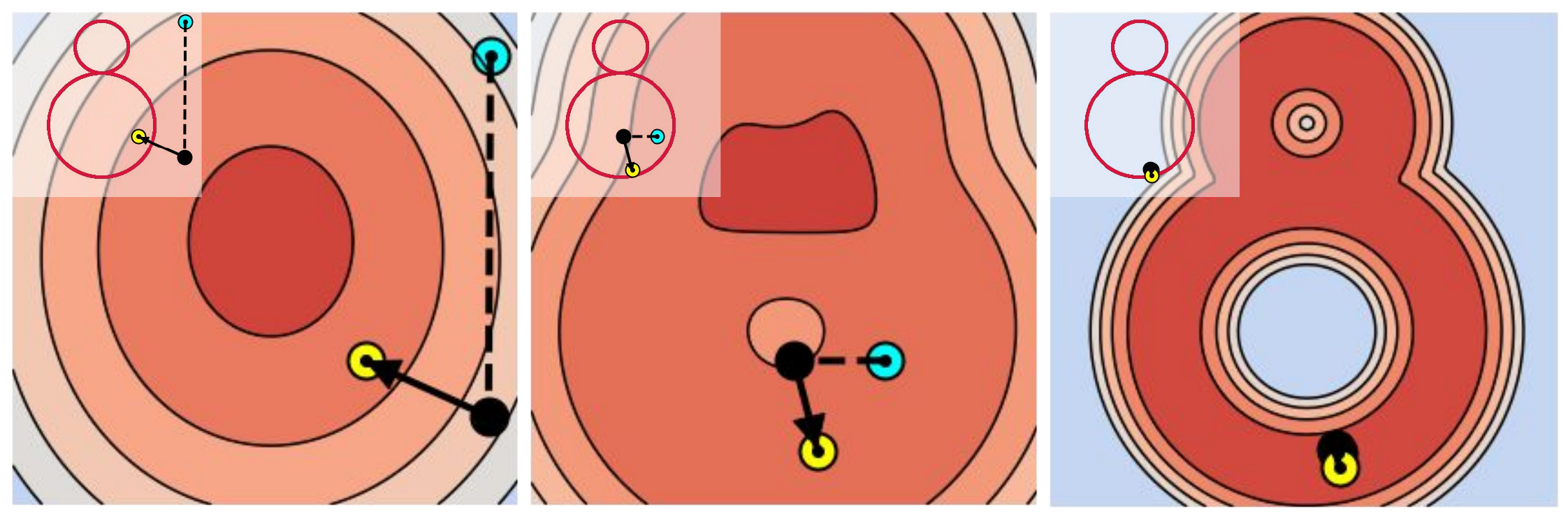}
\end{center}
    \caption{
    Log density field with different $\sigma$ (biggest to smallest) and a Langevin Dynamic point update step with that $\sigma$.
    Deeper color indicates higher density. 
    The ground truth shape is shown in the upper left corner. 
    Dotted line indicated Gaussian noise and solid arrows indicates gradient step.
    As sigma decreases, the log-density field changes from coarse to fine, and points are moved closer to the surface.
    }
    \label{fig:inference_mnist_samples}
\end{figure}

Prior works have observed that a main challenge for using Langevin dynamics is its slow mixing time~\cite{song2019generative,wenliang2018learning}.
To alleviate this issue, Song and Ermon~\cite{song2019generative} propose an annealed version of Langevin dynamics, which gradually anneals the noise for the score function.
Specifically, we first define a list of $\sigma_i$ with $\sigma_1 \geq \dots \geq \sigma_k$, then train one single denoising score matching model that could approximate $q_{\sigma_i}$ for all $i$.
Then, annealed Langevin dynamics will recursively compute the $x_t$ while gradually decreasing $\sigma_i$:
\begin{align}
    x_{t+1}' &= x_t + \frac{\sqrt{\alpha}\sigma_i \epsilon_t}{\sigma_k}, \quad \epsilon_t \sim \mathcal{N}(0, I), \\
    x_{t+1} &= x_{t+1}' + \frac{\alpha\sigma_i^2}{2\sigma_k^2} \score(x_{t+1}', \sigma_i). ~\label{eq:annealed-langevin-dynamic}
\end{align}
Figure~\ref{fig:inference_mnist_samples} demonstrates the sampling across the annealing process in a 2D point cloud.
As discussed in Section~\ref{sec:loss}, larger $\sigma$'s correspond to coarse shapes while smaller $\sigma$'s correspond to fine shape. Thus, this annealed Langevin dynamics can be thought of as a coarse-to-fine refinement of the shape.
Note that we make the noise perturnbation step before the gradient update step, which leads to cleaner point clouds.
The supplementary material contains detailed hyperparameters.

\subsection{Implicit surface extraction} \label{sec:implicit-surface}
Next we show that our learned gradient field (e.g. $\score$) also allows for obtaining an implicit surface.  
The key insight here is that the surface is defined as the set of points that reach the maximum density in the data distribution $\pdata(x)$, and thus these points have \textit{zero} gradient.
Another interpretation is that when $\sigma$ is small enough (i.e. $\qdata(x)$ approximates the true data distribution $p(x)$), the gradient for points near the surface will point to its nearest point on the surface, as described in Section~\ref{sec:interepting}:
\begin{equation}
    \score(x, \sigma) \approx \frac{1}{\sigma^2}\left(-x + \operatorname{argmin}_{s\in S}\|x-s\|\right).
    \label{eq:approx}
\end{equation}
Thus, for a point near the surface, its norm equals \textit{zero} if and only if $x \in S$ (provided the $\arg\min$ is unique). 
Therefore, the shape can be approximated by the zero iso-surface of the gradient norm:
\begin{equation}
    S \approx \{x~|~\|\score(x, \sigma)\| = \delta\},
\end{equation}
for some $\delta > 0$ that is sufficiently small. One caveat is that points for which the $\arg\min$ in Equation~\ref{eq:approx} is not unique may also have a zero gradient.
These correspond to \emph{local minimas} of the likelihood. In practice, this is seldom a problem for surface extraction, and it is possible to discard these regions by conducting the second partial derivative test.

Also as mentioned in Section \ref{sec:interepting}, when the $\sigma$ is small, the norm of the gradient field approximates a distance field of the surface, scaled by a constant $\sigma^{-2}$. This allows us to retrieval the surface $S$ efficiently using an off-the-shelf ray-casting technique~\cite{roth1982ray}
(see Figures \ref{fig:teaser},\ref{fig:recon},\ref{fig:gen}).

\subsection{Generating multiple shapes} \label{sec:multiple-shapes}
\begin{figure}[t]
\begin{center}
	\includegraphics[width=\textwidth]{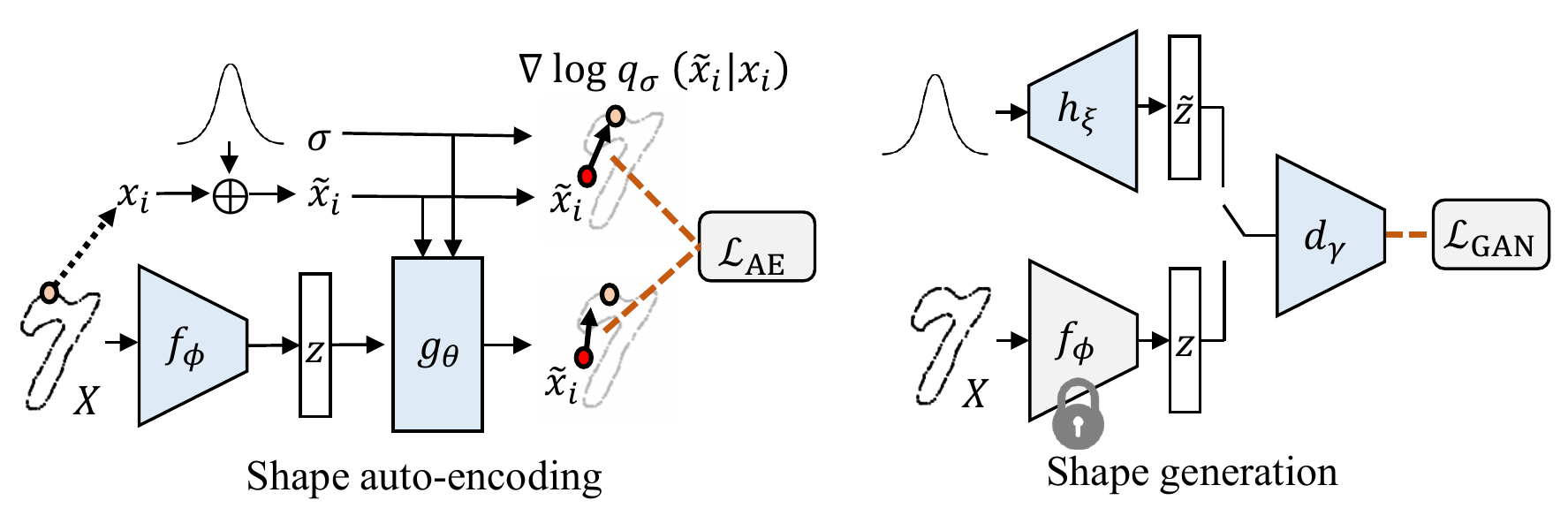}
\end{center}
\caption{
Illustration of training pipe for shape auto-encoding and generation.
 }
\label{fig:overview}
\end{figure}
In the previous sections, we focused on learning the distribution of points that represent a single shape. Our next goal is to model the distribution of shapes. 
We, therefore, introduce a latent code $z$ to encode which specific shape we want to sample point clouds from.
Furthermore, we adapt our gradient decoder to be conditional on the latent code $z$ (in addition to $\sigma$ and the sampled point).

As illustrated in Figure~\ref{fig:overview}, the training is conducted in two stages. 
We first train an auto-encoder with an encoder $f_\phi$ that takes a point cloud and outputs the latent code $z$.
The gradient decoder is provided with $z$ as input and produces a gradient field with noise level $\sigma$. 
The auto-encoding loss is thus:
\begin{equation}
    \mathcal{L}_{AE}(\mathcal{X}) = \mathop{\mathbb{E}}_{X\sim \mathcal{X}}
    \left[
    \frac{1}{2|X|}\sum_{x \in X, \sigma_i}\lambda(\sigma_i)
    \left\|\score(\tilde{x}, f_\phi(X), \sigma_i) - \frac{x - \tilde{x}}{\sigma_i^2}
    \right\|_2^2
    \right],
\end{equation}
where each $\tilde{x_j}$ is drawn from a $\mathcal{N}(x_j, \sigma_i^2I)$ for a corresponding $\sigma_i$.
This first stage provides us with a network that can model the distribution of points representing the shape encoded in the latent variable $z$.
Once the auto-encoder is fully trained, we apply a latent-GAN~\cite{achlioptas_L3DP} to learn the distribution of the latent code $p(z) = p(f_\phi(X))$, where $X$ is a point cloud sampled from the data distribution.
Doing so provides us with a generator $h_\xi$ that can sample a latent code from $p(z)$, allowing us control over which shape will be generated.
To sample a novel shape, we first sample a latent code $\tilde{z}$ using $h_\xi$.
We can then use the trained gradient decoder $\score$ to sample point clouds or extract an implicit surface from the shape represented as $z$.
For more details about hyperparameters and model architecture, please refer to the supplementary material.

\section{Experiments}

In this section, we will evaluate our model's performance in point cloud auto-encoding (Sec~\ref{sec:auto-encoding}), up-sampling (Sec~\ref{sec:robustness}), and generation (Sec~\ref{sec:shape-generaton}) tasks.
Finally, we present an ablation study examining our model design choices (Sec~\ref{sec:ablation}).
Implementation details will be shown in the supplementary materials.

\subsubsection{Datasets} 

Our experiments focus mainly on two datasets: MNIST-CP and ShapeNet. MNIST-CP was recently proposed by Yifan~\etal~\cite{yifan2019patch} and consists of 2D contour points extracted from the MNIST \cite{lecun2010mnist} dataset, which contains 50K and 10K training and testing images. 
Each point cloud contains 800 points.
The ShapeNet \cite{chang2015shapenet} dataset contains 35708 shapes in training set and 5158 shapes in test set, capturing 55 categories.
For our method, we normalize all point clouds by centering their bounding boxes to the origin and scaling them by a constant such that all points range within the cube $[-1, 1]^3$ (or the square in the 2D case).

% <left> <lower> <right> <upper>}
\begin{figure}[t]
    \begin{center}
    \newcommand{\sizea}{0.1665\linewidth}
    \newcommand{\tal}{5.5cm}
    \newcommand{\tab}{4.5cm}
    \newcommand{\tar}{4.5cm}
    \newcommand{\tat}{4.8cm}
    \newcommand{\tcl}{2.0cm}
    \newcommand{\tcb}{2.2cm}
    \newcommand{\tcr}{3.0cm}
    \newcommand{\tct}{4cm}
    \newcommand{\thl}{2.0cm}
    \newcommand{\thb}{1cm}
    \newcommand{\thr}{2.0cm}
    \newcommand{\tht}{2cm}
    \setlength{\tabcolsep}{0pt}
    \renewcommand{\arraystretch}{0}
    \begin{tabular}{@{}ccc:ccc@{}}
        \includegraphics[width=\sizea, trim={\tal} {\tab} {\tar} {\tat},clip]{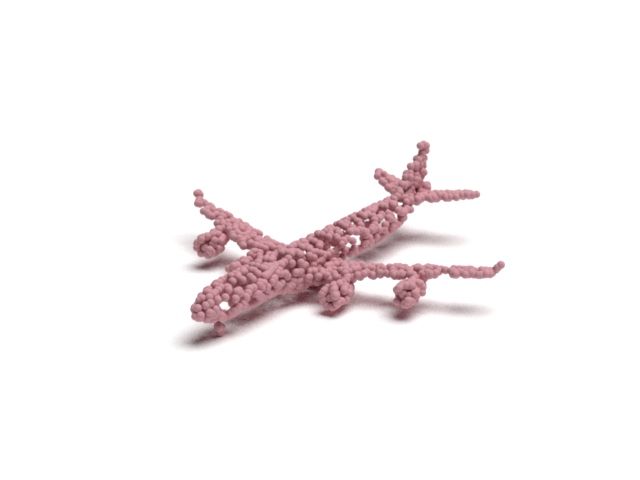} &
        \includegraphics[width=\sizea, trim={\tal} {\tab} {\tar} {\tat},clip]{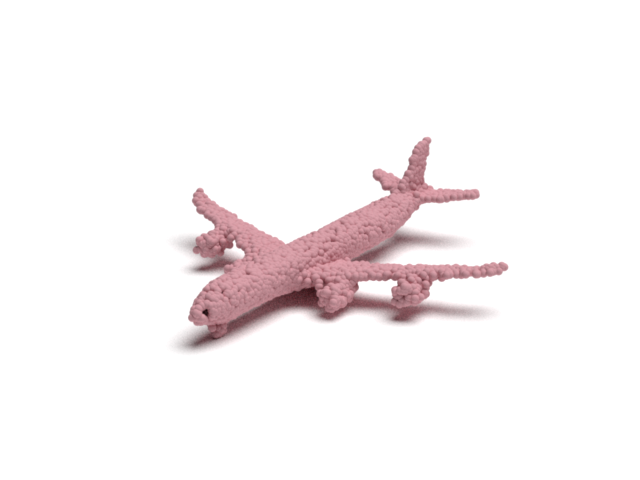} &
        \includegraphics[width=\sizea, trim={\tal} {\tab} {\tar} {\tat},clip]{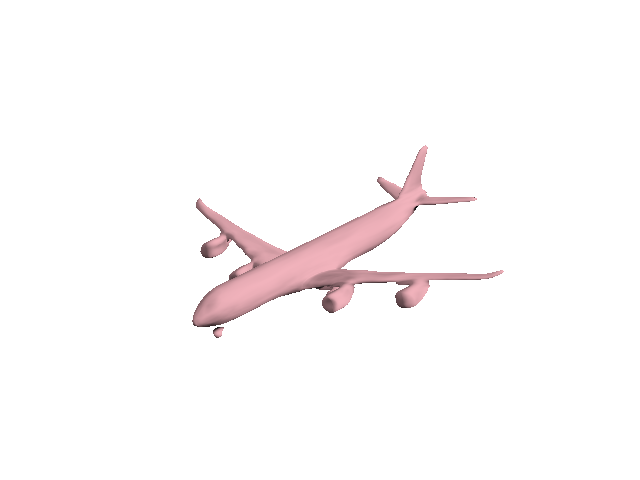} &
        \includegraphics[width=\sizea, trim={\tal} {\tab} {\tar} {\tat},clip]{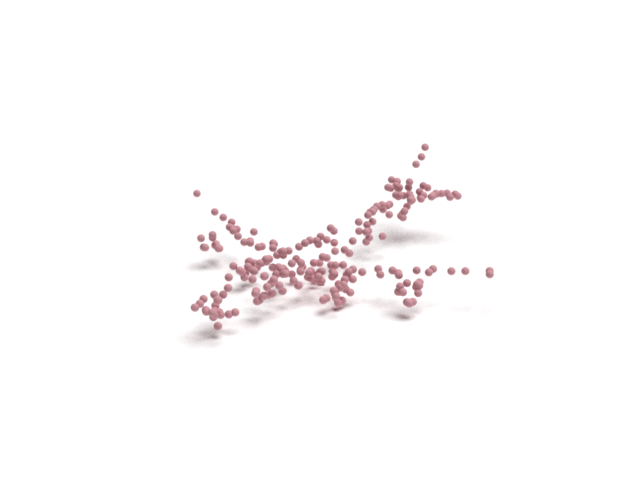} &
        \includegraphics[width=\sizea, trim={\tal} {\tab} {\tar} {\tat},clip]{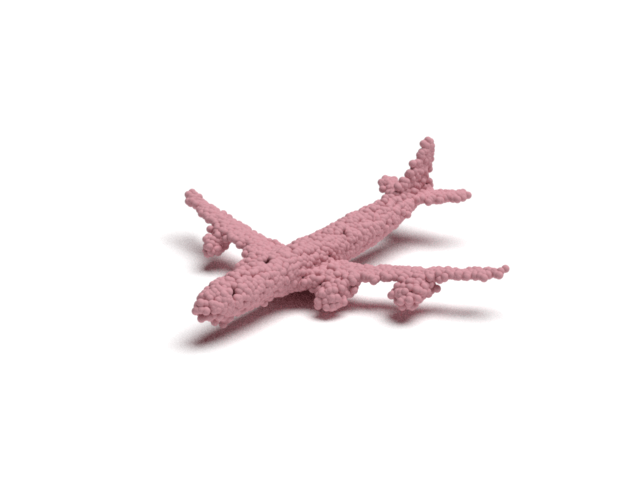} &
        \includegraphics[width=\sizea, trim={\tal} {\tab} {\tar} {\tat},clip]{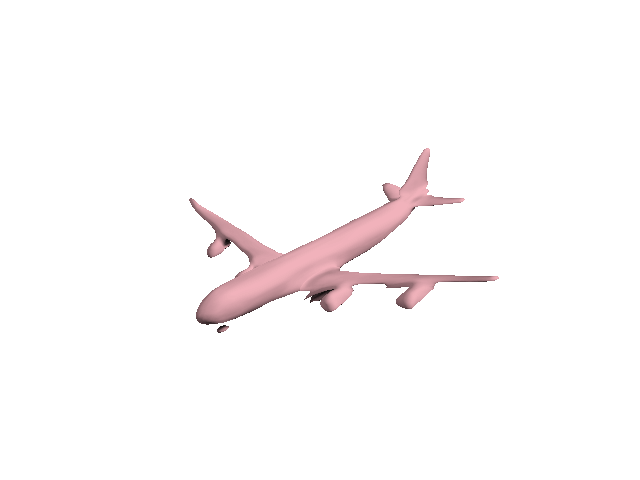} \\
        \includegraphics[width=\sizea, trim={\tal} {\tab} {\tar} {\tat},clip]{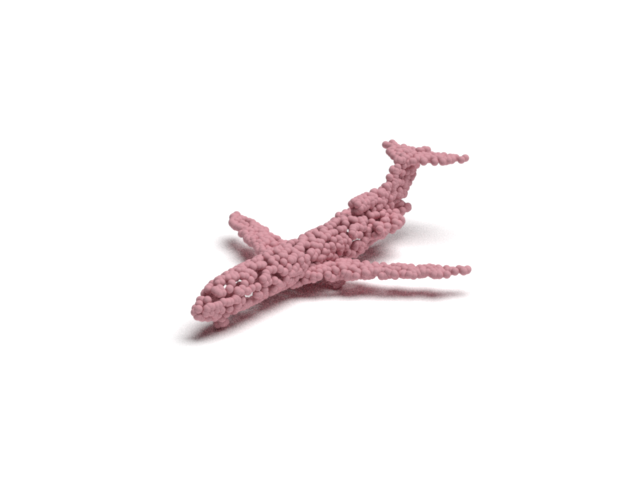} &
        \includegraphics[width=\sizea, trim={\tal} {\tab} {\tar} {\tat},clip]{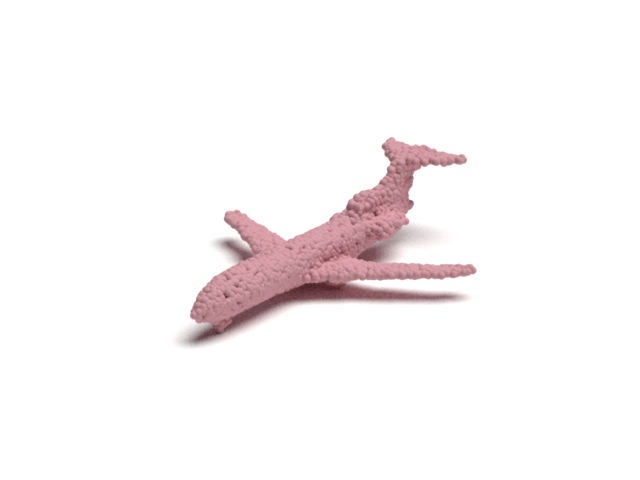} &
        \includegraphics[width=\sizea, trim={\tal} {\tab} {\tar} {\tat},clip]{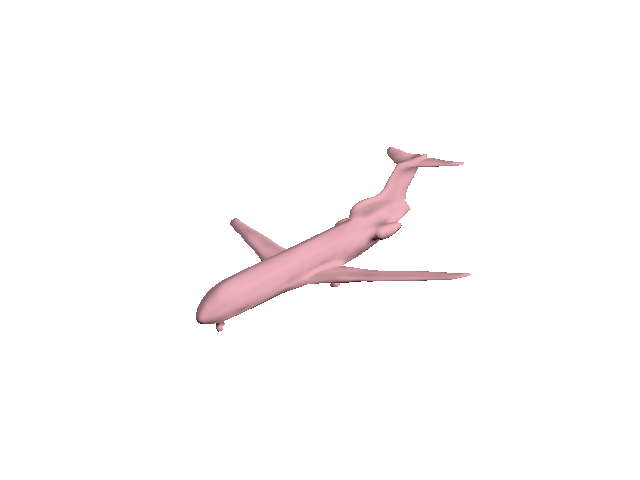} &
        \includegraphics[width=\sizea, trim={\tal} {\tab} {\tar} {\tat},clip]{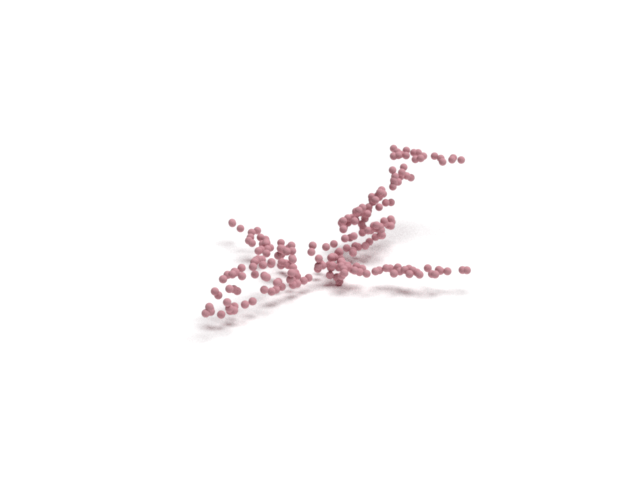} &
        \includegraphics[width=\sizea, trim={\tal} {\tab} {\tar} {\tat},clip]{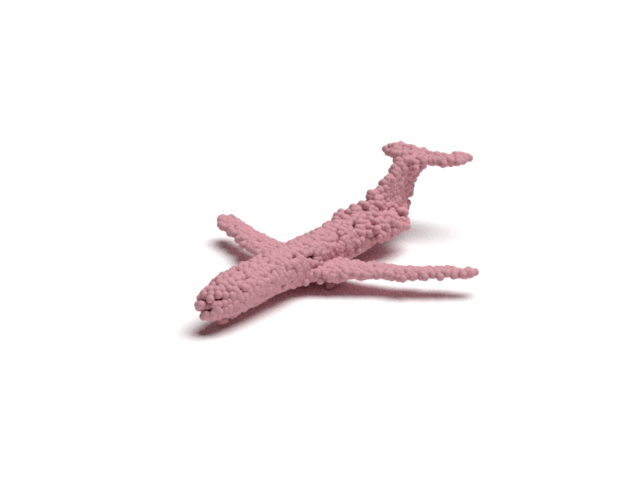} &
        \includegraphics[width=\sizea, trim={\tal} {\tab} {\tar} {\tat},clip]{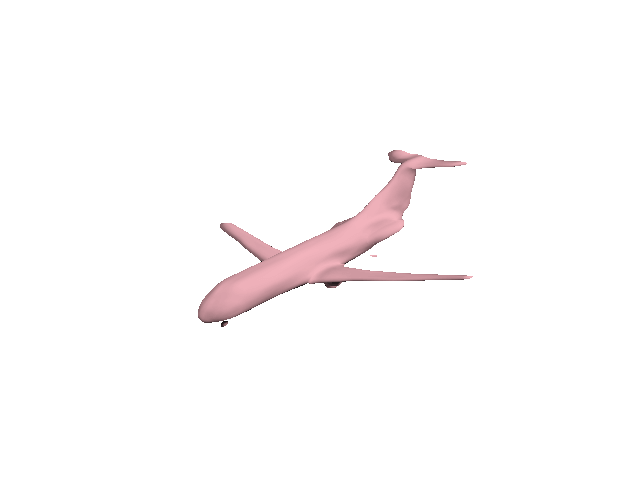} \\
        
        \includegraphics[width=\sizea, trim={\tcl} {\tcb} {\tcr} {\tct},clip]{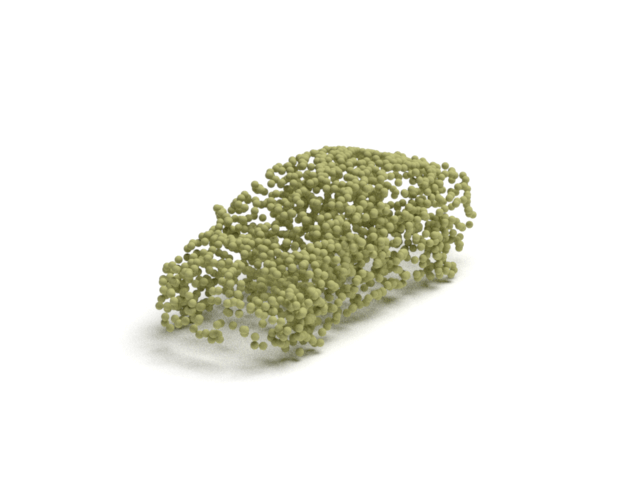} &
        \includegraphics[width=\sizea, trim={\tcl} {\tcb} {\tcr} {\tct},clip]{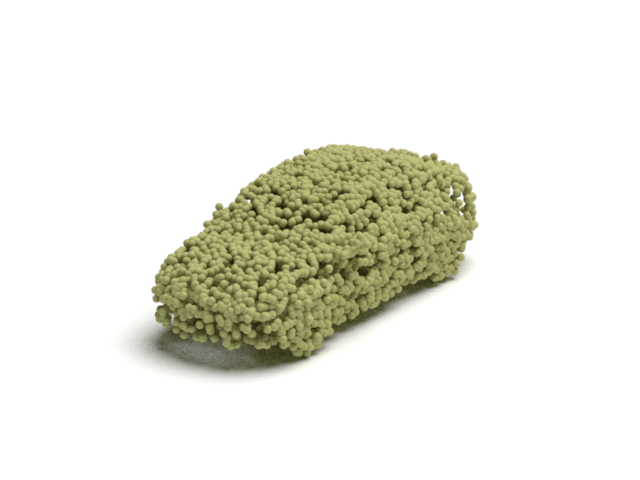} &
        \includegraphics[width=\sizea, trim={\tcl} {\tcb} {\tcr} {\tct},clip]{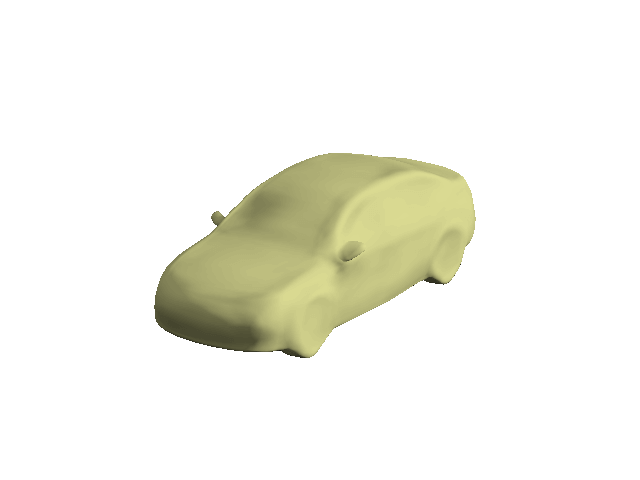} &
        \includegraphics[width=\sizea, trim={\tcl} {\tcb} {\tcr} {\tct},clip]{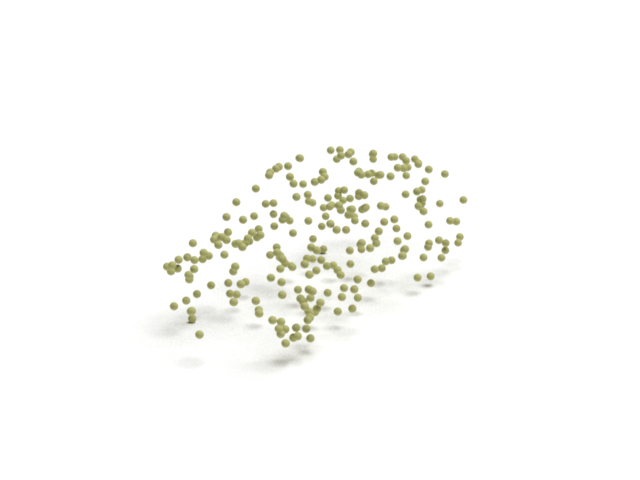} &
        \includegraphics[width=\sizea, trim={\tcl} {\tcb} {\tcr} {\tct},clip]{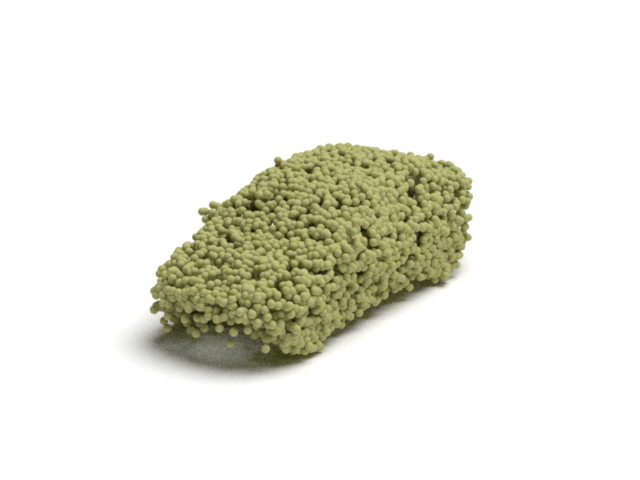} &
        \includegraphics[width=\sizea, trim={\tcl} {\tcb} {\tcr} {\tct},clip]{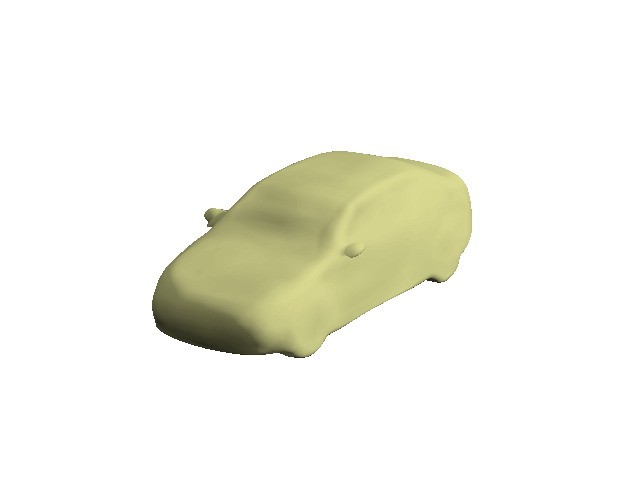} \\
        \includegraphics[width=\sizea, trim={\tcl} {\tcb} {\tcr} {\tct},clip]{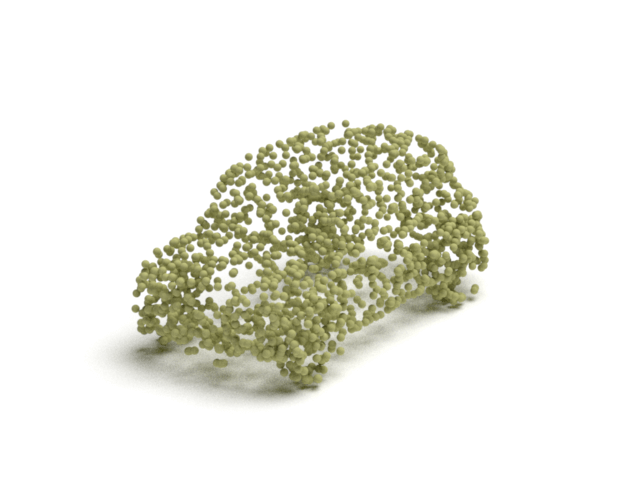} &
        \includegraphics[width=\sizea, trim={\tcl} {\tcb} {\tcr} {\tct},clip]{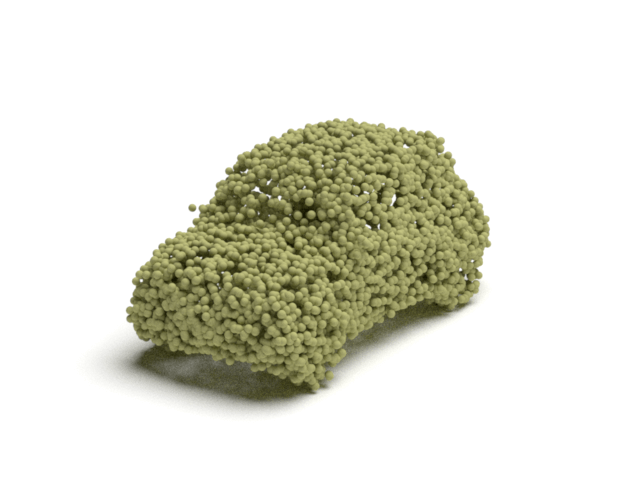} &
        \includegraphics[width=\sizea, trim={\tcl} {\tcb} {\tcr} {\tct},clip]{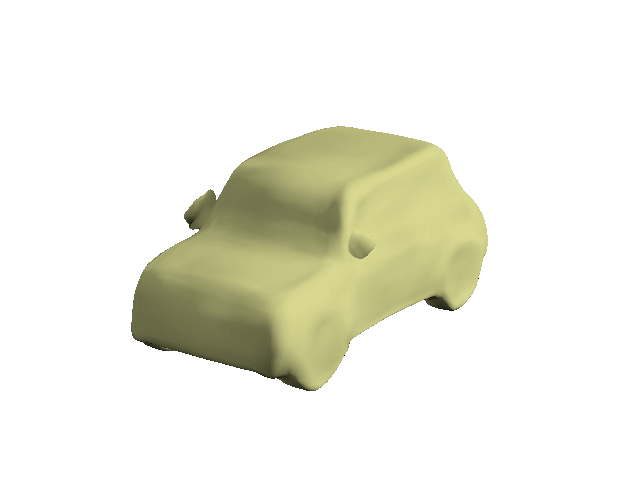} &
        \includegraphics[width=\sizea, trim={\tcl} {\tcb} {\tcr} {\tct},clip]{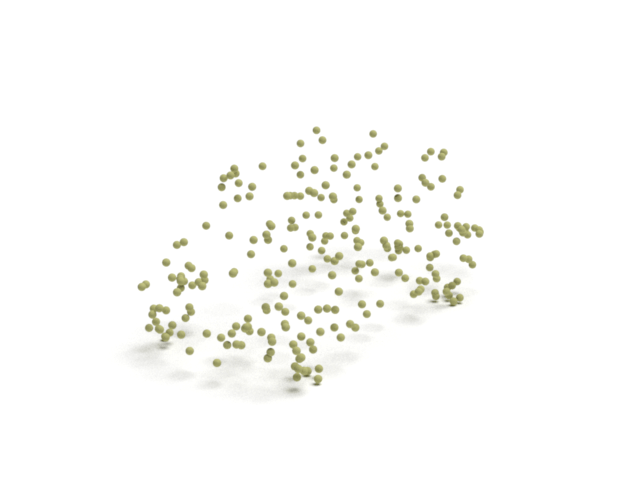} &
        \includegraphics[width=\sizea, trim={\tcl} {\tcb} {\tcr} {\tct},clip]{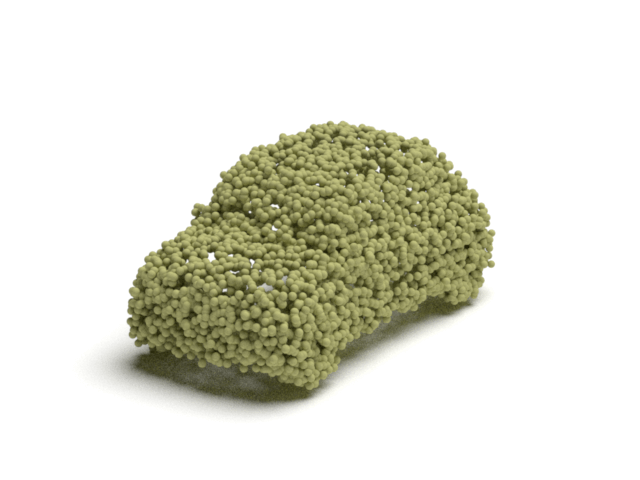} &
        \includegraphics[width=\sizea, trim={\tcl} {\tcb} {\tcr} {\tct},clip]{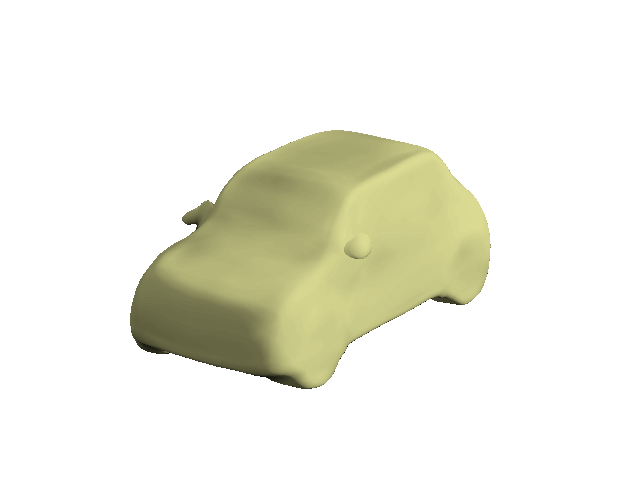} \\
        
        \includegraphics[width=\sizea, trim={\thl} {\thb} {\thr} {\tht},clip]{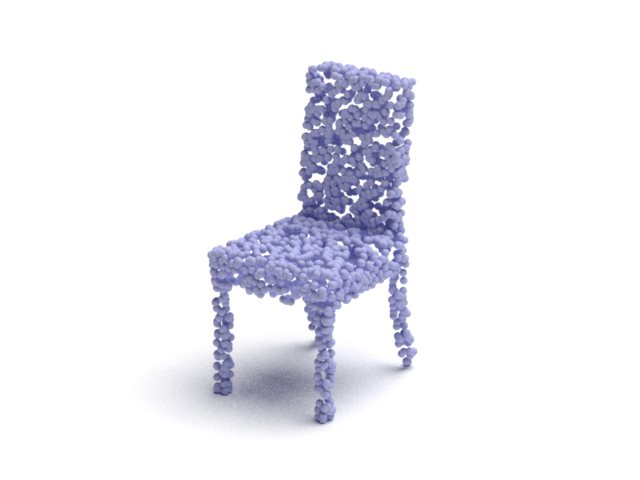} &
        \includegraphics[width=\sizea, trim={\thl} {\thb} {\thr} {\tht},clip]{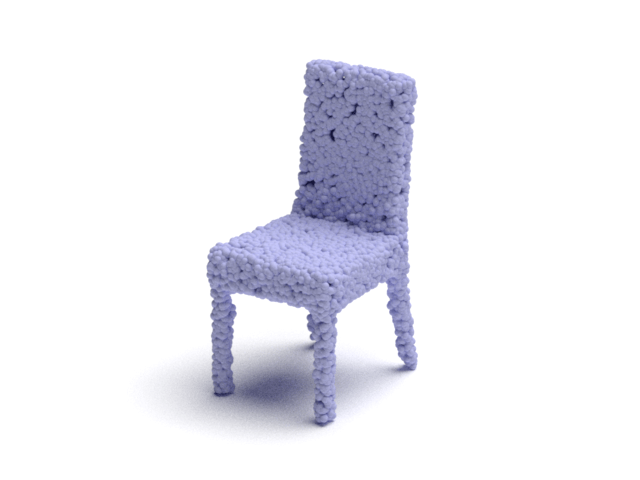} &
        \includegraphics[width=\sizea, trim={\thl} {\thb} {\thr} {\tht},clip]{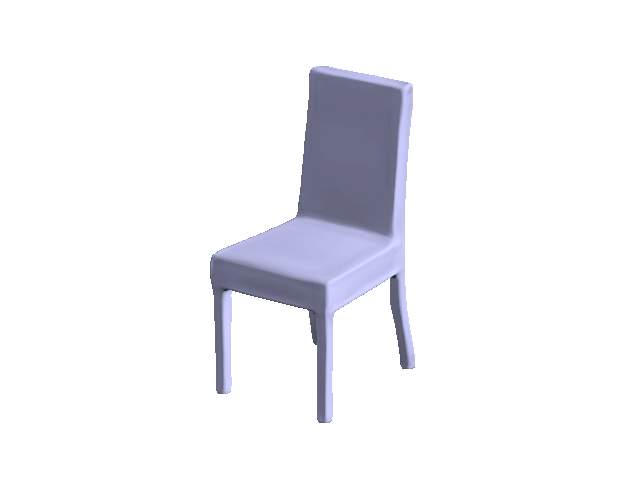} &
        \includegraphics[width=\sizea, trim={\thl} {\thb} {\thr} {\tht},clip]{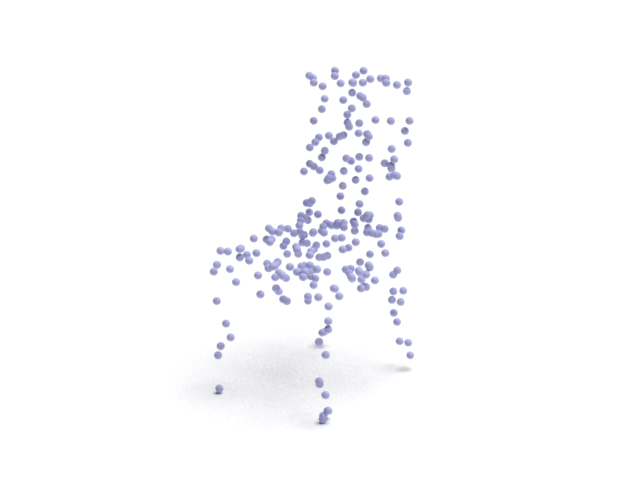} &
        \includegraphics[width=\sizea, trim={\thl} {\thb} {\thr} {\tht},clip]{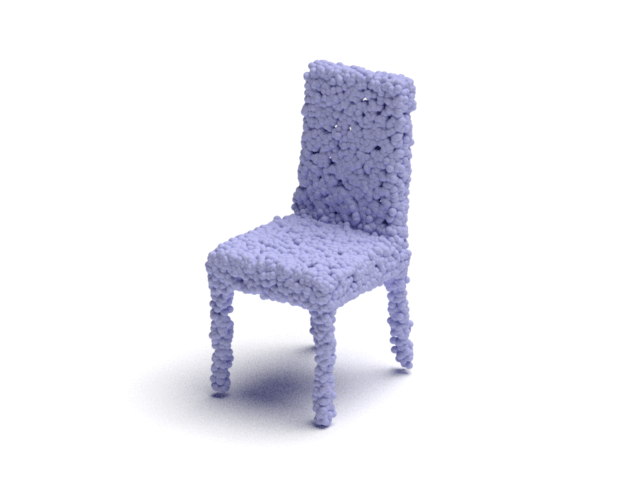} &
        \includegraphics[width=\sizea, trim={\thl} {\thb} {\thr} {\tht},clip]{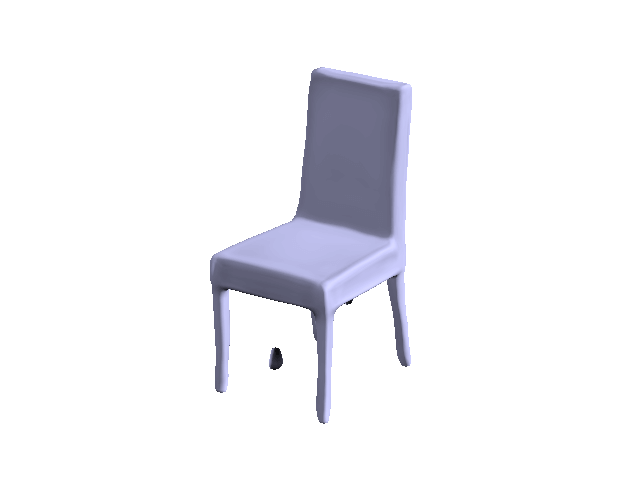} \\
        \includegraphics[width=\sizea, trim={\thl} {\thb} {\thr} {\tht},clip]{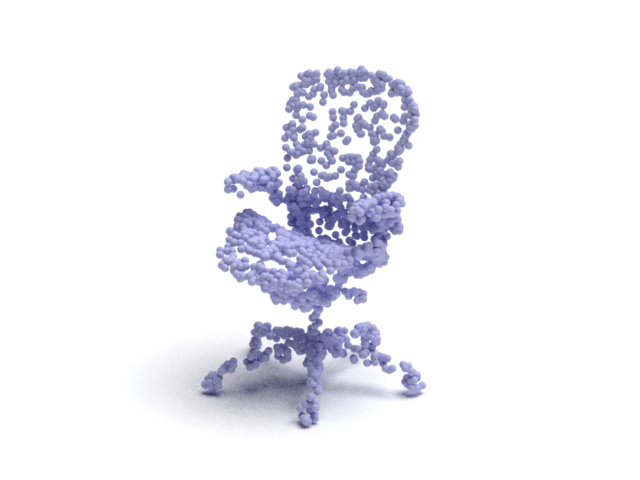} &
        \includegraphics[width=\sizea, trim={\thl} {\thb} {\thr} {\tht},clip]{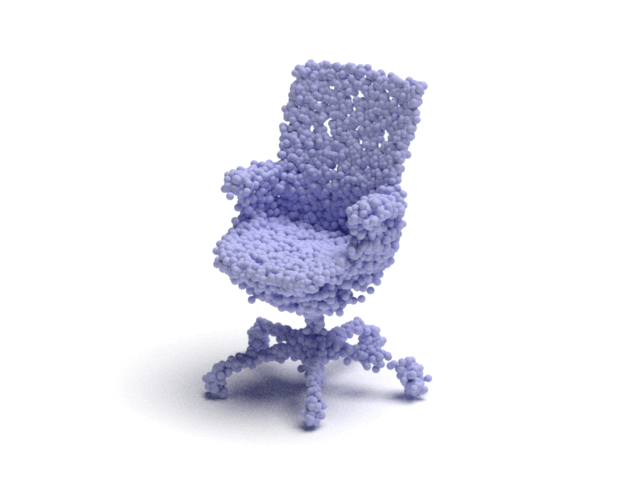} &
        \includegraphics[width=\sizea, trim={\thl} {\thb} {\thr} {\tht},clip]{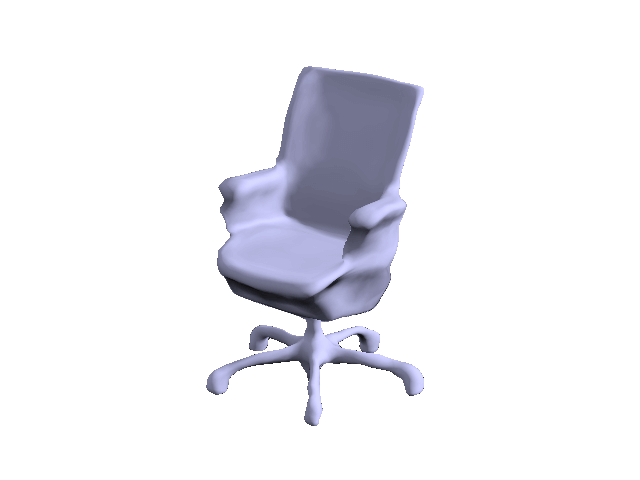} &
        \includegraphics[width=\sizea, trim={\thl} {\thb} {\thr} {\tht},clip]{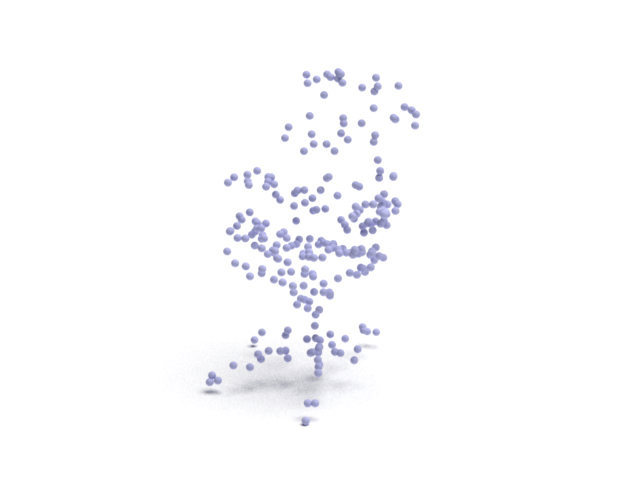} &
        \includegraphics[width=\sizea, trim={\thl} {\thb} {\thr} {\tht},clip]{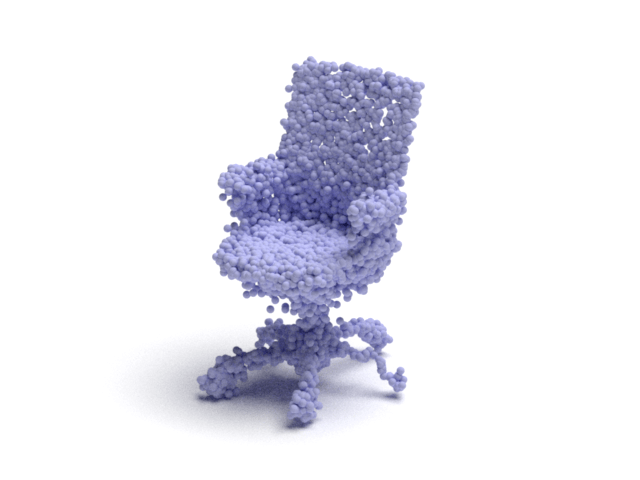} &
        \includegraphics[width=\sizea, trim={\thl} {\thb} {\thr} {\tht},clip]{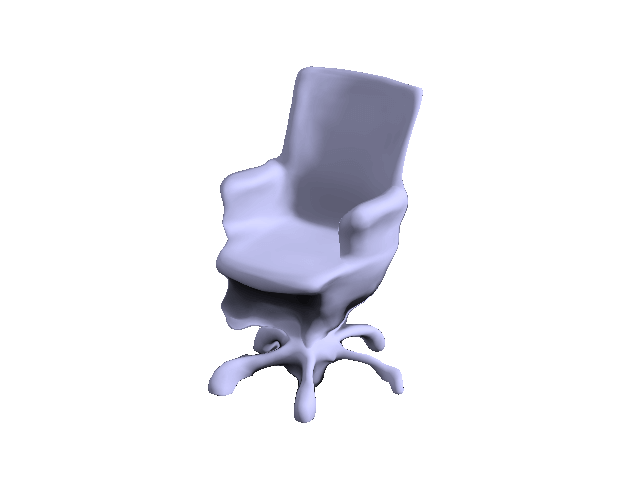} \\
    \end{tabular}

    \end{center}
    \caption{Shape auto-encoding test results. Our model can accurately reconstruct shapes given 2048 points (left) or only 256 points (right) describing the shape. Output point clouds are illustrated in the center and implicit surfaces on the left.
    }
    \label{fig:recon}
\end{figure}

\subsubsection{Evaluation metrics}
Following prior works~\cite{yang2019pointflow,atlasnet,achlioptas_L3DP}, we use the symmetric Chamfer Distance (CD) and the Earth Mover's Distance (EMD) to evaluate the reconstruction quality of the point clouds.
To evaluate the generation quality, we use metrics in Yang~\etal~\cite{yang2019pointflow} and Achlioptas~\etal~\cite{achlioptas_L3DP}.
Specifically, Achilioptas~\etal~\cite{achlioptas_L3DP} suggest using Minimum Matching Distance (MMD) to measure fidelity of the generated point cloud and Coverage (COV) to measure whether the set of generated samples cover all the modes of the data distribution.
Yang~\etal~\cite{yang2019pointflow} propose to use the accuracy of a $k$-NN classifier performing two-sample tests.
The idea is that if the sampled shapes seem to be drawn from the actual data distribution, then the classifier will perform like a random guess (i.e. results in $50\%$ accuracy).
To evaluate our results, we first conduct per-shape normalization to center the bounding box of the shape and scale its longest length to be 2, which allows the metric to focus on the geometry of the shape and not the scale.

\subsection{Shape auto-encoding}\label{sec:auto-encoding}
\newcolumntype{Y}{>{\centering\arraybackslash}X}
\newcolumntype{L}{>{\arraybackslash}X}
\begin{table*}[t]
\centering
\caption{Shape auto-encoding on the MNIST-CP and ShapeNet datasets. The best results are highlighted in bold. 
CD is multiplied by $10^4$ and EMD is multiplied by $10^2$.}
\label{tab:auto-encoding}
\begin{tabularx}{\textwidth}{l*{9}{Y}}
\toprule
                          &           & \multicolumn{2}{c}{l-GAN~\cite{achlioptas_L3DP}} & \multicolumn{2}{c}{AtlasNet~\cite{atlasnet}} & \multirow{2}{*}{PF~\cite{yang2019pointflow}} & \multirow{2}{*}{Ours} & \multirow{2}{*}{Oracle} \\ 
\cmidrule(lr){3-4} \cmidrule(lr){5-6}
Dataset                   & Metric    & CD          & EMD        & Sphere       & Patches       &                            &                       &                         \\ \midrule
\multirow{2}{*}{MNIST-CP} & CD      & 8.204     & -           & 7.274        & 4.926         & 17.894                     & \bf{2.669}                 & 1.012                   \\
                          & EMD     & 40.610    & -     & 19.920 & 15.970 & 8.705 & \textbf{7.341} & 4.875 \\
                         \cmidrule{1-9}
\multirow{2}{*}{Airplane} & CD      & 1.020     & 1.196     & 1.002     & 0.969     & 1.208     & \textbf{0.96}     & 0.837 \\
                          & EMD     & 4.089     & 2.577     & 2.672     & 2.612     & 2.757     & \textbf{2.562}    & 2.062 \\
                          \cmidrule{1-9}
\multirow{2}{*}{Chair}    & CD      & 9.279     & 11.21     & 6.564     & 6.693     & 10.120    & \textbf{5.599}    & 3.201 \\
                          & EMD     & 8.235     & 6.053     & 5.790     & 5.509     & 6.434     & \textbf{4.917}    & 3.297 \\
                          \cmidrule{1-9}
\multirow{2}{*}{Car}     & CD       & 5.802     & 6.486     & 5.392     & 5.441     & 6.531     & \textbf{5.328}    & 3.904 \\
                         & EMD      & 5.790     & 4.780     & 4.587     & 4.570     & 5.138     & \textbf{4.409}    & 3.251 \\
                          \cmidrule{1-9}
\multirow{2}{*}{ShapeNet}& CD       & 7.120     & 8.850     & 5.301     & \bf{5.121}& 7.551     & 5.154             & 3.031 \\
                         & EMD      & 7.950     & 5.260     & 5.553     & 5.493     & 5.176     & \textbf{4.603}    & 3.103
\\ \bottomrule
\end{tabularx}
\end{table*}
In this section, we evaluate how well our model can learn the underlying distribution of points by asking it to auto-encode a point cloud.
We conduct the auto-encoding task for five settings: all 2D point clouds in MNIST-CP, 3D point clouds on the whole ShapeNet, and three categories in ShapeNet (Airplane, Car, Chair).
In this experiment, our method is compared with the current state-of-the-art AtlasNet~\cite{atlasnet} with patches and with sphere.
Furthermore, we also compare against Achilioptas~\etal~\cite{achlioptas_L3DP} which predicts point clouds as a fixed-dimensional array, and PointFlow~\cite{yang2019pointflow} which uses a flow-based model to represent the distribution.
We follow the experiment set-up in PointFlow to report performance in both CD and EMD in Table~\ref{tab:auto-encoding}.
Since these two metrics depend on the scale of the point clouds, we also report the upper bound in the ``oracle" column.
The upper bound is produced by computing the error between two different point clouds with the same number of points sampled from the same underlying meshes.

Our method consistently outperforms all other methods on the EMD metric, which suggests that our point samples follow the distribution or they are more uniformly distributed among the surface.
Note that our method outperforms PointFlow in both CD and EMD for all datasets, but requires much less time to train.
Our training for the Airplane category can be completed in about less than 10 hours, yet reproducing the results for PointFlow's pretrained model takes at least two days.
Our method can even sometimes outperform Achilioptas~\etal and AtlasNet in CD, which is the loss they are directly optimizing at.
\begin{table*}[t]
\centering
\caption{Auto-encoding sparse point clouds. We randomly sample $N$ points from each shape (in the Airplane dataset). During training, the model is provided with $M$ points (the columns). CD is multiplied by $10^4$ and EMD is multiplied by $10^2$.
}
\label{tab:airplane-ablation}
\begin{tabularx}{\textwidth}{L*{5}{Y}l*{5}{Y}}
\toprule
    & \multicolumn{5}{c}{CD}                &  & \multicolumn{5}{c}{EMD}               \\ \cmidrule(lr){2-6} \cmidrule(l){8-12} 
N   & 2048  & 1024  & 512   & 256   & 128   &  & 2048  & 1024  & 512   & 256   & 128   \\ \midrule
$10K$ & 0.993 & 1.057 & 0.999 & 1.136 & 1.688 &  & 2.463 & 2.608 & 2.589 & 3.042 & 3.715 \\
$3K$  & 1.080 & 1.059 & 1.003 & 1.142 & 1.753 &  & 2.533 & 2.586 & 2.557 & 2.997 & 3.878 \\
$1K$  & -     & -     & 1.021 & 1.149 & 1.691 &  & -     & -     & 2.565 & 2.943 & 3.633 \\ \bottomrule
\end{tabularx}
\end{table*}

\subsubsection*{Point cloud upsampling}\label{sec:robustness}
We conduct a set of experiments on subsampled ShapeNet point clouds. 
These experiments are primarily focused on showing that (i) our model can learn from sparser datasets, and that (ii) we can infer a dense shape from a sparse input.
In the regular configuration (reported above), we learn from $N=10K$ points which are uniformly sampled from each shape mesh model. 
During training, we sample $M=2048$ points (from the $10K$ available in total) to be the input point cloud. 
To evaluate our model, we perform the Langevin dynamic procedure (described in Section \ref{sec:sample-point-clouds}) over $2048$ points sampled from the prior distribution and compare these to $2048$ points from the reference set. 

To evaluate whether our model can effectively upsample point clouds and learn from a sparse input, we train models with $N=[1K,3K,10K]$ and $M=[128,256,512,1024,2048]$ on the Airplane dataset. 
To allow for a fair comparison, we evaluate all models using the same number of output points (i.e. $2048$ points are sampled from the prior distribution in all cases).
As illustrated in Table~\ref{tab:airplane-ablation}, we obtain comparable auto-encoding performance while training with significantly sparser shapes. 
Interestingly, the number of points available from the model (i.e. $N$) does not seem to affect performance, suggesting that we can indeed learn from sparser datasets.
Several qualitative examples auto-encoding shapes from the regular and sparse configurations are shown in Figure~\ref{fig:recon}. 
We also demonstrate that our model can also provide a smooth iso-surface, even when only a sparse point cloud (i.e. 256 points) is provided as input.

\subsection{Shape generation}\label{sec:shape-generaton}
We quantitatively compare our method's performance on shape generation with r-GAN~\cite{achlioptas_L3DP}, GCN-GAN \cite{valsesia2018learning}, TreeGAN \cite{shu20193d}, and PointFlow~\cite{yang2019pointflow}.
We use the same experiment setup as PointFlow except for the data normalization before the evaluation.
The generation results are reported in Table~\ref{tab:generation}. 
Though our model requires a two-stage training, the training can be done within one day with a 1080 Ti GPU, while reproducing PointFlow's results requires training for at least two days on the same hardware.
Despite using much less training time, our model achieves comparable performance to PointFlow, the current state-of-the-art.
As demonstrated in Figure~\ref{fig:gen}, our generated shapes are also visually cleaner.

% <left> <lower> <right> <upper>}
\begin{figure}[t]
    \begin{center}
    \newcommand{\sizea}{0.124\linewidth}
    \newcommand{\sizeb}{0.14\linewidth}
    \newcommand{\tare}{5cm}
    \newcommand{\tale}{4.5cm}
    \newcommand{\tal}{3.5cm}
    \newcommand{\tab}{3.5cm}
    \newcommand{\tar}{4cm}
    \newcommand{\tat}{4.8cm}
    \newcommand{\tcl}{3.0cm}
    \newcommand{\tcb}{3cm}
    \newcommand{\tcr}{4cm}
    \newcommand{\tct}{4.2cm}
    \newcommand{\thl}{3.0cm}
    \newcommand{\thb}{0.0cm}
    \newcommand{\thr}{3cm}
    \newcommand{\tht}{2cm}
    \setlength{\tabcolsep}{0pt}
    \renewcommand{\arraystretch}{0}
    \begin{tabular}{@{}cccc:cccc@{}}
        \includegraphics[width=\sizea, trim={\tale} {\tab} {\tare} {\tat},clip]{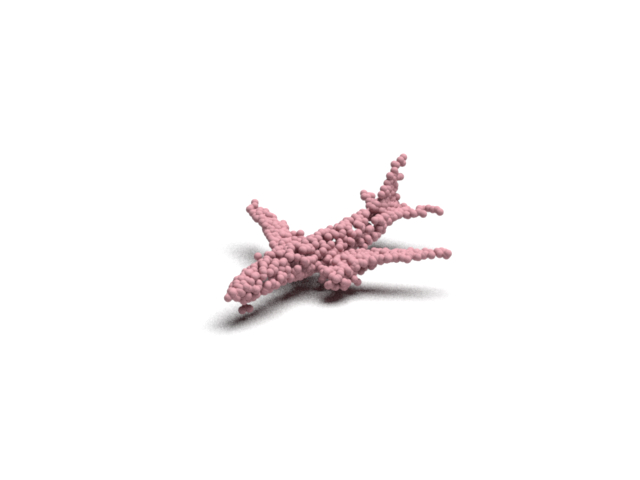}&
        \includegraphics[width=\sizea, trim={\tale} {\tab} {\tare} {\tat},clip]{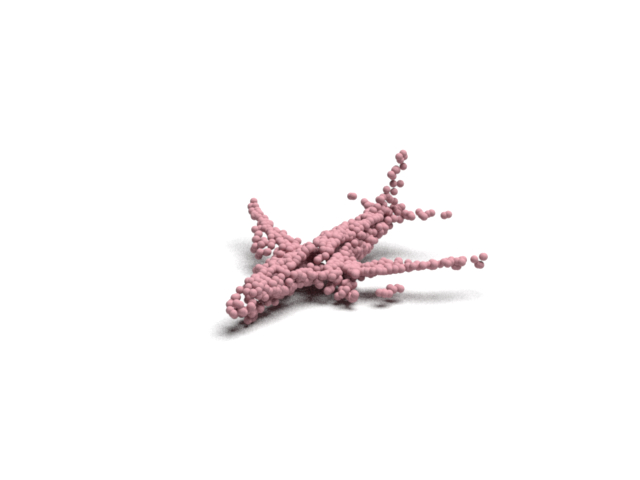}&
        \includegraphics[width=\sizea, trim={\tale} {\tab} {\tare} {\tat},clip]{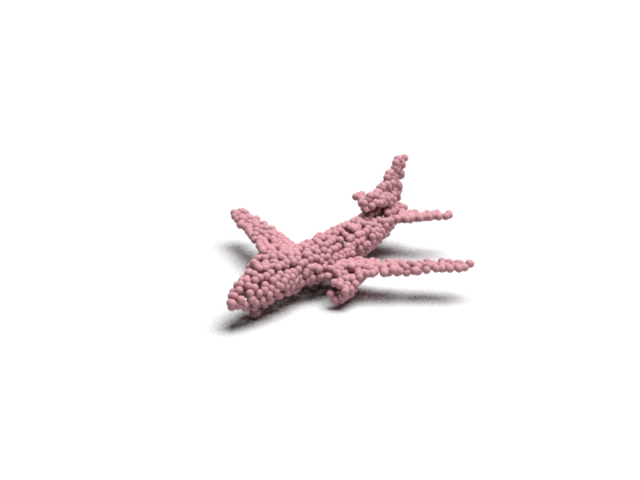}&
        \includegraphics[width=\sizea, trim={\tal} {\tab} {\tar} {\tat},clip]{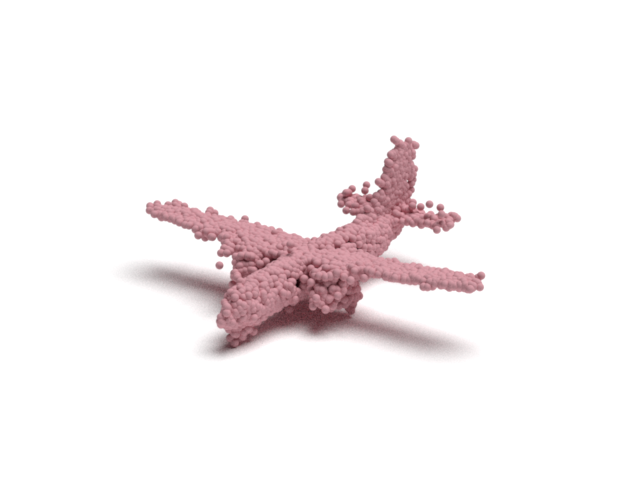}&
        \includegraphics[width=\sizea, trim={2cm} {\tab} {5cm} {\tat},clip]{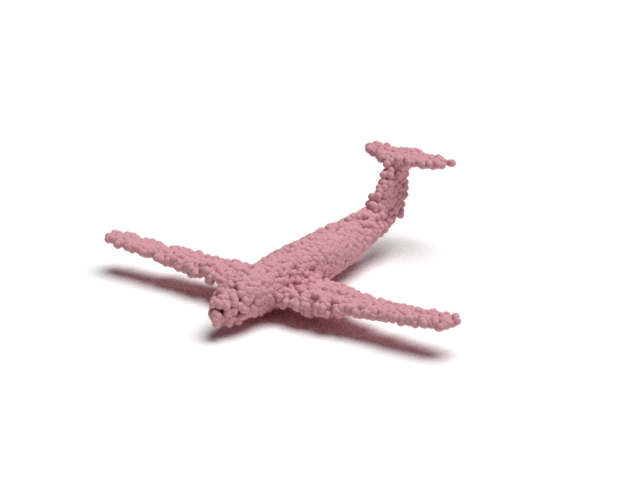}& 
        \includegraphics[width=\sizea, trim={2cm} {\tab} {5cm} {\tat},clip]{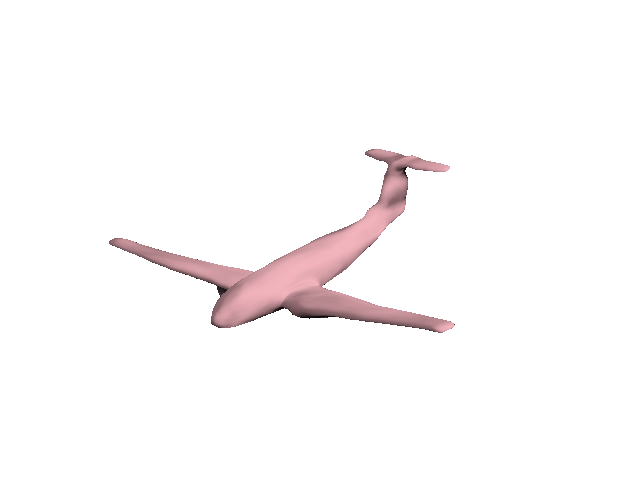}& 
        \includegraphics[width=\sizea, trim={\tal} {\tab} {\tar} {\tat},clip]{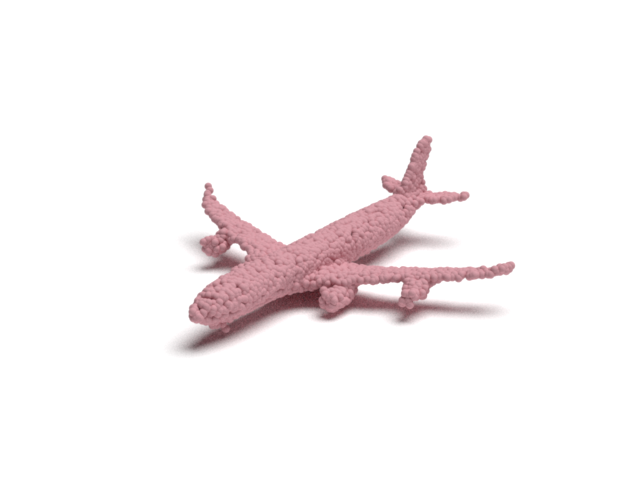}&
        \includegraphics[width=\sizea, trim={\tal} {\tab} {\tar} {\tat},clip]{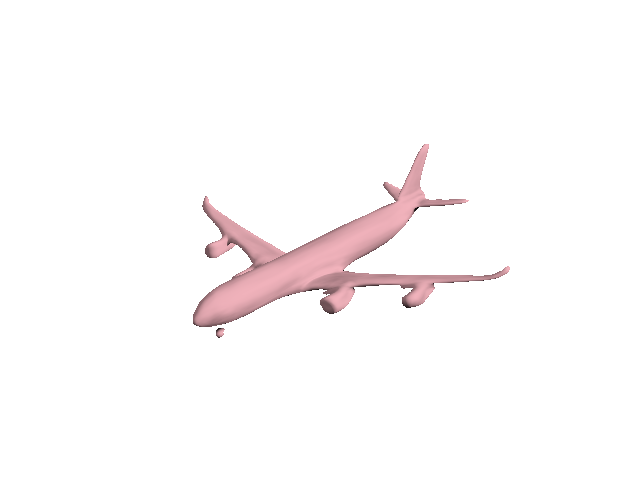}
        \\
        \includegraphics[width=\sizea, trim={\thl} {\thb} {\thr} {\tht},clip]{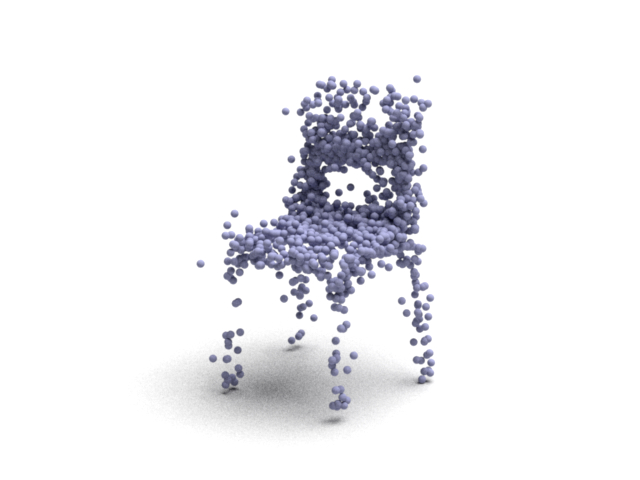}&
        \includegraphics[width=\sizea, trim={\thl} {\thb} {\thr} {\tht},clip]{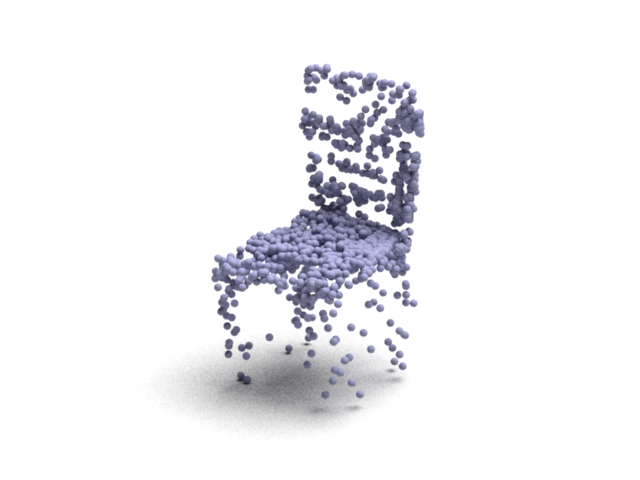}&
        \includegraphics[width=\sizea, trim={\thl} {\thb} {\thr} {\tht},clip]{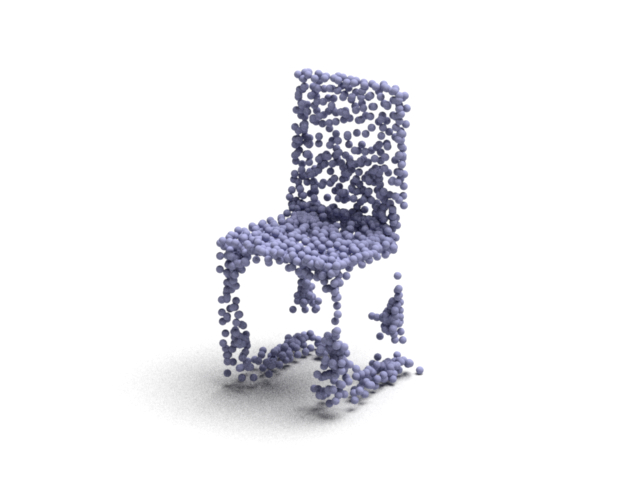}&
        \includegraphics[width=\sizea, trim={\thl} {\thb} {\thr} {\tht},clip]{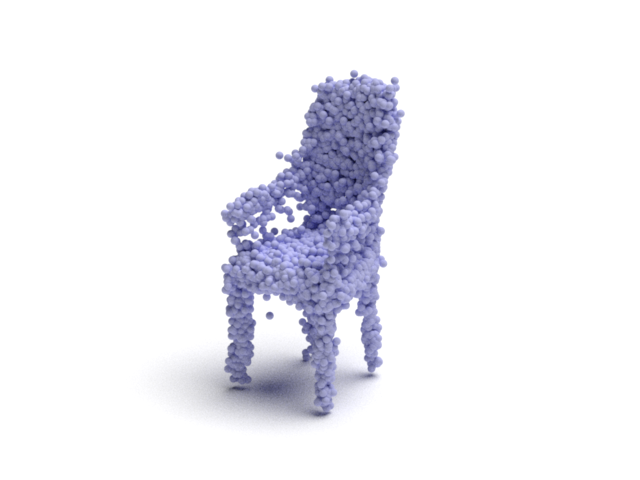}&
        \includegraphics[width=\sizea, trim={2.7cm} {\thb} {2.7cm} {\tht},clip]{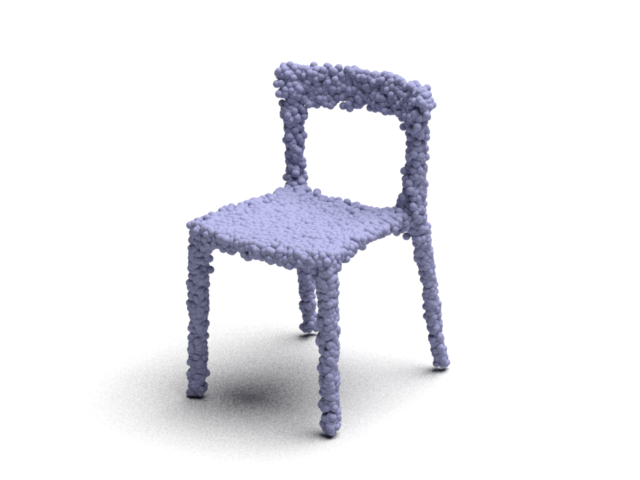} &
        \includegraphics[width=\sizea, trim={4cm} {\thb} {4cm} {\tht},clip]{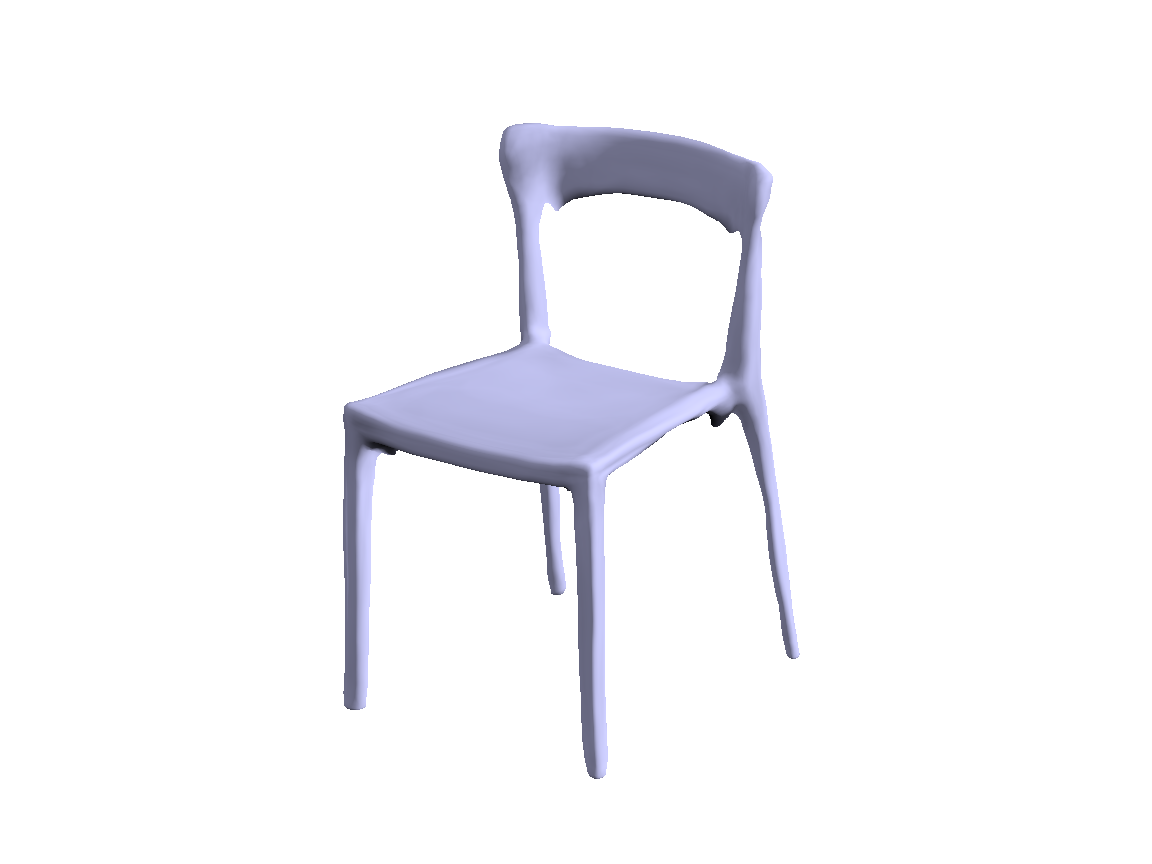} &
        \includegraphics[width=\sizea, trim={\thl} {\thb} {\thr} {\tht},clip]{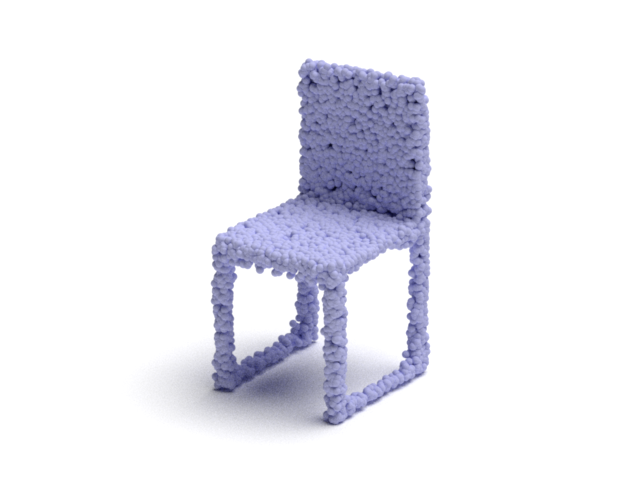} &
        \includegraphics[width=\sizea, trim={\thl} {\thb} {\thr} {\tht},clip]{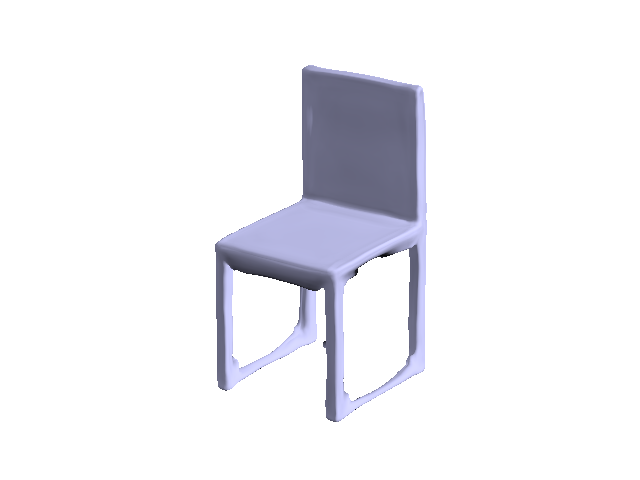} 
        \\
        r-GAN & CGN & Tree & PF & \multicolumn{4}{c}{Ours}
    \end{tabular}

    \end{center}
    \caption{Generation results. We shown results from r-GAN, GCN, TreeGAN (Tree), and PointFlow (PF) are illustrated on the left for comparison. Generated point clouds are illustrated alongside the corresponding implicit surfaces.
    }
    \label{fig:gen}
\end{figure}

\begin{table*}[ht]
\centering
\caption{Shape generation results. $\uparrow$ means the higher the better, $\downarrow$ means the lower the better. 
MMD-CD is multiplied by $10^3$ and MMD-EMD is multiplied by $10^2$.}
\label{tab:generation}
\begin{tabularx}{\textwidth}{LL*{2}{Y}l*{2}{Y}l*{2}{Y}}
\toprule
                          &           & \multicolumn{2}{c}{MMD ($\downarrow$)} &  & \multicolumn{2}{c}{COV (\%, $\uparrow$)} &  & \multicolumn{2}{c}{1-NNA (\%, $\downarrow$)} \\ \cmidrule(lr){3-4} \cmidrule(lr){6-7} \cmidrule(l){9-10} 
Category                  & Model     & CD         & EMD        &  & CD         & EMD       &  & CD     & EMD         \\ \midrule
\multirow{4}{*}{Airplane} & r-GAN~\cite{achlioptas_L3DP}           
                          & 1.657           & 13.287         &  & 38.52         & 19.75     &  & 95.80  & 100.00         \\
                          & GCN~\cite{valsesia2018learning} 
                          & 2.623           & 15.535         &  & 9.38          & 5.93      &  & 95.16  & 99.12 \\
                          & Tree~\cite{shu20193d} 
                          & 1.466           & 16.662         &  & 44.69         & 6.91      &  & 95.06  & 100.00 \\
                          & PF~\cite{yang2019pointflow} 
                          & 1.408           & 7.576 &  & 39.51 & \textbf{41.98}          &  & \textbf{83.21} & \textbf{82.22} \\
                          & Ours            & \textbf{1.285}  & \textbf{7.364}         &  & \textbf{47.65}          & \textbf{41.98} &  & 85.06          & 83.46\\
                          \cmidrule(l{-1pt}r{0pt}){2-10} 
                          & Train           & 1.288           & 7.036         &  & 45.43          & 45.43          &  & 72.10          & 69.38          \\\midrule
\multirow{4}{*}{Chair}    & r-GAN~\cite{achlioptas_L3DP}           & 18.187          & 32.688          &  & 19.49          & 8.31           &  & 84.82          & 99.92          \\
                          & GCN~\cite{valsesia2018learning} 
                          & 23.098          & 25.781         &  & 6.95          & 6.34      &  & 86.52  & 96.48 \\
                          & Tree~\cite{shu20193d} 
                          & 16.147          & 36.545         &  & 40.33         & 8.76      &  & 74.55  & 99.92 \\
                          & PF~\cite{yang2019pointflow} & 15.027          & 19.190          &  & 40.94          & 44.41          &  & 67.60          & 72.28          \\
                          & Ours            & \textbf{14.818} & \textbf{18.791} &  & \textbf{46.37} & \textbf{46.22} &  & \textbf{66.16} & \textbf{59.82} \\\cmidrule(l{-1pt}r{0pt}){2-10} 
                          & Train           & 15.893          & 18.472          &  & 50.45          & 52.11          &  & 53.93          & 54.15          \\ 
\bottomrule
\end{tabularx}
\end{table*}

\begin{table*}[t]
\centering
\caption{Ablation study comparing auto-encoding performance on the Airplane dataset.
CD is multiplied by $10^4$ and EMD is multiplied by $10^2$.
}
\label{tab:fix-sigma-ablation}
\begin{tabularx}{0.8\textwidth}{l*{7}{Y}}
\toprule
& \multicolumn{3}{c}{Single noise level}                  & \multicolumn{3}{c}{Prior distribution} \\ \cmidrule(l{6pt}r{4pt}){2-4} \cmidrule(l{2pt}r{0pt}){5-7}  
Metric                & 0.1    & 0.05  & 0.01     & Uniform & Fixed            & Gaussian           \\ \midrule
CD                              & 2.545  & 1.573 & 1009.357 & 0.993                & 0.993             & 0.996             \\
EMD                             & 4.400 & 8.455 & 36.715 & 2.463 & 2.476 & 2.475              
\\ \bottomrule
\end{tabularx}
\end{table*}

\subsection{Ablation study}\label{sec:ablation}
We conduct an ablation study quantifying the importance of learning with multiple noise levels. As detailed in Sections \ref{sec:loss}-\ref{sec:sample-point-clouds}, we train $s_\theta$ for multiple $\sigma$'s. During inference, we sample point clouds using an annealed Langevin dynamics procedure, using the same $\sigma$'s seen during training.
In Table \ref{tab:fix-sigma-ablation} we show results for models trained with a single noise level and tested without annealing.
As illustrated in the table, the model does not perform as well when learning using a single noise level only. 
This is especially noticeable for the model trained on the smallest noise level in our model ($\sigma=0.01$), as large regions in space are left unsupervised, resulting in significant errors.

We also demonstrate that our model is insensitive to the choice of the prior distribution. We repeat the inference procedure for our auto-encoding experiment, initializing the prior points with a Gaussian distribution or in a fixed location (using the same trained model). 
Results are reported on the right side of Table \ref{tab:fix-sigma-ablation}.
Different prior configurations don't affect the performance, which is expected due to the stochastic nature of our solution.
We further demonstrate our model's robustness to the prior distribution in Figure \ref{fig:teaser}, where the prior depicts 3D letters.

\section{Conclusions}

In this work, we propose a generative model for point clouds which learns the gradient field of the logarithmic density function encoding a shape.
Our method not only allows sampling of high-quality point clouds, but also enables extraction of the underlying surface of the shape.
We demonstrate the effectiveness of our model on point cloud auto-encoding, generation, and super-resolution.
Future work includes extending our work to model texture, appearance, and scenes.

\textbf{Acknowledgment.}
This work was supported in part by grants from Magic Leap and Facebook AI, and the Zuckerman STEM leadership program.

\clearpage
\bibliographystyle{splncs04}
\bibliography{refs}

\newpage

\begin{appendix}
\section{Overview}

In this appendix, we provide additional method details, implementation details, experimental results, and qualitative results visualizations. 
\begin{enumerate}
    \item Section~\ref{sec:method_detail} Method Details - Additional details about the method including algorithm blocks of training, inference and surface extraction, and the proof of equivalent objective function. 
    \item Section~\ref{sec:impl-detail} Implementation Details - Description of the network architecture, experimental setting, and implementation details.
    \item Section~\ref{sec:quant_result} Additional Quantitative Results - More quantitative results on shape generation and implicit surface tasks, comparing our model with additional baselines.
    \item Section \ref{sec:ablation_studies} Additional Ablation Studies - Results of additional ablation studies.
    \item Section \ref{sec:qual_result} Additional Qualitative Results - Additional qualitative results on scanned data, visualization of latent space, and extended visualizations for 2D and 3D point clouds.
\end{enumerate}

\section{Method Details}\label{sec:method_detail}
In this section, we provide extensive details of our method.
We provide detailed explanations for training and inference in Section~\ref{sec:train-inference}.
More details about surface extraction are provided in Section~\ref{sec:surface}.
Finally, a mathematical proof about our objective function is shown in Section~\ref{sec:objective}.

\subsection{Training and inference}\label{sec:train-inference}
We provide algorithm blocks to better illustrate the training
and inference  procedures. See Algorithm \ref{alg:training} and Algorithm \ref{alg:inference} for training and inference, respectively. Please refer to Section~\ref{sec:impl-detail} for hyper-parameters and neural network architectures.
\begin{algorithm}
\caption{Training.}
\begin{algorithmic}[1]
\REQUIRE Noise levels $\{\sigma_i\}_{i=1}^k$; Weight for noise levels' loss $ \lambda(\sigma_i)$; Point cloud encoder $f_\phi$; A neural network $s_\theta$; Total number of training iterations $T$; Point cloud $X_t$

\FOR{$t \gets 1$ to $T$}
\STATE $z \gets f_\phi(X_t)$ 
\FOR{$\sigma \in \{\sigma_i\}_{i=1}^k$}
\FOR{$x_i \in X_t$}
\STATE $\tilde{x}_i \gets x_i + N(0,\sigma^2I)$
\ENDFOR
\STATE 
$\ell(\sigma, X_t) \gets 
    \frac{1}{|X_t|}\sum_{x_i \in X_t} 
    \|\score(\tilde{x_i}, \sigma, z) - \frac{x_i-\tilde{x}_i}{\sigma^2}\|_2^2$
\ENDFOR
\STATE $\mathcal{L}(\{\sigma_i\}_{i=1}^k, X_t) \gets \sum_{i=1}^k \lambda(\sigma_i) \ell(\sigma_i, X_t)$
\STATE $\phi, \theta \gets Adam(\mathcal{L},\phi, \theta)$
\ENDFOR
\RETURN $f_\phi$, $s_\theta$
\end{algorithmic}
\label{alg:training}
\end{algorithm}
\begin{algorithm}
\caption{Annealed Langevin dynamics.}
\begin{algorithmic}[1]
\REQUIRE Noise levels $\{\sigma_i\}_{i=1}^k$; Step size $\alpha$; Number of steps $T$
\STATE Initialize $x_0$\;
\FOR{$i \gets 1$ to $k$}
\FOR{$t \gets 0$ to $T-1$}
\STATE $\epsilon_t \sim \mathcal{N}(0, I)$
\STATE $x_{t+1}' \gets x_t + \frac{\sqrt{\alpha}\sigma_i \epsilon_t}{\sigma_k}$
\STATE $x_{t+1} \gets x_{t+1}' + \frac{\alpha\sigma_i^2}{2\sigma_k^2} \score(x_{t+1}', \sigma_i)$
\ENDFOR
\STATE $x_0 \gets x_T$
\ENDFOR
\RETURN $x_T$
\end{algorithmic}
\label{alg:inference}
\end{algorithm}

\subsection{Surface extraction}\label{sec:surface}
We use a modified version of the volumetric ray casting algorithm \cite{roth1982ray} to render the iso-surface produced by the learned gradient field. Algorithm \ref{alg:surface} below shows the rendering process of a single pixel. For each pixel in the rendered image, we cast a ray $\langle o_0, u\rangle$ towards the gradient field according to the camera model. We advance the ray for a fixed number of steps $k_{\mathrm{max}}$. For reasonable choices of the step rate $\gamma$, the ray will either converge to the iso-surface, or miss the iso-surface and march towards infinity. If the ray does reach the iso-surface, we calculate the RGB values for that pixel based on its 3D location and surface normal. If the ray misses the surface, we assign the pixel with background color $y_{\mathrm{bg}}$. In this paper, we use $\gamma=1.0$,  $k_\mathrm{max}=64$ and $\delta=0.005$. 
\begin{algorithm}
\caption{Ray Casting for Rendering the Iso-surface.}
\begin{algorithmic}[1]
\REQUIRE Neural network $s_\theta$; Minimum noise level $\sigma_k$; Initial ray origin $o_0$; Ray direction $u$; Maximum ray travel $d_{\mathrm{max}}$; Step rate $\gamma$; Number of steps $k_{\mathrm{max}}$; Iso-surface level $\delta$; Background color $y_{\mathrm{bg}}$
\STATE $d = 0$\;
\FOR{$k \gets 1$ to $k_{\mathrm{max}}$}
    \STATE $x \gets o_0 + du$
    \STATE $d \gets d + \gamma (\|s_\theta(x, \sigma_k)\| - \delta)$
\ENDFOR
\IF{$d < d_{\mathrm{max}}$}
\STATE $n \gets -\frac{s_\theta(x, \sigma_k)}{\|s_\theta(x, \sigma_k)\|}$
\STATE $y \gets \mathrm{Shading}(x,n)$
\ELSE
\STATE $y \gets y_{\mathrm{bg}}$
\ENDIF
\RETURN $y$
\end{algorithmic}
\label{alg:surface}
\end{algorithm}

\subsection{Objective Function}\label{sec:objective}

Here, we provide a proof to show that optimizing $\ell_{\text{direct}}(\sigma, S)$ (Equation 4 in the main paper) is equivalent to optimizing $\ell_{\text{denoising}}(\sigma, S)$ (Equation 5 in the main paper). 
The proof is largely similar to the one in the Appendix in Vincent 2010~\cite{vincent2011connection}. 
We will re-visit the prove here using the notation from our paper for convenience of the readers.

% \textbf{Theorem} 
\begin{theorem}
% Let $\theta_{\text{direct}}^*$ minimizes  $\ell_{\text{direct}}(\sigma, S)$, and $\theta_{\text{denoise}}^*$ minimizes $\ell_{\text{denoising}}(\sigma, S)$, then $\theta_{\text{direct}}^* = \theta_{\text{denoise}}^*$. 
Optimizing $\ell_{\text{denoise}}(\sigma, S)$ leads to the same $\theta$ as optimizing $\ell_{\text{denoise}}(\sigma, S)$.
\end{theorem}

\begin{proof}
% \textbf{Proof}:
We want to show that $\ell_{\text{direct}}(\sigma, S) = \ell_{\text{denoise}}(\sigma, S) + C$ for some constant $C$ that doesn't depend on $\theta$.
Note that $\ell_{\text{denoise}}$ can be decomposed as follows:
\begin{align}
    \ell_{\text{denoise}}(\sigma, S) &= 
    \mathbb{E}_{x\sim \pdata, \tilde{x} \sim \qnoise(\tilde{x}|x)}\left[\frac{1}{2}\|\score(\tilde{x}, \sigma)\|^2\right] \nonumber \\
    &\quad - \mathbb{E}_{x\sim \pdata, \tilde{x} \sim \qnoise(\tilde{x}|x)}\left[\score(\tilde{x}, \sigma)^T\nabla_{\tilde{x}}\log{\qnoise(\tilde{x}|x)}\right] + C_1 \label{eq:denoise},
\end{align}
where $C_1=\mathbb{E}_{x\sim \pdata, \tilde{x} \sim \qnoise(\tilde{x}|s)}\left[\frac{1}{2}\left\|\nabla_{\tilde{x}}\log{\qnoise(\tilde{x}|x)}\right\|^2\right] $. 
Similarly, we can decompose $\ell_{\text{direct}}$ as follows:
\begin{align}
    \ell_{\text{direct}}(\sigma, S) &=
    \mathbb{E}_{x \sim \qdata}\left[\frac{1}{2}\|\score(x, \sigma)\|^2\right] \nonumber \\
    &\quad - \mathbb{E}_{x \sim \qdata}\left[\score(x, \sigma)^T\nabla_x\log{\qdata(x)}\right] + C_2\label{eq:direct},
\end{align}
where $C_2=\mathbb{E}_{x \sim \qdata}\left[\frac{1}{2}\|\nabla_x \log{\qdata(x)}\|^2\right]$.
Now we will compare the first two terms of Equation~\ref{eq:denoise} and Equation~\ref{eq:direct} to show they are the same. 
Looking at the first terms:
\begin{align*}
    &\mathbb{E}_{s\sim \pdata, \tilde{x} \sim \qnoise(\tilde{x}|s)}\left[\frac{1}{2}\|\score(\tilde{x}, \sigma)\|^2\right] \\
    &\quad= \frac{1}{2}\int_x \pdata(x) \int_{\tilde{x}} \qnoise(\tilde{x}|x) \|\score(\tilde{x}, \sigma)\|^2 d\tilde{x}dx \\
    &\quad= \frac{1}{2}\int_{\tilde{x}}\int_x \pdata(x) \qnoise(\tilde{x}|x) \|\score(\tilde{x}, \sigma)\|^2 dxd\tilde{x} \\
    &\quad= \frac{1}{2}\int_{\tilde{x}}\|\score(\tilde{x}, \sigma)\|^2\int_x \pdata(x) \qnoise(\tilde{x}|x)  dxd\tilde{x} \\
    &\quad= \frac{1}{2}\int_{\tilde{x}}\|\score(\tilde{x}, \sigma)\|^2 \qdata(\tilde{x}) d\tilde{x} \quad \left(\text{since } \qdata(x) = \int_y \pdata(y)\qnoise(x|y)dy\right) \\
    &\quad=\mathbb{E}_{x\sim\qdata}\left[\frac{1}{2}\|\score(x,\sigma)\|^2\right].
\end{align*}
Looking at the second terms:
\begin{align*}
     &\mathbb{E}_{x\sim \pdata, \tilde{x} \sim \qnoise(\tilde{x}|x)}\left[\score(\tilde{x}, \sigma)^T\nabla_{\tilde{x}}\log{\qnoise(\tilde{x}|x)}\right] \\
     &\quad=\int_x\pdata(x)\int_{\tilde{x}}\qnoise(\tilde{x}|x) \score(\tilde{x}, \sigma)^T\nabla_{\tilde{x}} \log{\qnoise(\tilde{x}|x)} d\tilde{x}dx \\
     &\quad=\int_{\tilde{x}}\int_x\pdata(x)\qnoise(\tilde{x}|x) \score(\tilde{x}, \sigma)^T\nabla_{\tilde{x}}\log{\qnoise(\tilde{x}|x)} dxd\tilde{x} \\
      &\quad=\int_{\tilde{x}}\score(\tilde{x}, \sigma)^T
      \int_x\pdata(x)\qnoise(\tilde{x}|x) \nabla_{\tilde{x}}\log{\qnoise(\tilde{x}|x)} dxd\tilde{x} \\
     &\quad=\int_{\tilde{x}}\score(\tilde{x}, \sigma)^T
     \int_x\pdata(x)\qnoise(\tilde{x}|x) \frac{\nabla_{\tilde{x}}\qnoise(\tilde{x}|x)}{\qnoise(\tilde{x}|x)}
     dxd\tilde{x} \\
     &\quad=\int_{\tilde{x}}\score(\tilde{x}, \sigma)^T
     \int_x\pdata(x)\nabla_{\tilde{x}}\qnoise(\tilde{x}|x)dxd\tilde{x} \\
     &\quad=\int_{\tilde{x}}\score(\tilde{x}, \sigma)^T
     \int_x\pdata(x)\nabla_{\tilde{x}}\qnoise(\tilde{x}|x)dxd\tilde{x}.
\end{align*}
Since $\qnoise(\tilde{x}|x) = \mathcal{N}(\tilde{x}; x, \sigma^2I)$,  it is bounded by $\qnoise(\tilde{x}|x) \leq \mathcal{N}(x;x,\sigma^2I)$.
As a result, we can take the derivative outside of the integral $\int_x\pdata(x)\nabla_{\tilde{x}}\qnoise(\tilde{x}|x)dx$:
\begin{align*}
     &\mathbb{E}_{x\sim \pdata, \tilde{x} \sim \qnoise(\tilde{x}|x)}\left[\score(\tilde{x}, \sigma)^T\nabla_{\tilde{x}}\log{\qnoise(\tilde{x}|x)}\right]\\
     &\quad=\int_{\tilde{x}}\score(\tilde{x}, \sigma)^T
     \nabla_{\tilde{x}}\left(\int_x\pdata(x)\qnoise(\tilde{x}|x)dx\right)
     d\tilde{x} \\
     &\quad=\int_{\tilde{x}}\score(\tilde{x}, \sigma)^T\qdata(\tilde{x})
     \frac
     {\nabla_{\tilde{x}}\left(\int_x\pdata(x)\qnoise(\tilde{x}|x)dx\right)}
     {\qdata(\tilde{x})} d\tilde{x} \\
     &\quad=\int_{\tilde{x}}\score(\tilde{x}, \sigma)^T\qdata(\tilde{x})
     \frac
     {\nabla_{\tilde{x}}\qdata(\tilde{x})}
     {\qdata(\tilde{x})} d\tilde{x} \\
     &\quad=\int_{\tilde{x}}\score(\tilde{x}, \sigma)^T\qdata(\tilde{x})
     \nabla_{\tilde{x}}\log{\qdata(\tilde{x})}d\tilde{x} \\
     &\quad=\mathbb{E}_{\tilde{x}\sim \qdata}\left[\score(\tilde{x}, \sigma)^T \nabla_{\tilde{x}}\log{\qdata(\tilde{x})}\right]
\end{align*}
At this point, we can conclude that by setting $C=C_1-C_2$, we will have $\ell_{\text{direct}}(\sigma, S) = \ell_{\text{denoise}}(\sigma, S) + C$.
So optimizing either of $\ell_{\text{direct}}$ or $\ell_{\text{denoise}}$ will give the same optimal $\theta$, which concludes the proof.
\end{proof}
% Q.E.D.
% \qedsymbol

     % method explaination
\section{Implementation Details}\label{sec:impl-detail}
\subsection{Network architecture}\label{sec:net-architecture}

\paragraph{Auto-encoding}
For auto-encoding, our model takes 2D or 3D shape point cloud $X$ as input to the encoder, which follows the architecture proposed by \cite{pointnet}, and outputs the 128-dimensional latent code $z$ for each shape.

For the decoder, we train several noise level $\sigma$ at the same time. To condition on different noise level, we concatenate the noise level $\sigma$ at the end of latent code $z$. 
The input point cloud $X$ has 800 or 2048 points $x$ in total,
and each point is concatenated with latent code $z$ and $\sigma$ yielding a 131 or 132 dimensional input for the decoder (depending if the point cloud is in 2D or 3D).

Following the architecture used by OccNet \cite{mescheder2019occupancy}, first the input is scaled with a fully-connected layer to the hidden dimension 256. Then there are 8 pre-activation ResNet-blocks with 256 dimensions for every hidden layer, and each Res-block consists of two sets of Conditional Batch-Normalization (CBN), a ReLU activation layer and a fully-connected layer. The output of these two sets is added to the input of the Res-block. Then the output of all the Res-blocks is passed through another set of CBN, ReLU and FC layer, and this is the final output of the model -- a 3-dimensional vector describing the gradient for the input point.

The CBN layer takes the concatenated input as the latent code $\tilde{z}=[x, z, \sigma]$. The input $\tilde{z}$ is passed through two FC layers to output the 256-dimensional vectors $\beta(\tilde{z})$ and $\gamma(\tilde{z})$. The output of the CBN is computed according to:
\begin{align}
    f_{out}=\gamma(\tilde{z})\frac{f_{in}-\mu}{\sqrt{\sigma'^2+\epsilon}}+\beta(\tilde{z}),
\end{align}
where $\mu$ and $\sigma'$ are the mean and standard deviation of the batch feature data $f_{in}$. During training, $\mu$ and $\sigma'$ are the running mean with momentum 0.1, and they are replaced with the corresponding running mean at test time.
Figure \ref{fig:architecture} describes the architecture of our decoder $s_\theta$.

\paragraph{Generation}
For generation, based on the pretrained auto-encoder model, we use l-GAN to train the latent code generator.
Specifically, we train our GAN with WGAN-GP~\cite{iWGAN} objective.
We use Adam optimizer ($\beta_1=0.5, \beta_2=0.9$) with learning rate $10^{-4}$ for both the discriminator and the generator.
The latent-code dimension is set to be $256$.
The generator takes a $256$-dimensional noise vector sampled from $\mathcal{N}(\mathbf{0}, 0.2^2I_{256})$, and passes it through an MLP with hidden dimensions of $\{256, 256\}$ before outputting the final $256$-dimensional latent code.
We apply ReLU activation between layers and there is no batch normalization.
The discriminator is a three-layers MLP with hidden dimension $\{512, 512\}$.
We use LeakyReLU with slope $0.2$ between the layers.
We fixed both the pretrained encoder and decoder (i.e. set into evaluation mode).
The latent-GAN is trained for $5000$ epochs for each of the category.

\begin{figure}[ht]
\begin{center}
	\includegraphics[width=0.6\textwidth]{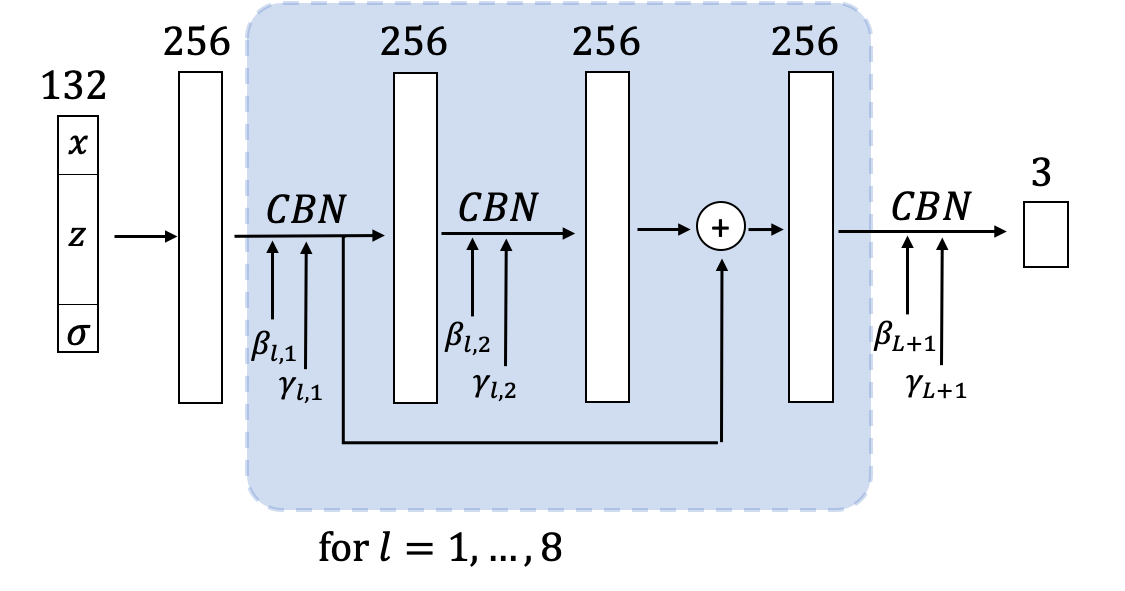}
\end{center}
\caption{The gradient decoder network $s_\theta$. For each layer, the spatial size is specified on top. The input point $x$ is concatenated to the latent code $z$ and noise level $\sigma$. The output of $s_\theta$ is the gradient corresponding to the given point.
 }
\label{fig:architecture}
\end{figure}
\subsection{Experimental setting}

For all the experiments, we use an Adam optimizer. We use ten different $\sigma$'s ranging from 1 to 0.01. For the ShapeNet dataset, the learning rates are $1\times10^{-4}$ for decoder and $1\times10^{-3}$ for encoder, with linear decay starting at 1000 epoch, reaching $1\times10^{-5}$ and $1\times10^{-4}$ for the encoder and the decoder, respectively. For the MNIST-CP dataset, the learning rate starts with $1\times10^{-3}$ for both the decoder and encoder, with linear decay starting at 1000 epoch, ending at $1\times10^{-4}$. 
Each batch consists of 64 shapes (or 200 shapes for the MNIST-CP dataset). We train the model for 2000 epochs. 
For inference, we set $T=10$ and $\alpha=2\times 10^{-4}$.

\subsection{Baselines}
In this section, we provide more details on how we obtain the reported scores for alternative methods.
\paragraph{PointFlow~\cite{yang2019pointflow}}
To get the results for ShapeNet, we run the pre-trained checkpoint release in the official code repository\footnote{\url{https://github.com/stevenygd/PointFlow}}. 
The auto-encoding results for MNIST is obtained by running the released code on the pre-processed MNIST-CP dataset.

\paragraph{AtlasNet~\cite{atlasnet}}
The code we used for the AtlasNet decoder comes from the official code repository\footnote{\url{https://github.com/ThibaultGROUEIX/AtlasNet}}. 
To enable a fair comparison, we use the same encoder as used in the PointFlow repository, and set the latent code dimension to be the same as our own model.
We use the suggested learning rate and optimizer setting from the AtlasNet code-base and paper during training.
We train AtlasNet for the same amount of iterations as our method to obtain the reported performance in Table 1 in the main paper.

\paragraph{l-GAN and r-GAN~\cite{achlioptas_L3DP}}
We modify the official released code repository\footnote{\url{https://github.com/optas/latent_3d_points}} to take our pre-processed point cloud from ShapeNet version 2~\cite{shapenet,chang2015shapenet}.
The auto-encoding results (i.e. Table 1 in main paper) for l-GAN and r-GAN are obtained by running the official code for the same number of iterations as our model.
The generation results for r-GAN in Table 3 in the main paper is obtained by running the latent-GAN in the official code for the default amount of iterations in the configuration.

\paragraph{GraphCNN-GAN~\cite{valsesia2018learning}} We use the official code released in this repository to obtain the results: \url{https://github.com/diegovalsesia/GraphCNN-GAN}.

\paragraph{ TreeGAN~\cite{shu20193d}} We use the official code released in this repository to obtain the results: \url{https://github.com/seowok/TreeGAN.git}.
 % implementation details

\section{Additional Quantitative Results}\label{sec:quant_result}
\subsection{Shape generation}
We compare our method's performance on shape generation with GraphCNN-GAN~\cite{valsesia2018learning} and TreeGAN~\cite{shu20193d} in Table \ref{tab:generation_supp} (in addition to the comparisons to r-GAN~\cite{achlioptas2017learning} and PointFlow~\cite{yang2019pointflow} which we report in the main paper). 
As discussed in the Related Work section (Section 2 in the main paper), both of these two baselines treat generating point clouds with $N$ points as predicting a fixed dimensional vector (in practice predicting a $N\times3$ vector but they could potentially use more upsampling layers to predict more points), using the same discriminator as r-GAN~\cite{achlioptas_L3DP}.
These works report performance for models trained on smaller collections (i.e. the ShapeNet Benchmark dataset\footnote{\url{https://shapenet.cs.stanford.edu/ericyi/shapenetcore_partanno_segmentation_benchmark_v0.zip}}) using different splits and normalization. 
Therefore, in addition to comparing their publicly available models (in the first rows), we retrain their models on the full ShapeNet collections using the same splits and preprocessing performed on our trained models (in the rows marked with an asterisk ($*$)). 
For each model, the training is performed over two days with a GeForce GTX TITAN X GPU.

\begin{table*}[h]
\centering
\caption{Additional shape generation results. Rows marked with an asterisk ($*$) denote retrained models. $\uparrow$ means the higher the better, $\downarrow$ means the lower the better. 
MMD-CD is multiplied by $10^3$ and MMD-EMD is multiplied by $10^2$.}
\label{tab:generation_supp}
\begin{tabularx}{\textwidth}{LLYYYYYY}
\toprule
&& \multicolumn{2}{c}{MMD ($\downarrow$)} &  \multicolumn{2}{c}{COV (\%, $\uparrow$)} &  \multicolumn{2}{c}{1-NNA (\%, $\downarrow$)} \\ 
\cmidrule(lr){3-4} \cmidrule(lr){5-6} \cmidrule(l){7-8} 
Category                  & Model     & CD         & EMD        & CD         & EMD        & CD          & EMD         \\ 
\midrule
\multirow{6}{*}{Airplane} 
& GCN~\cite{valsesia2018learning}                  & 2.623     & 15.535    & 9.38       & 5.93      & 95.16      & 99.12        \\
& GCN~\cite{valsesia2018learning}$^*$              & 44.93   & 35.52   & 1.98    & 1.23   & 99.99    & 99.99         \\
& Tree~\cite{shu20193d}                             & 1.466     & 16.662    & 44.69      & 6.91     & 95.06      & 100.00         \\
& Tree~\cite{shu20193d}$^*$                         & 1.798     & 24.723    & 31.60      & 5.43      & 95.43      & 99.88         \\
& Ours            & \textbf{1.285}  & \textbf{7.364}          & \textbf{47.65}          & \textbf{41.98}  & \textbf{85.06}          & \textbf{83.46}\\
\cmidrule(l{-1pt}r{0pt}){2-8} 
& Train           & 1.288           & 7.036         & 45.43          & 45.43          & 72.10          & 69.38          \\
\midrule
\multirow{4}{*}{Chair}    
& GCN~\cite{valsesia2018learning}                      & 23.098          & 25.781        & 6.95          & 6.34          & 86.52       & 96.48    \\
& GCN~\cite{valsesia2018learning}$^*$         & 140.84         & 0.5163        & 1.67          & 1.06          & 100       & 100         \\
& Tree~\cite{shu20193d}                                 & 16.147          & 36.545       & 40.33         & 8.76           & 74.55          & 99.92         \\
& Tree~\cite{shu20193d}$^*$                             & 17.124          & 26.405       & 42.90         & 20.09          & 77.49          & 98.11         \\
& Ours                                              & \textbf{14.818} & \textbf{18.791}  & \textbf{46.37} & \textbf{46.22} & \textbf{66.16} & \textbf{59.82} \\
\cmidrule(l{-1pt}r{0pt}){2-8} 
& Train           & 15.893          & 18.472          & 50.45          & 52.11          & 53.93          & 54.15           
\\ 
\bottomrule
\end{tabularx}
\end{table*}

We also present generation results for the car category in Table~\ref{tab:generation-other-cates}.
Our model achieves performance that's un-par with the state-of-the-arts.

\begin{table*}[h!]
\centering
\caption{Shape generation results. $\uparrow$ means the higher the better, $\downarrow$ means the lower the better. 
MMD-CD is multiplied by $10^3$ and MMD-EMD is multiplied by $10^2$.}
\label{tab:generation-other-cates}
\begin{tabularx}{\textwidth}{LL*{2}{Y}l*{2}{Y}l*{2}{Y}}
\toprule
                          &           & \multicolumn{2}{c}{MMD ($\downarrow$)} &  & \multicolumn{2}{c}{COV (\%, $\uparrow$)} &  & \multicolumn{2}{c}{1-NNA (\%, $\downarrow$)} \\ \cmidrule(lr){3-4} \cmidrule(lr){6-7} \cmidrule(l){9-10} 
Category                  & Model     & CD         & EMD        &  & CD         & EMD       &  & CD     & EMD         \\ \midrule
\multirow{4}{*}{Car}      & rGAN~\cite{achlioptas_L3DP}           & 6.233           & 18.561          &  & 8.24           & 5.11           &  & 99.29          & 99.86          \\
                          & PF~\cite{yang2019pointflow} & 4.207           & 10.631          &  & 39.20          & 44.89          &  & 68.75          & \textbf{62.64} \\
                          & Ours            & \textbf{4.085}  & \textbf{10.610} &  & \textbf{44.60} & \textbf{46.88} &  & \textbf{65.48} & 62.93          \\\cmidrule(l{-1pt}r{0pt}){2-10} 
                          & Train           & 4.207           & 10.631          &  & 48.30          & 54.26          &  & 52.98          & 49.57         
\\ 
\bottomrule
\end{tabularx}
\end{table*}

In Figure \ref{fig:convergence}, we show convergence curves for our method and for PointFlow \cite{yang2019pointflow} on the auto-encoding task (over the Airplane category). As the figure illustrates, our method converges much faster and to a better result. 
\begin{figure*}[h]
\begin{center}
	\includegraphics[width=\textwidth]{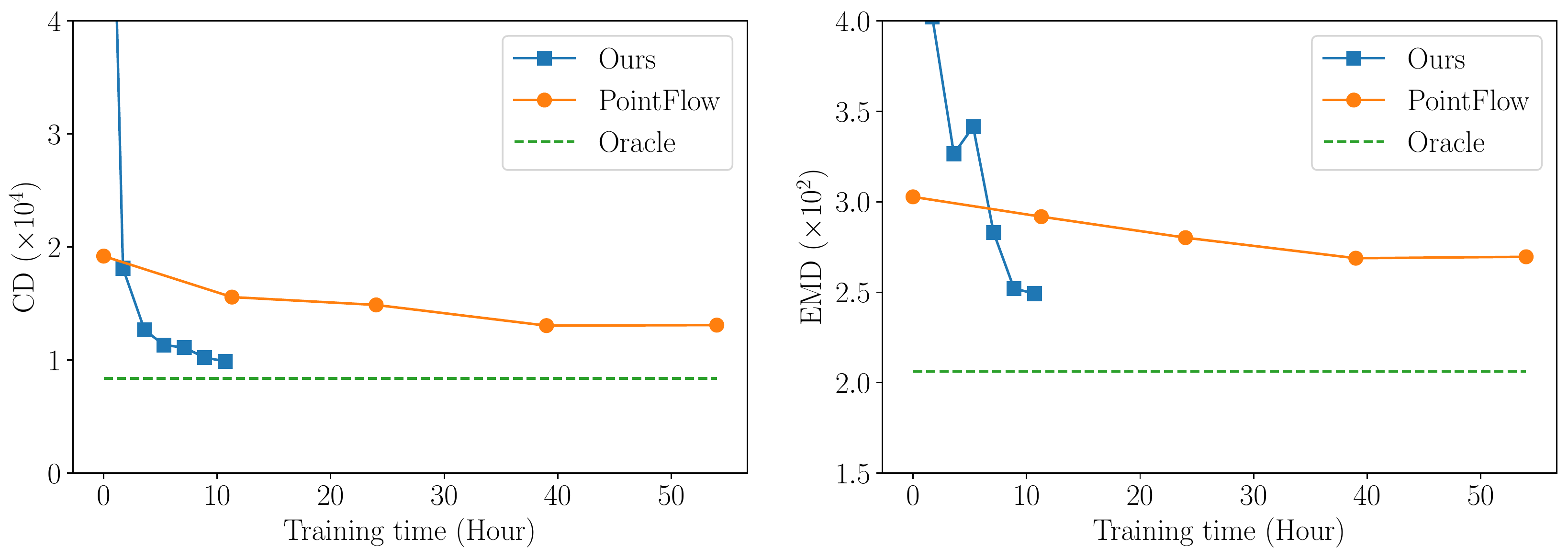}
\end{center}
\caption{Convergence curves for PointFlow~\cite{yang2019pointflow} and our method on the auto-encoding task. As illustrated above, 
our method converges much faster to a better result.
}
\label{fig:convergence}
\end{figure*}

\subsection{Implicit surface}
\begin{table*}[ht]
\centering
\caption{Implicit surface results on Airplane category. 
CD is multiplied by $10^4$ and EMD is multiplied by $10^2$.}
\label{tab:implicit}
\begin{tabularx}{\textwidth}{LYYYY}
\toprule
Metrics          & AtlasNet-Sph. & AtlasNet-25 & DeepSDF        & Ours \\ \midrule
CD & 1.88           & 2.16         & 1.43           & \textbf{1.022}             \\
CD (median) & 0.79           & 0.65         & \textbf{0.36}  & 0.442                      \\
EMD & 3.8            & 4.1          & \textbf{3.3}   & 5.545                      \\
Mesh acc & 0.013          & 0.013        & \textbf{0.004} & 0.008                                         \\
\bottomrule
\end{tabularx}
\end{table*}

In this section, we demonstrate that one can use MISE~\cite{mescheder2019occupancy}, an octree-based marching cue algorithm, to extract a ground truth mesh that from the learned gradient field.
We compute ground truth meshes for the test set of the airplane category following DeepSDF's set-up.
As mentioned in Section 1 of main paper, prior implicit representations \cite{chen2019learning}\cite{mescheder2019occupancy}\cite{park2019deepsdf} require knowing the ground truth meshes in order to provide a supervision signal during training, while our model can be trained solely from sparse point clouds. 
We conducted a preliminary quantitative comparison between our implicit surface and that of DeepSDF\cite{park2019deepsdf}. 
We follow DeepSDF’s experiment set-up to report results on the airplane category in Table\ref{tab:implicit}.
Our implicit surfaces outperforms AtlasNet in both CD and Mesh accuracy metrics and are competitive with DeepSDF (which uses more supervision) in CD.
Failure cases for our extracted meshes usually comes from the bifurcation area (i.e. the local minimums and saddle points) where gradients are close to zero.
Another problem with our extracted surface is that marching cue tend to create a double surface around the shape,
As our focus is generating point clouds, we will leave the improvement of surface extraction to future work.
 % quantitative results
\section{Additional Ablation Studies}\label{sec:ablation_studies}
Next we report results of additional ablation studies to evaluating several design considerations and our choice in modeling the distribution of shapes.

\paragraph{Network architecture} \label{sec:ablation-network}
\begin{table*}[t]
\centering
\caption{Architecture ablation study, comparing auto-encoding performance on the Airplane category. CD is multiplied by $10^4$ and EMD is multiplied by $10^2$. The ablated models are detailed in the text.
}
\label{tab:architecture-ablation}
\begin{tabularx}{0.99\textwidth}{l*{6}{Y}}
\toprule
Metrics & (a)     & (ab)    & (abc)   & (bd)    & (abcd)  \\ 
\midrule
CD      & 1.234$\pm$0.007 & 0.992$\pm$0.002 & 0.998$\pm$0.003 & 1.011$\pm$0.008 & 0.987$\pm$0.001 \\
EMD     & 2.718$\pm$0.039 & 2.513$\pm$0.019 & 2.493$\pm$0.010 & 2.462$\pm$0.042 & 2.524$\pm$0.001 \\
\bottomrule
\end{tabularx}
\end{table*}
We evaluate different architectures considering the following:
\begin{enumerate}[(a)]
    \item Replacing BN with CBN.
    \item Adding shortcuts.
    \item Replacing the latent code $z$ with $\tilde{z}=[x,z,\sigma]$ for the CBN layer.
    \item Concatenating the latent code $z$ and $\sigma$ to $x$ as input for the decoder.
\end{enumerate}
In Table \ref{tab:architecture-ablation} we report multiple configurations, with the rightmost one (abcd) corresponding to our full model. For each model, we perform 3 inference runs, and report the average and the standard deviation over these runs. As the table illustrates, our full model yields better performance as well as significantly smaller variance across different runs.

\paragraph{Modeling the distribution of shapes} \label{sec:ablation-gen}
In our work, we propose a new approach of modeling the distribution of points using the gradient of the log density field. To model the distribution of shapes, we use a latent GAN \cite{achlioptas2017learning}. Next, we explore a different method to model the distribution of shapes.
Specifically, we train a VAE using the same encoder and decoder setting as the one used in Section 4.2 in the main paper. 
We double the output dimension for the encoder so that it can output both $\mathbf{\mu}$ and $\mathbf{\sigma}$ for the re-parameterization.
We add the KL-divergence loss with weight $10^{-3}$ and train for the same amount of time (in terms of epochs and iterations) as the model reported in the main paper.
The results for the Chair and Airplane categories are reported in Table~\ref{tab:ablation-generation}.
We can see that our reported model, which uses a two-stage training with a latent-GAN, outperforms the model trained with VAE for both the Chair and the Airplane category on all metrics, except for the 1-NNA-EMD metrics on Airplane.
\begin{table*}[h]
\centering
\caption{Modeling the distribution of shapes using different techniques. We compare our model's performance on the generation task against a VAE model. $\uparrow$ means the higher the better, $\downarrow$ means the lower the better. 
MMD-CD is multiplied by $10^3$ and MMD-EMD is multiplied by $10^2$.}
\label{tab:ablation-generation}
\begin{tabularx}{\textwidth}{LLYYYYYY}
\toprule
&& \multicolumn{2}{c}{MMD ($\downarrow$)} & \multicolumn{2}{c}{COV (\%, $\uparrow$)} & \multicolumn{2}{c}{1-NNA (\%, $\downarrow$)} \\ 
\cmidrule(lr){3-4} \cmidrule(lr){5-6} \cmidrule(l){7-8} 
Category & Model    & CD         & EMD        & CD         & EMD        & CD          & EMD         \\ \midrule
\multirow{3}{*}{Airplane} 
& VAE     & 1.909             & 9.004             & 37.78           & 38.27             & 89.14            & \textbf{86.05} \\
& Ours    & \textbf{1.332}    & \textbf{7.768}    & \textbf{39.01}  & \textbf{43.46}    & \textbf{88.52}   & 86.91         \\
\cmidrule(l{-1pt}r{0pt}){2-8} 
& Train           & 1.288           & 7.036         & 45.43          & 45.43          & 72.10          & 69.38          \\
\midrule
\multirow{3}{*}{Chair} 
& VAE    & 18.032           & 20.903            & 41.99             & 43.81             & 74.17             & 74.92 \\
& Ours   & \textbf{14.818}  & \textbf{18.791}   & \textbf{46.37}    & \textbf{46.22}    & \textbf{66.16}    & \textbf{59.82} \\ 
\cmidrule(l{-1pt}r{0pt}){2-8} 
& Train           & 15.893          & 18.472          & 50.45          & 52.11          & 53.93          & 54.15         \\
\bottomrule
\end{tabularx}
\end{table*}
 % ablation results
\section{Additional Qualitative Results}\label{sec:qual_result}
\begin{figure}[ht]
    \centering
    	\jsubfig{\includegraphics[height=2.4cm]{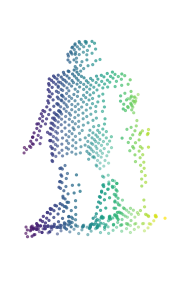}
	\includegraphics[height=2.4cm]{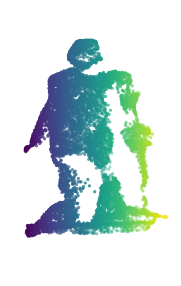}
	\includegraphics[height=2.4cm]{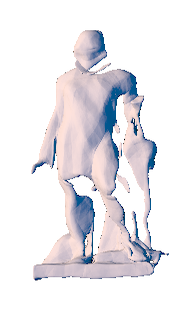}
	}
	{}%
	%\hfill
    \hspace{17pt}
	\jsubfig{\includegraphics[height=2.4cm]{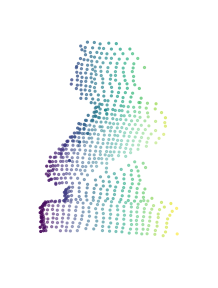}
	\hspace{3pt}
	\includegraphics[height=2.4cm]{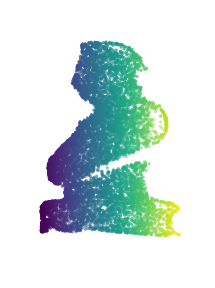}
	\includegraphics[height=2.4cm]{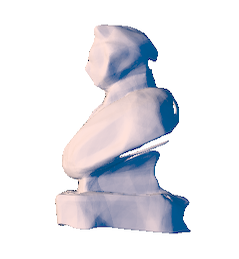}
	}
	{}%
	\hfill
	\\
	\hfill
		\jsubfig{\includegraphics[height=2.39cm]{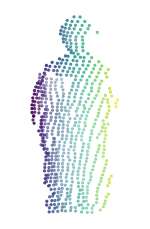} %\hspace{3pt}
	\includegraphics[height=2.4cm]{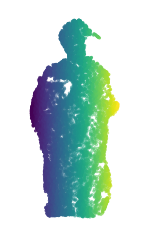}
	%\hspace{3pt}
	\includegraphics[height=2.4cm]{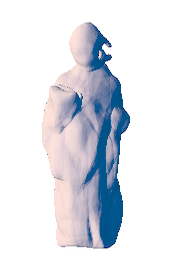}}
	{}%
	\hfill%
	\hspace{5pt}
	\jsubfig{\includegraphics[height=2.4cm]{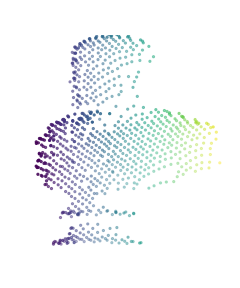}
	\includegraphics[height=2.4cm]{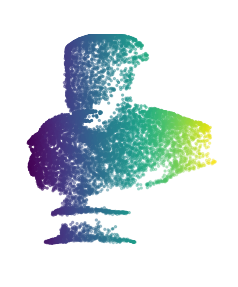}
	\includegraphics[height=2.4cm]{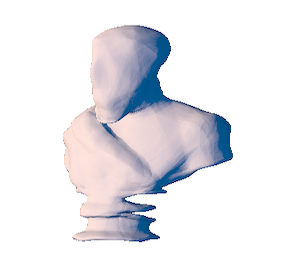}}
	{}%
    \caption{Autoencoding scanned shapes. Above we demonstrate our technique on sparse point clouds captured with a 3D scanner (left). We sample 10K points to obtain the point clouds illustrated in the middle and also extract the implicit surfaces (right). }
    \label{fig:scans}
\end{figure}
\paragraph{Scanned data}
To demonstrate that our technique can also model partial and incomplete shapes, we use scanned point clouds captured using a hand-held 3D scanner, as detailed in Yifan et al.\cite{yifan2019patch}. For each scanned object, we train a separate model, evaluating to what extent we can obtain a high-fidelity point cloud reconstruction for these scanned objects. See Figure \ref{fig:scans} for a qualitative comparison of the input scanned objects which contain roughly 600 points (left columns), reconstructed point clouds (middle columns) and extracted implicit surfaces (right columns). While we cannot model the full shape in this case (as we are only provided with a partial, single view), our technique enables reconstructing a denser representation even in this sparse real setting.

\begin{figure}[ht]
    \centering
    \newcommand{\sizea}{0.3\linewidth}
    \newcommand{\tal}{0cm}
    \newcommand{\tab}{0cm}
    \newcommand{\tar}{0cm}
    \newcommand{\tat}{0cm}
    \newcommand{\tcl}{0cm}
    \newcommand{\tcb}{0cm}
    \newcommand{\tcr}{0cm}
    \newcommand{\tct}{0cm}
    \newcommand{\thl}{0cm}
    \newcommand{\thb}{0cm}
    \newcommand{\thr}{0cm}
    \newcommand{\tht}{0cm}
    \begin{tabular}{ccc}
    \includegraphics[width=\sizea,trim={\tal} {\tab} {\tar} {\tat},clip]{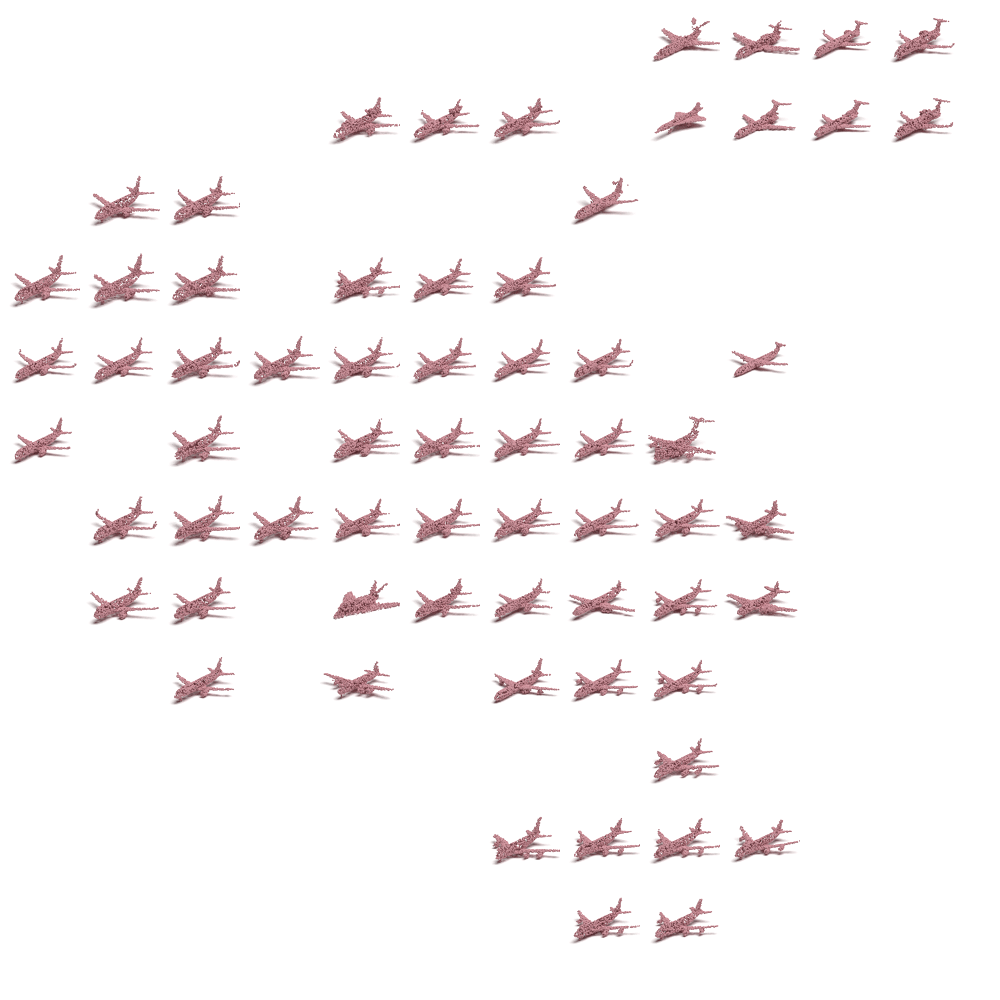}&
    \includegraphics[width=\sizea,trim={\tcl} {\tcb} {\tcr} {\tct},clip]{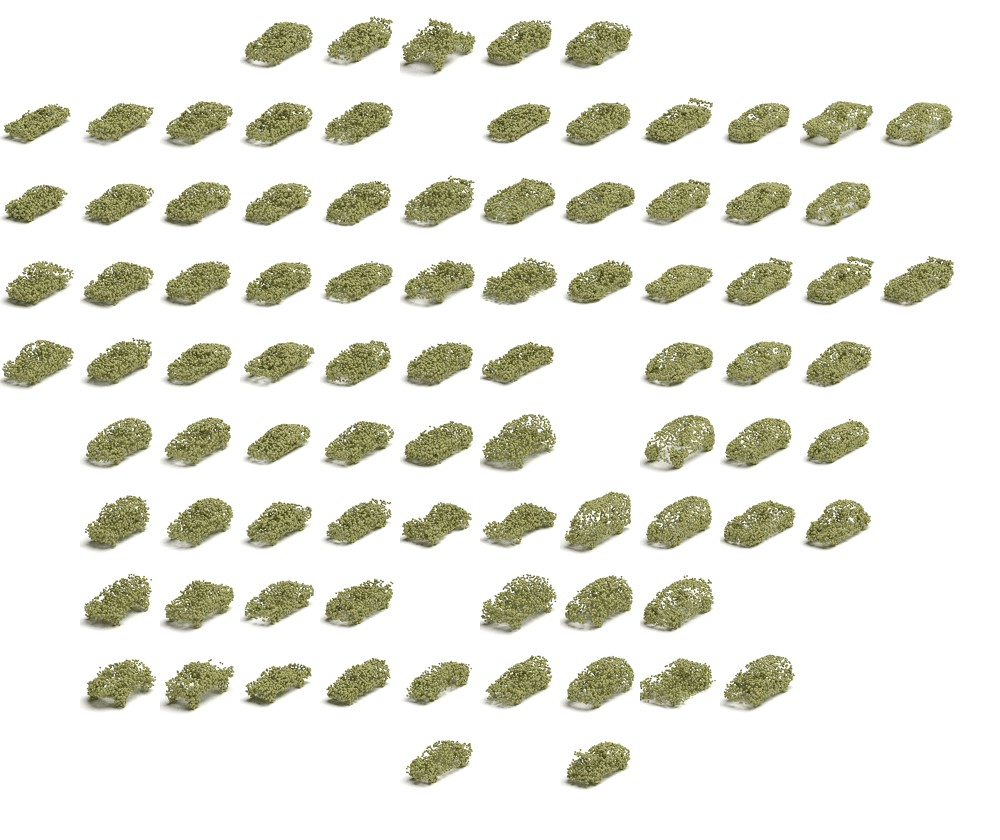}&
    \includegraphics[width=\sizea,trim={\thl} {\thb} {\thr} {\tht},clip]{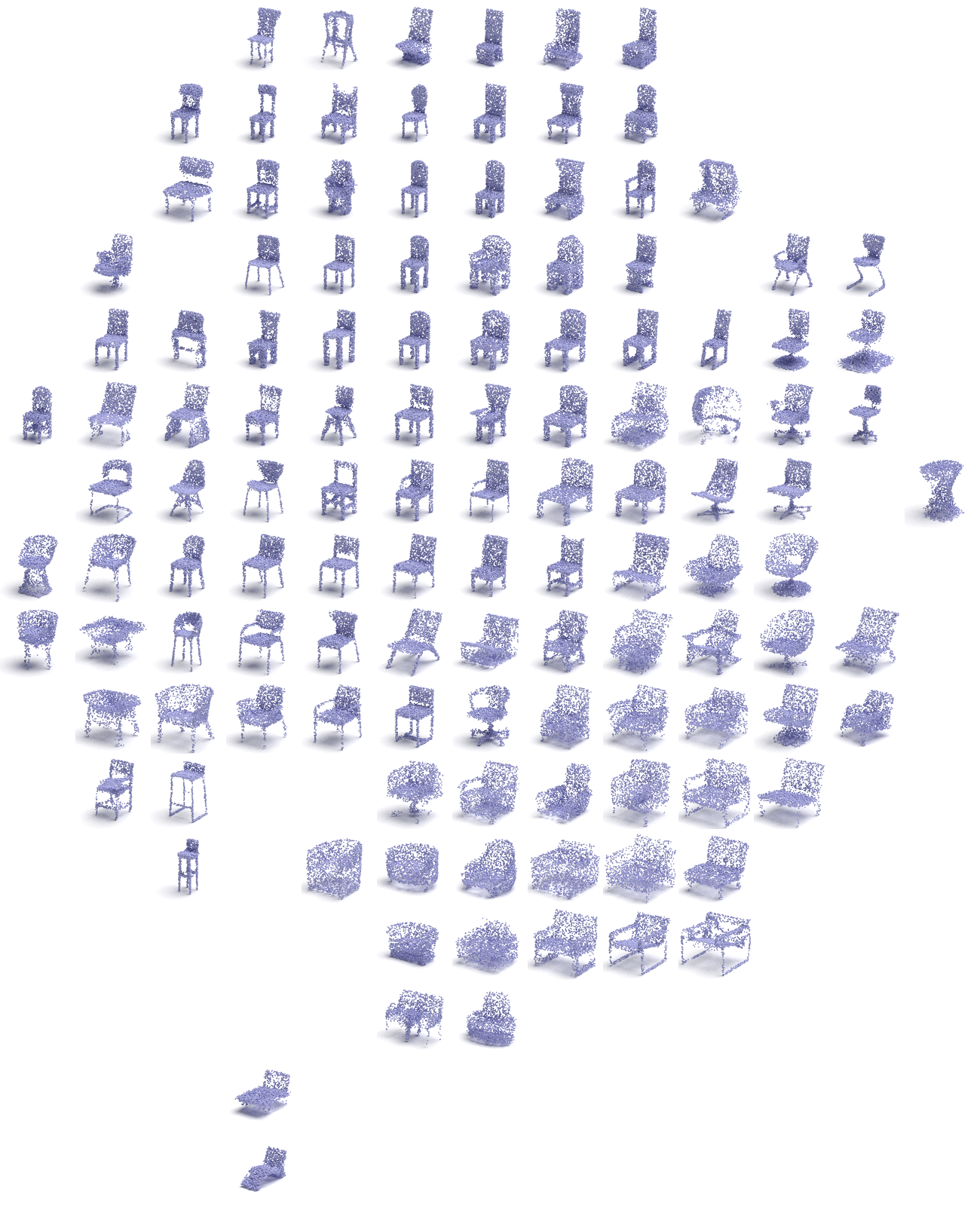}
    \end{tabular}
    \caption{Visualization of the latent space. 
    }
    \label{fig:latent_code}
\end{figure}
\paragraph{Visualization of latent space}
We visualize the latent code space in Figure \ref{fig:latent_code}. 
We first obtain all latent code  $z\in \mathbb{R}^{128}$ by running the encoder on the validation set of Airplane, Car, and Chair.
We run T-SNE~\cite{maaten2008visualizing} on the latent code for the same category to reduce the latent-codes' dimensionality down to $2$. 
These $2$-dimensional latent codes are then used to place rendered reconstructed point clouds on the figure.
The figure shows that the latent code places similar shapes nearby in the latent space, which suggests that we learn a meaningful latent space.

\paragraph{Extended visualizations for 2D and 3D point clouds}
In Figure \ref{fig:inference_mnist_samples_supp}, we demonstrate our annealed Langevin dynamic procedure for 2D point clouds from MNIST-CP. We also show that our model is insensitive to the choice of the prior distribution in the 2D case in Figure \ref{fig:prior}. We show more results on the ShapeNet dataset in Figure \ref{fig:reconstruction} (auto-encoding shapes), Figure \ref{fig:extra_recon}(auto-encoding point clouds), Figure \ref{fig:generation_supp} (shape generation), and Figure \ref{fig:gen_interp} (shape interpolation).

\begin{figure}[ht]
\centering
     \includegraphics[width=1.0\linewidth]{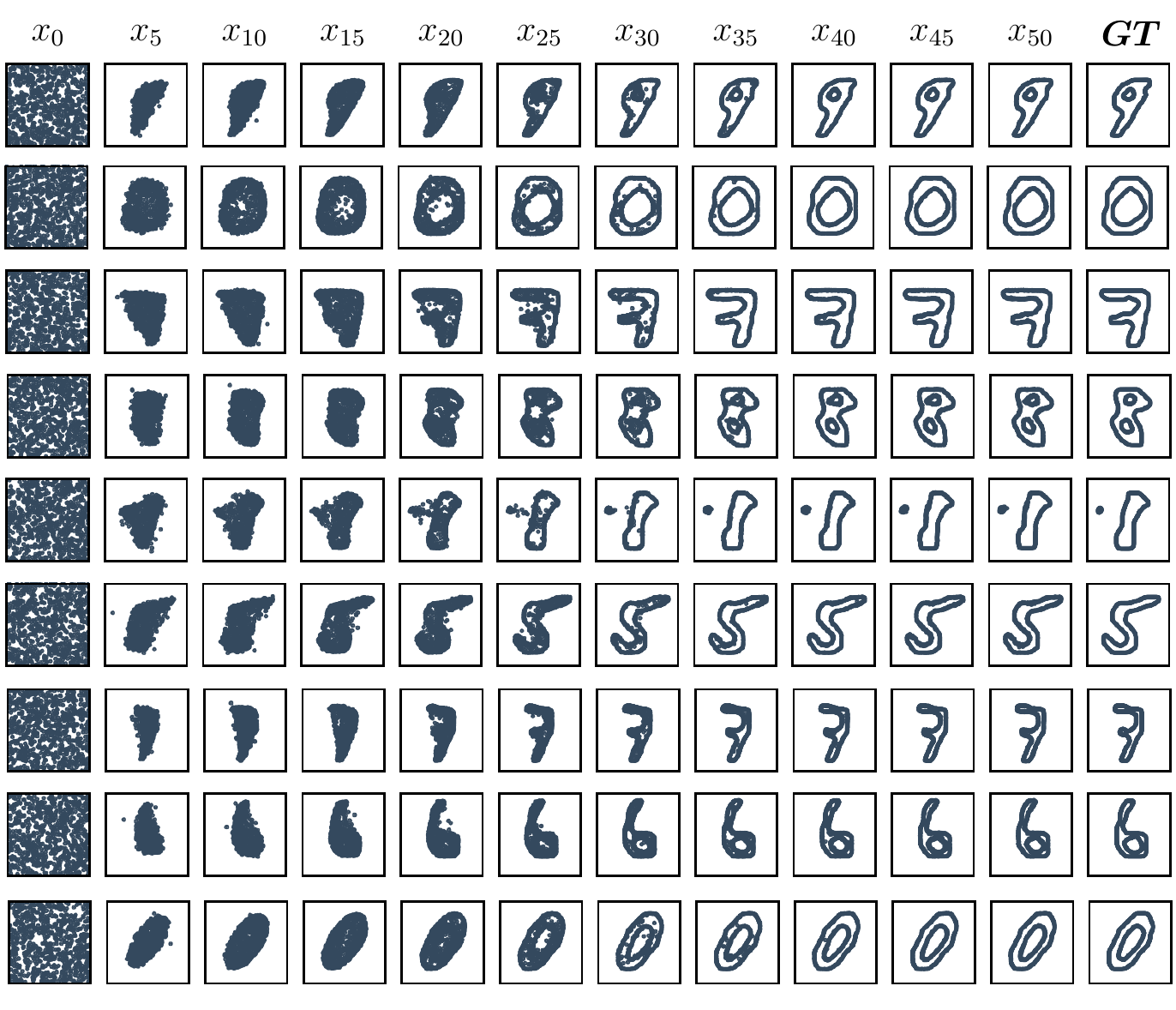}
    \caption{Point cloud sampling on the MNIST-CP dataset. Above we illustrate our annealed Langevin dynamics procedure for 2D shapes. Starting with points sampled from a uniform distribution, the points gradually move along the logarithmic density field. As illustrated on the right side, eventually these points are mostly indistinguishable from the ground truth point clouds.      }
    \label{fig:inference_mnist_samples_supp}
\end{figure}
\begin{figure}[H]

    \rotatebox[origin=l]{90}{\small{\whitetxt{s}Results (different priors)  \whitetxt{sssss}GT}}
    \hfill
    \includegraphics[width=0.98\textwidth,trim={1.2cm 0 0 0},clip]{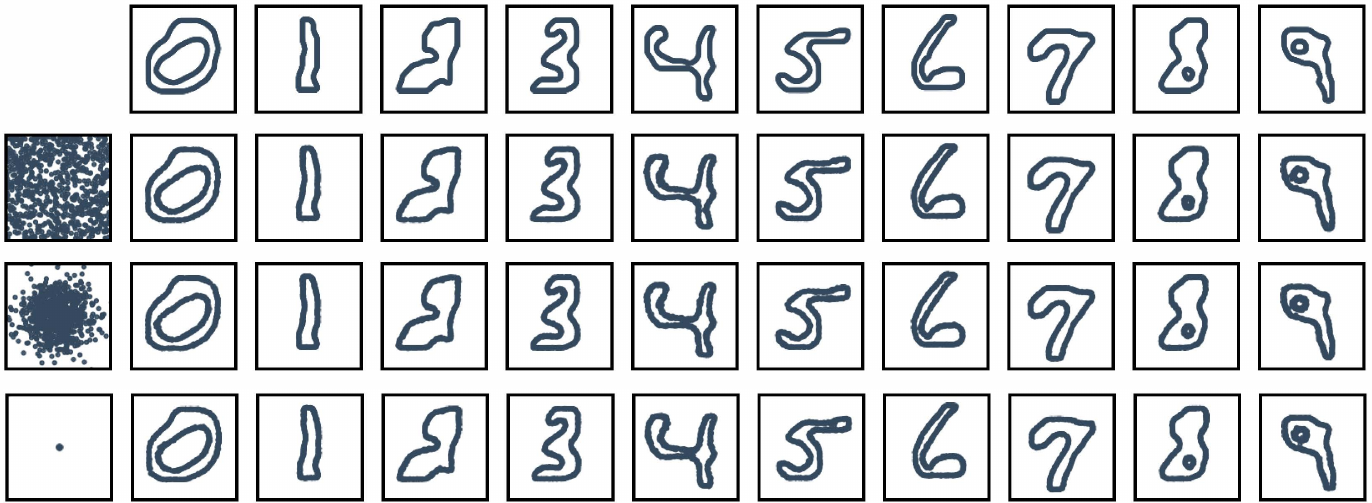}
    \caption{Reconstructing MNIST-CP shapes (illustrated on top), starting from different prior distributions. Above we demonstrate reconstruction results obtained starting from a uniform (second row), Gaussian (third row) or a single point (fourth row) distribution. As the figure illustrates, our method is insensitive to the prior distribution.}
    \label{fig:prior}
\end{figure}

\begin{figure}
    \begin{center}
    \newcommand{\sizea}{0.16\linewidth}
    \includegraphics[width=\sizea,trim={3cm 1cm 4cm 0.5cm},clip]{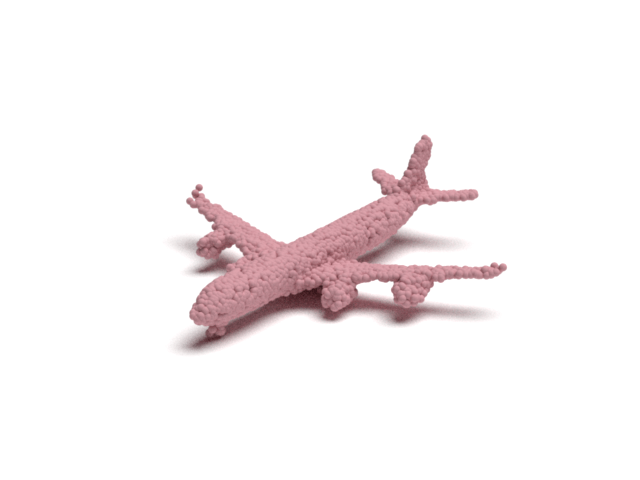}
     \includegraphics[width=\sizea,trim={3cm 1cm 4cm 0.5cm},clip]{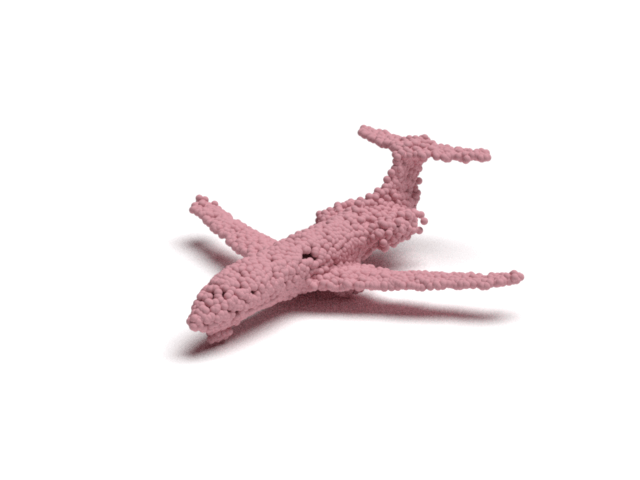}
     \includegraphics[width=\sizea,trim={3cm 1cm 4cm 0.5cm},clip]{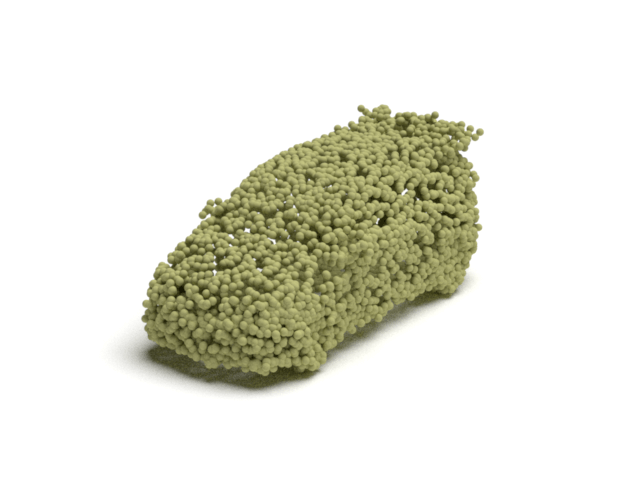}
     \includegraphics[width=\sizea,trim={3cm 1cm 4cm 0.5cm},clip]{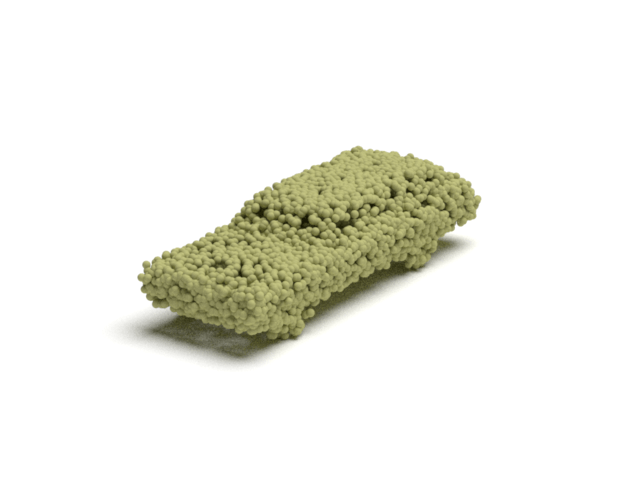}
     \includegraphics[width=\sizea,trim={3cm 1cm 4cm 0.5cm},clip]{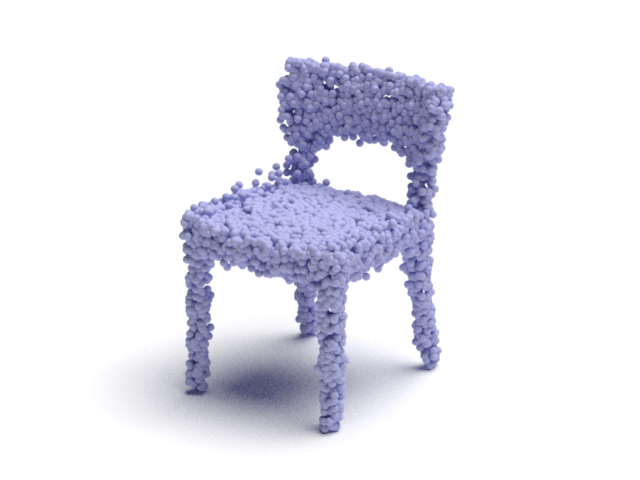}
    \includegraphics[width=\sizea,trim={3cm 1cm 4cm 0.5cm},clip]{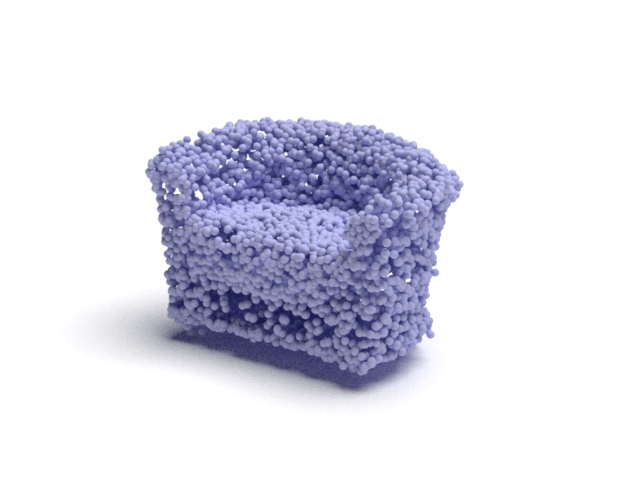}
     \\
    \includegraphics[width=\sizea,trim={3cm 1cm 4cm 0.5cm},clip]{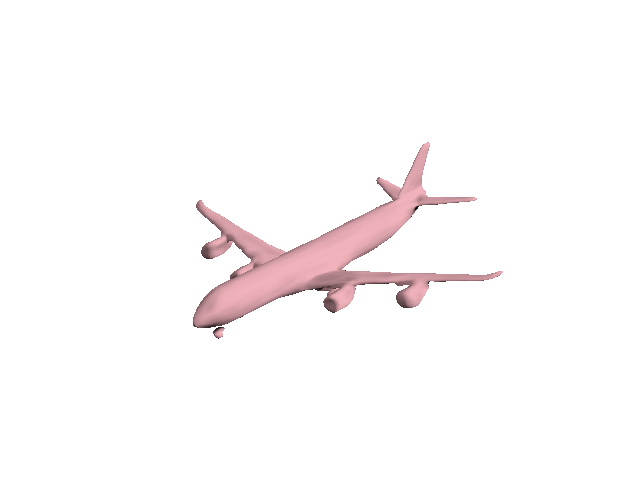}
     \includegraphics[width=\sizea,trim={3cm 1cm 4cm 0.5cm},clip]{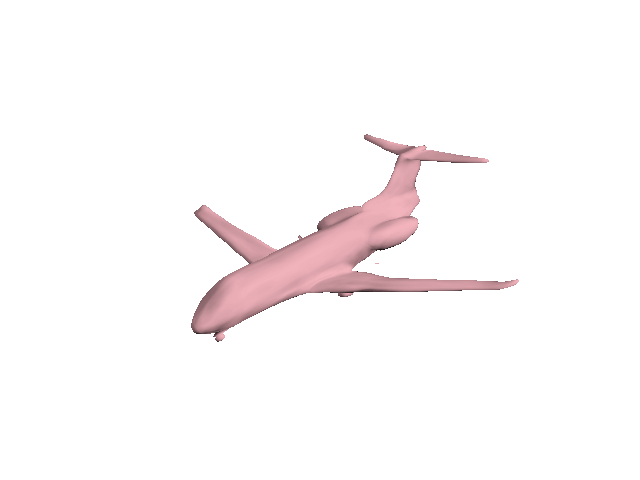}
     \includegraphics[width=\sizea,trim={3cm 1cm 4cm 0.5cm},clip]{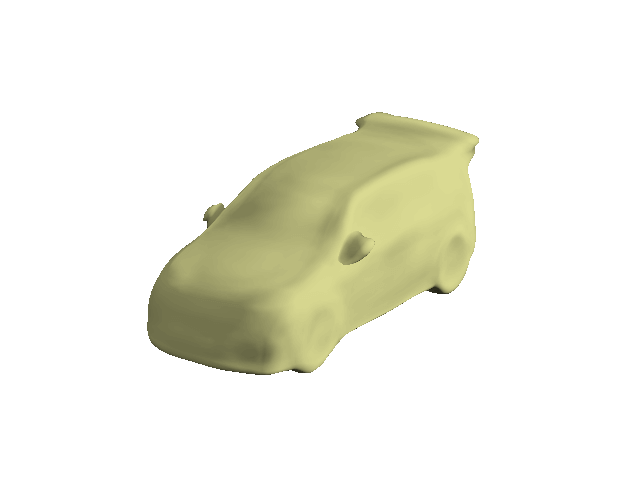}
     \includegraphics[width=\sizea,trim={3cm 1cm 4cm 0.5cm},clip]{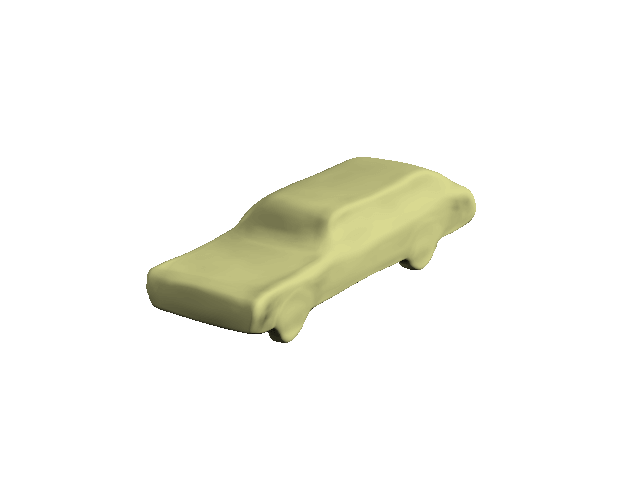}
    \includegraphics[width=\sizea,trim={3cm 1cm 4cm 0.5cm},clip]{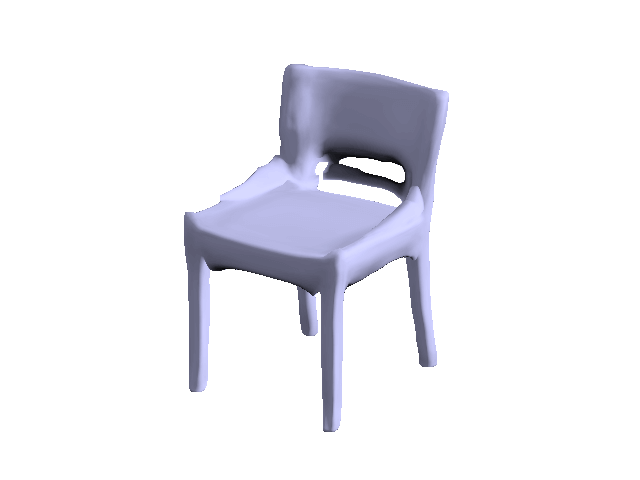}
    \includegraphics[width=\sizea,trim={3cm 1cm 4cm 0.5cm},clip]{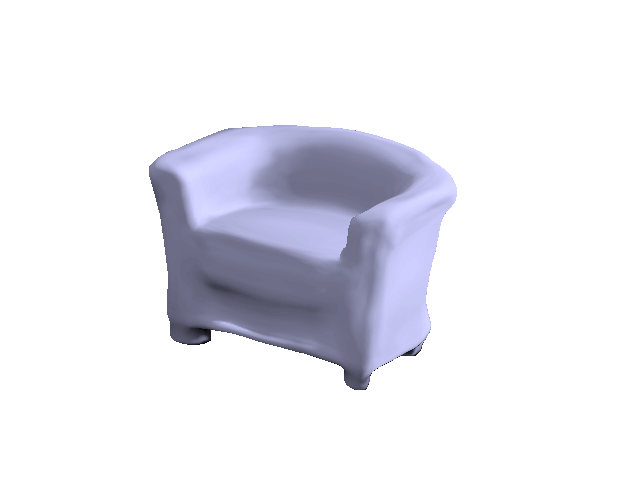}
    \end{center}
    \caption{Examples of reconstruction results. The first row depicts reconstructed shapes from the Airplane, Car and Chair categories. The second row is the corresponding implicit surfaces.
    }
    \label{fig:reconstruction}
\end{figure}
% <left> <lower> <right> <upper>}
\begin{figure}
    \begin{center}
    \newcommand{\sizea}{0.16\linewidth}
    \newcommand{\tal}{5.0cm}
    \newcommand{\tab}{5cm}
    \newcommand{\tar}{5.0cm}
    \newcommand{\tat}{5cm}
    \newcommand{\tcl}{4.0cm}
    \newcommand{\tcb}{4cm}
    \newcommand{\tcr}{4.0cm}
    \newcommand{\tct}{4cm}
    \newcommand{\thl}{1.5cm}
    \newcommand{\thb}{0.5cm}
    \newcommand{\thr}{2.0cm}
    \newcommand{\tht}{0.5cm}
    \includegraphics[width=\sizea, trim={\tal} {\tab} {\tar} {\tat},clip]{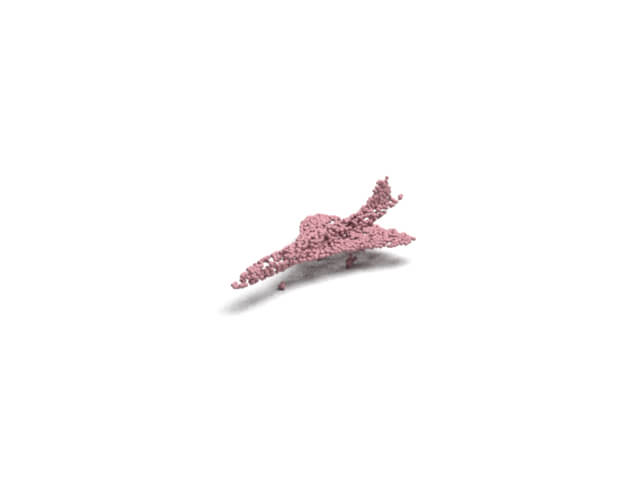}
    \includegraphics[width=\sizea, trim={\tal} {\tab} {\tar} {\tat},clip]{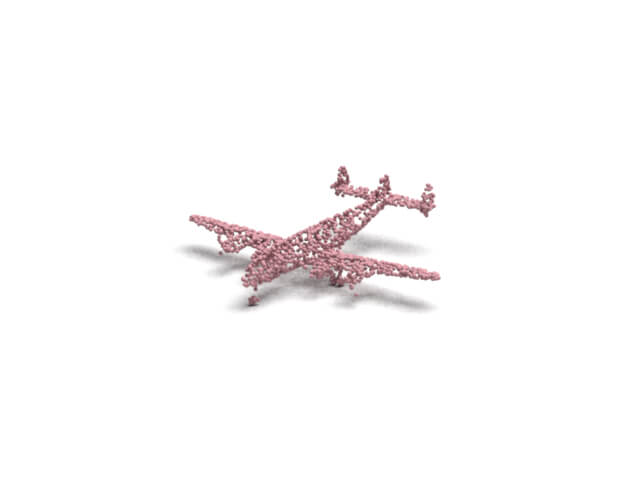}
    \includegraphics[width=\sizea, trim={\tal} {\tab} {\tar} {\tat},clip]{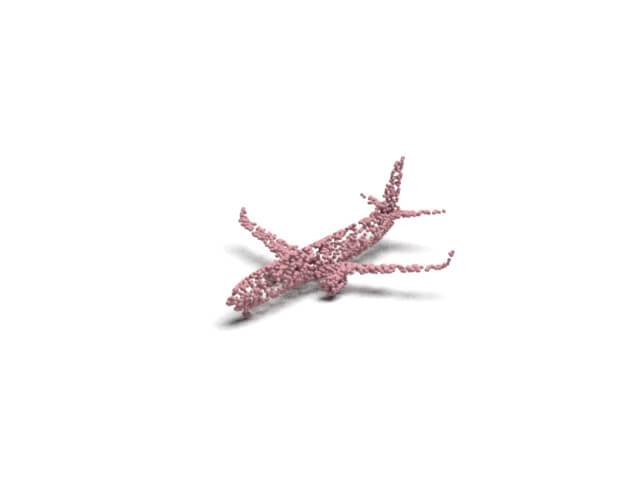}
    \includegraphics[width=\sizea, trim={\tal} {\tab} {\tar} {\tat},clip]{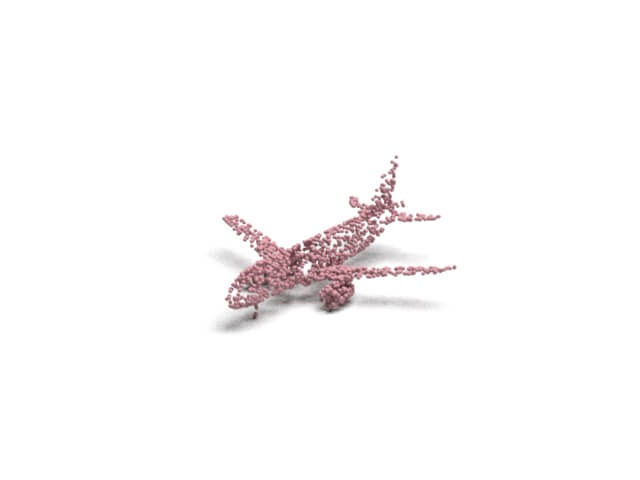}
    \includegraphics[width=\sizea, trim={\tal} {\tab} {\tar} {\tat},clip]{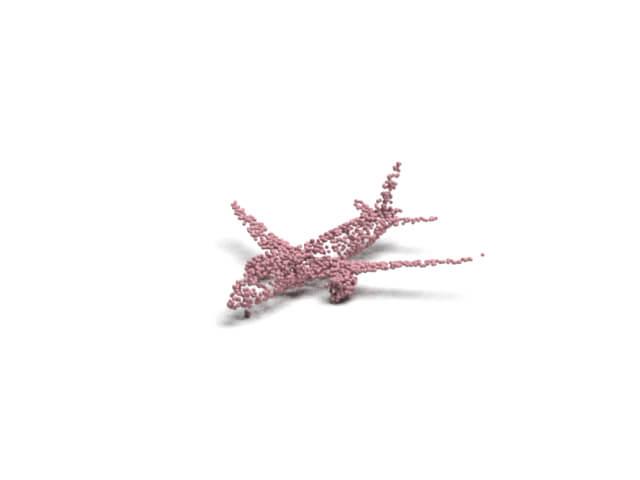}
    \includegraphics[width=\sizea, trim={\tal} {\tab} {\tar} {\tat},clip]{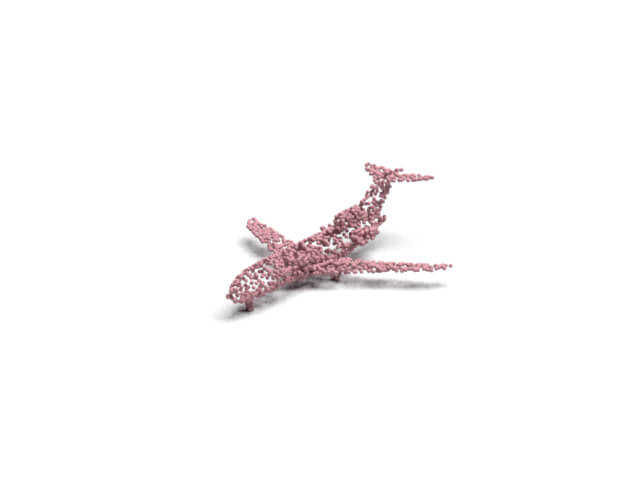}
    \\
    \includegraphics[width=\sizea, trim={\tal} {\tab} {\tar} {\tat},clip]{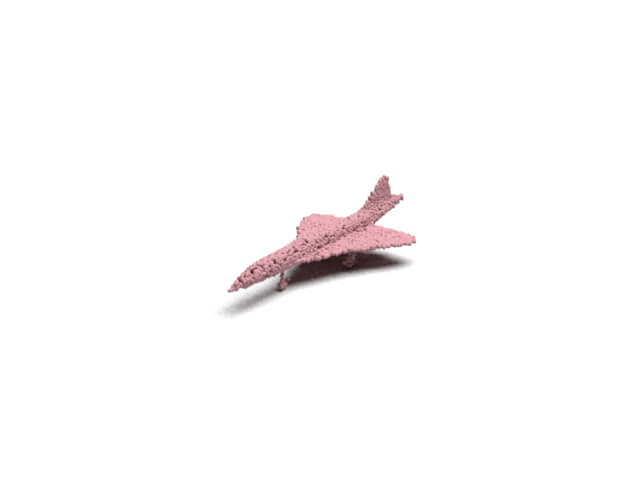}
    \includegraphics[width=\sizea, trim={\tal} {\tab} {\tar} {\tat},clip]{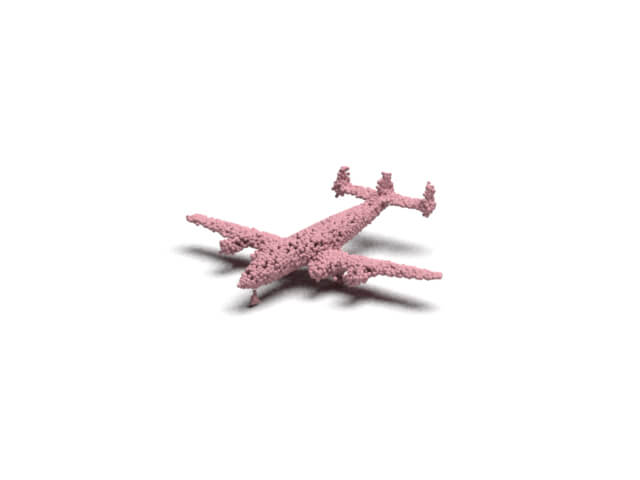}
    \includegraphics[width=\sizea, trim={\tal} {\tab} {\tar} {\tat},clip]{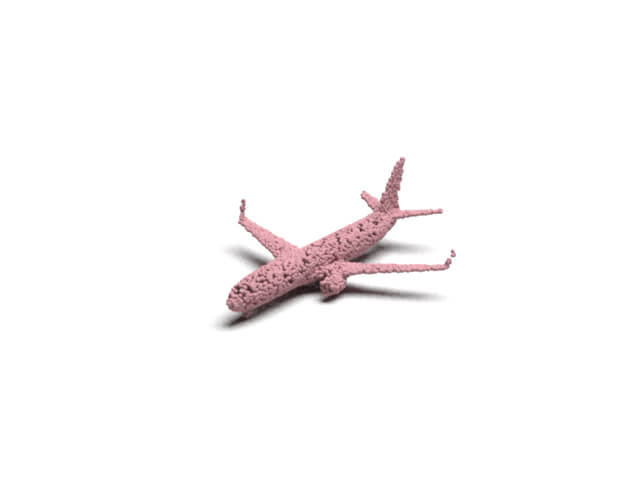}
    \includegraphics[width=\sizea, trim={\tal} {\tab} {\tar} {\tat},clip]{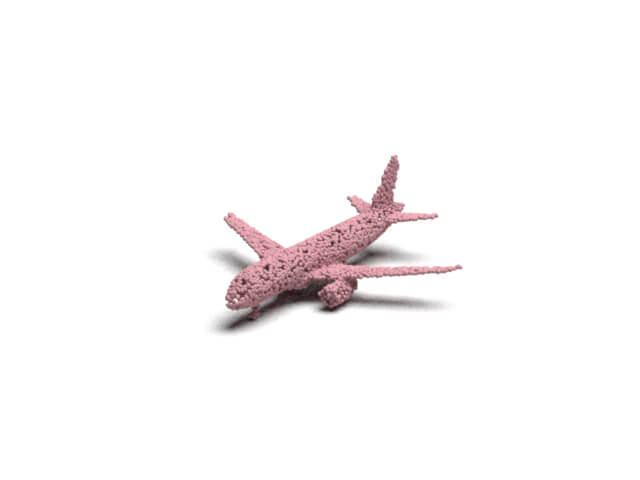}
    \includegraphics[width=\sizea, trim={\tal} {\tab} {\tar} {\tat},clip]{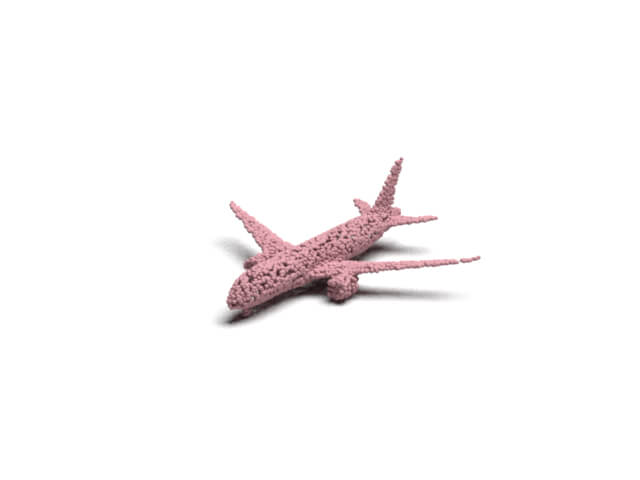}
    \includegraphics[width=\sizea, trim={\tal} {\tab} {\tar} {\tat},clip]{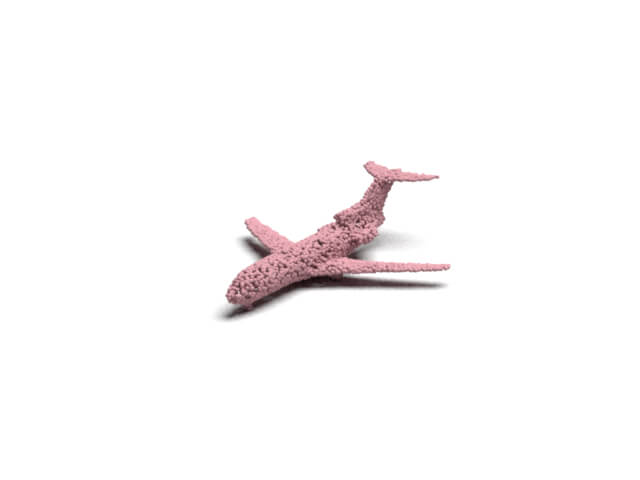}
    \\
    \includegraphics[width=\sizea, trim={\tcl} {\tcb} {\tcr} {\tct},clip]{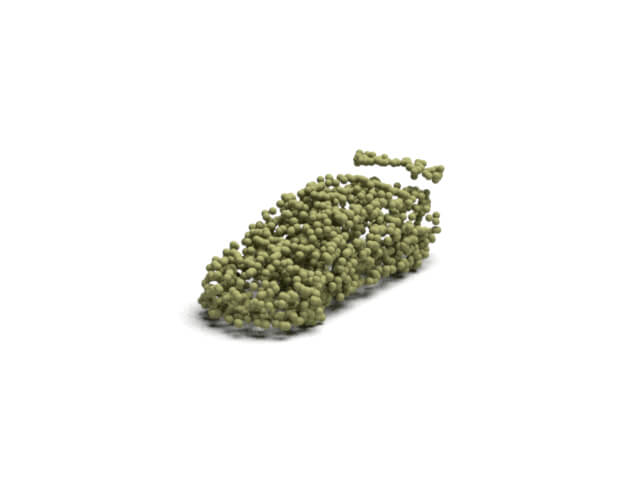} 
    \includegraphics[width=\sizea, trim={\tcl} {\tcb} {\tcr} {\tct},clip]{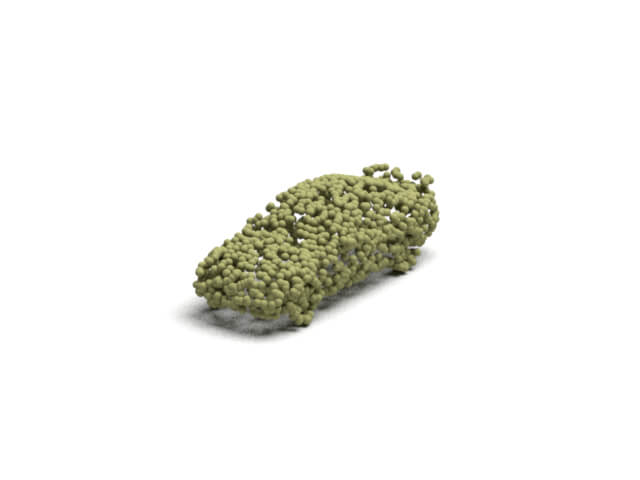} 
    \includegraphics[width=\sizea, trim={\tcl} {\tcb} {\tcr} {\tct},clip]{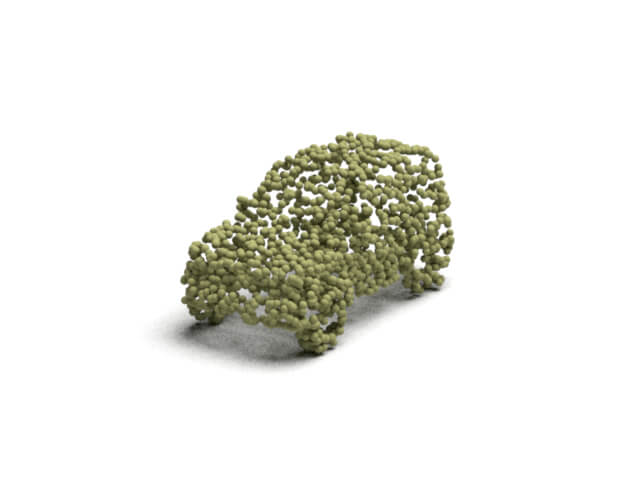} 
    \includegraphics[width=\sizea, trim={\tcl} {\tcb} {\tcr} {\tct},clip]{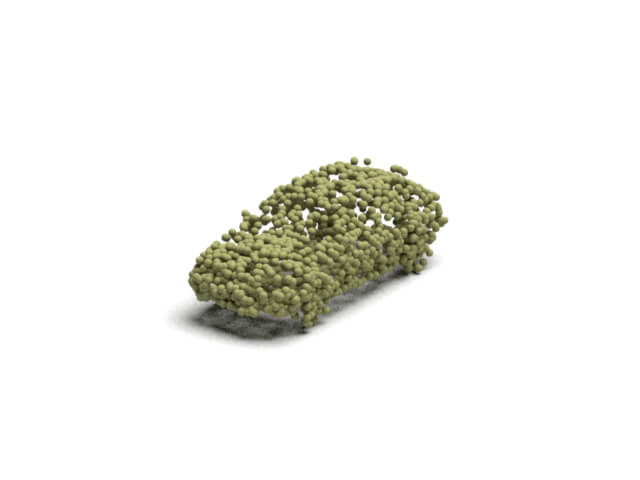}
    \includegraphics[width=\sizea, trim={\tcl} {\tcb} {\tcr} {\tct},clip]{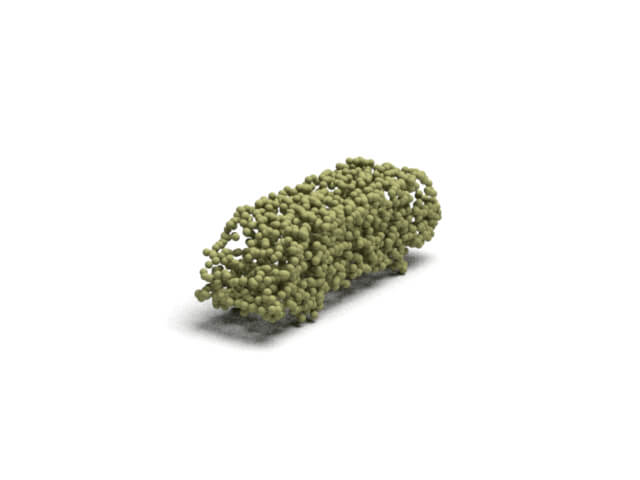} 
    \includegraphics[width=\sizea, trim={\tcl} {\tcb} {\tcr} {\tct},clip]{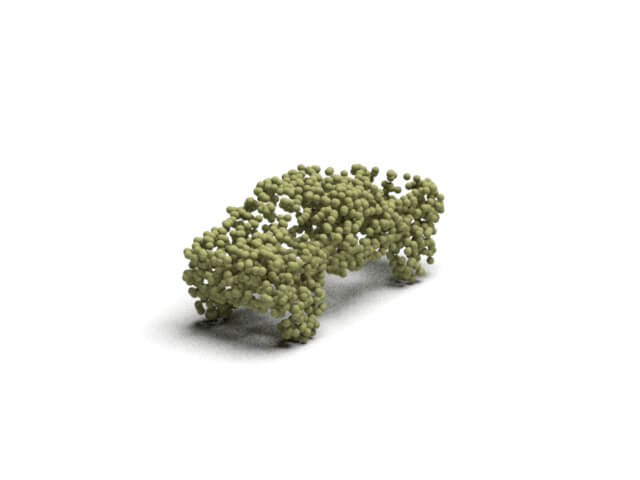} 
     \\
     \includegraphics[width=\sizea, trim={\tcl} {\tcb} {\tcr} {\tct},clip]{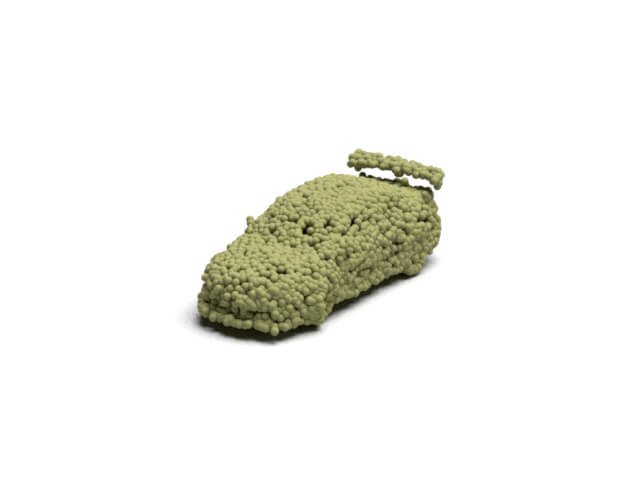} 
    \includegraphics[width=\sizea, trim={\tcl} {\tcb} {\tcr} {\tct},clip]{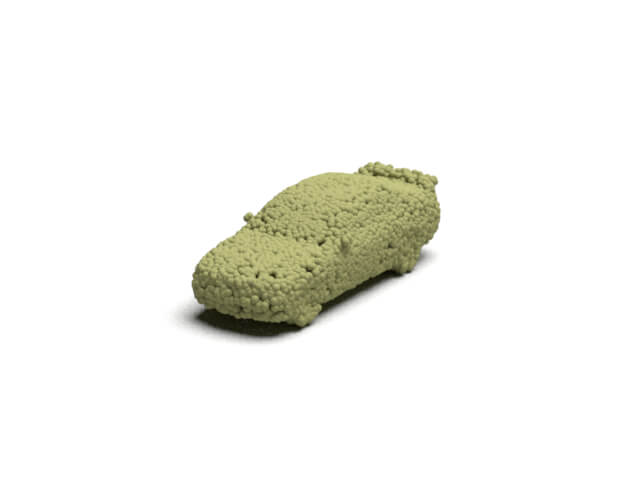} 
    \includegraphics[width=\sizea, trim={\tcl} {\tcb} {\tcr} {\tct},clip]{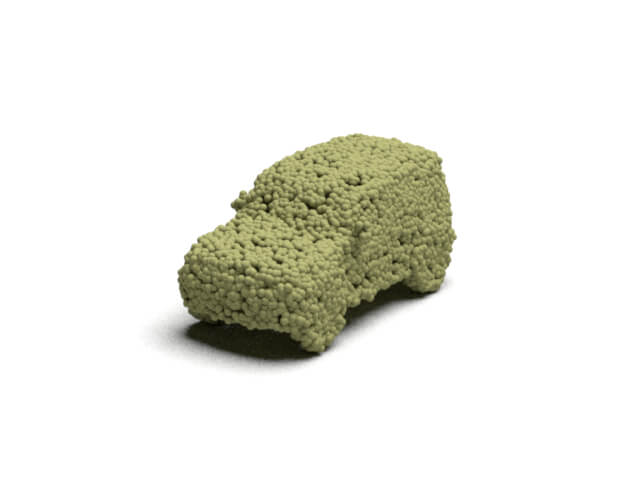} 
    \includegraphics[width=\sizea, trim={\tcl} {\tcb} {\tcr} {\tct},clip]{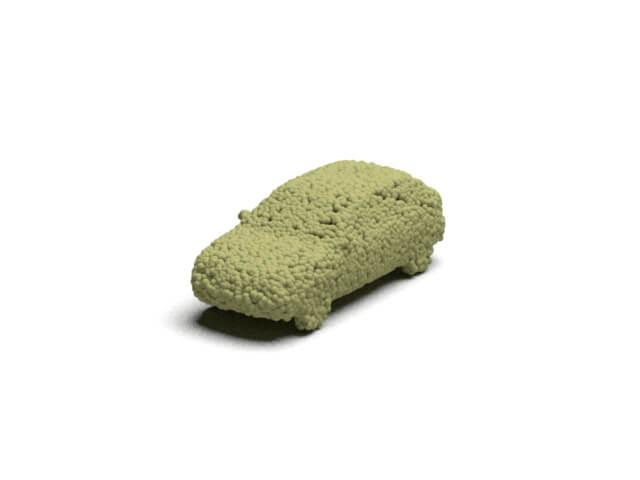}
    \includegraphics[width=\sizea, trim={\tcl} {\tcb} {\tcr} {\tct},clip]{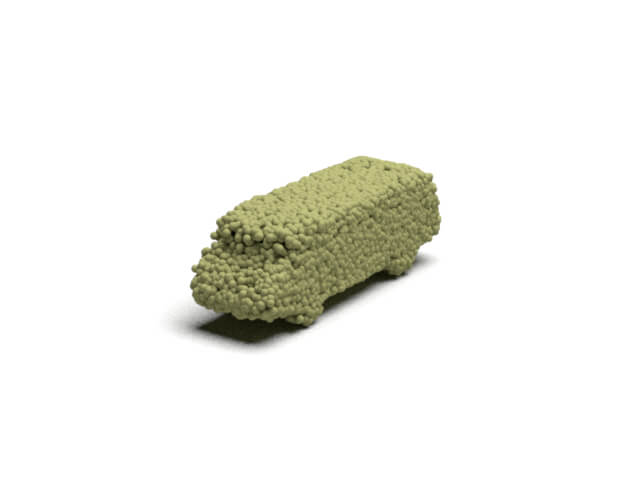} 
    \includegraphics[width=\sizea, trim={\tcl} {\tcb} {\tcr} {\tct},clip]{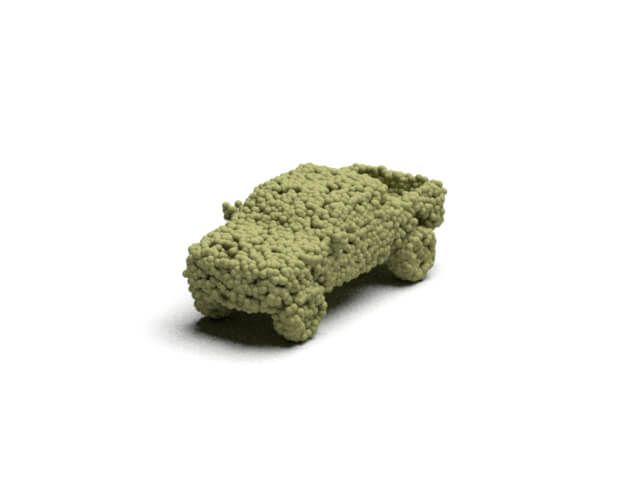}
     \\
    \includegraphics[width=\sizea, trim={\thl} {\thb} {\thr} {\tht},clip]{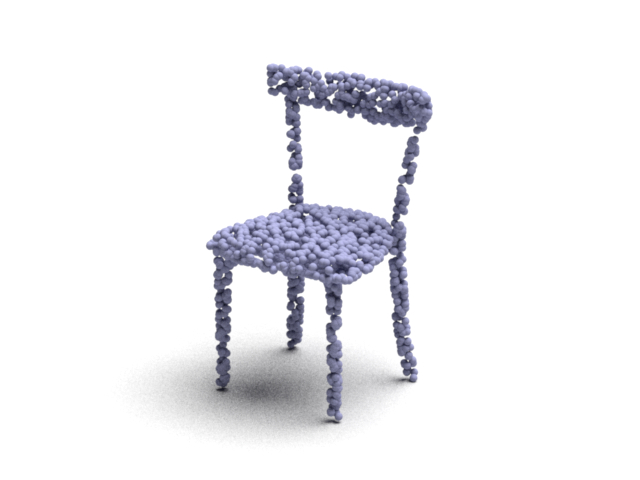}
    \includegraphics[width=\sizea, trim={\thl} {\thb} {\thr} {\tht},clip]{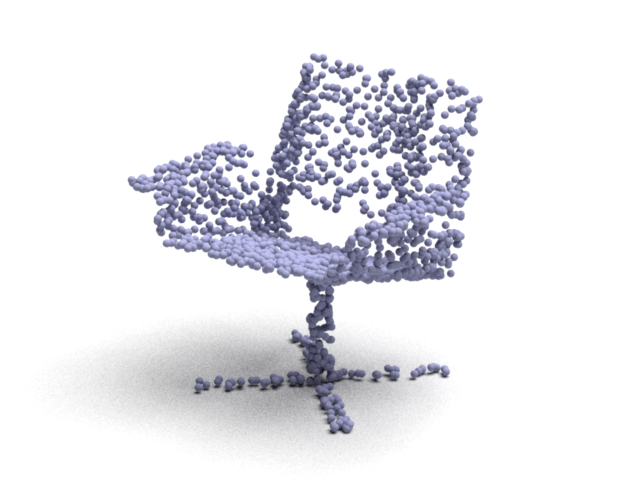}
    \includegraphics[width=\sizea, trim={\thl} {\thb} {\thr} {\tht},clip]{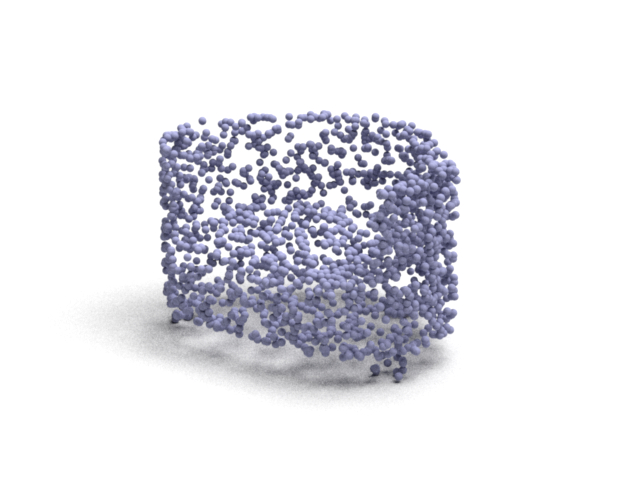}
    \includegraphics[width=\sizea, trim={\thl} {\thb} {\thr} {\tht},clip]{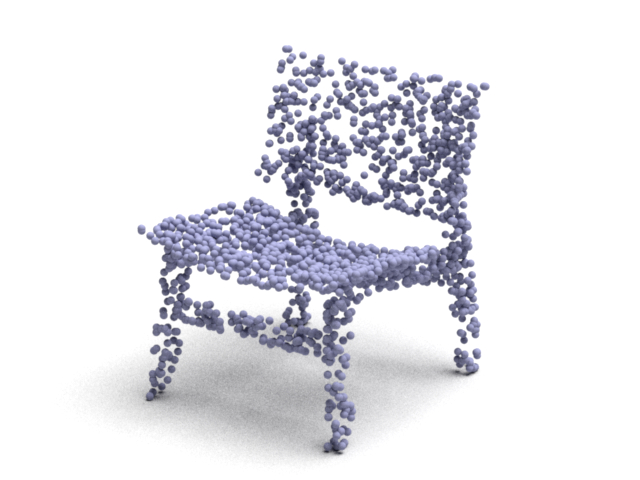}
    \includegraphics[width=\sizea, trim={\thl} {\thb} {\thr} {\tht},clip]{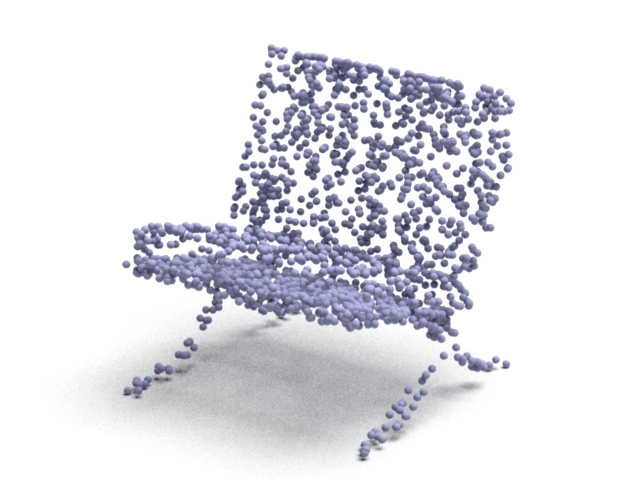}
    \includegraphics[width=\sizea, trim={\thl} {\thb} {\thr} {\tht},clip]{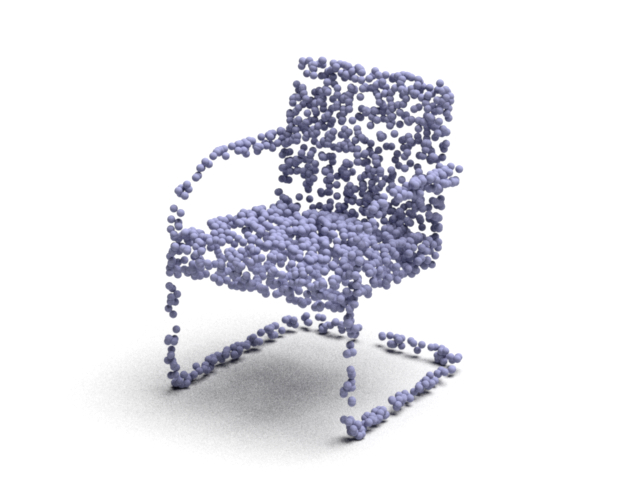}
    \\
    \includegraphics[width=\sizea, trim={\thl} {\thb} {\thr} {\tht},clip]{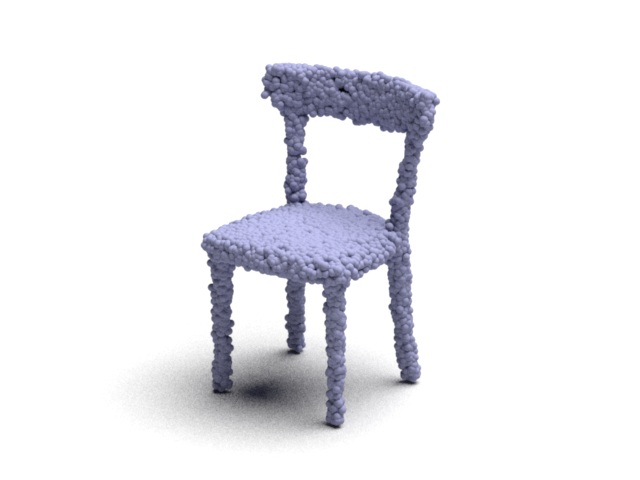}
    \includegraphics[width=\sizea, trim={\thl} {\thb} {\thr} {\tht},clip]{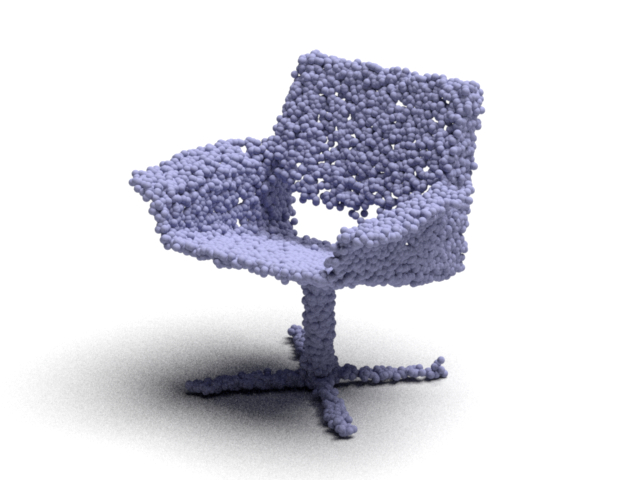}
    \includegraphics[width=\sizea, trim={\thl} {\thb} {\thr} {\tht},clip]{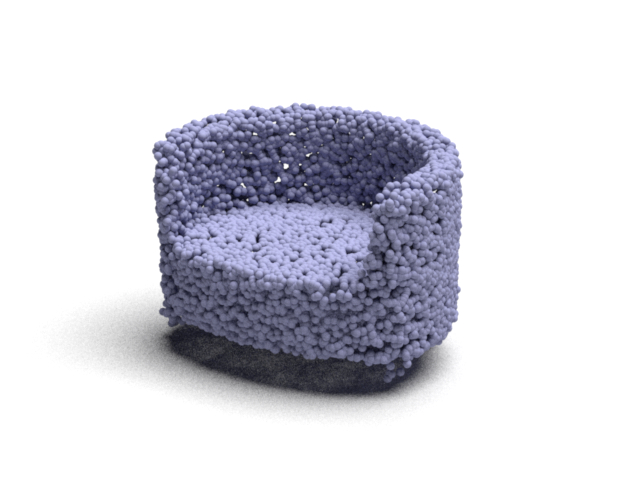}
    \includegraphics[width=\sizea, trim={\thl} {\thb} {\thr} {\tht},clip]{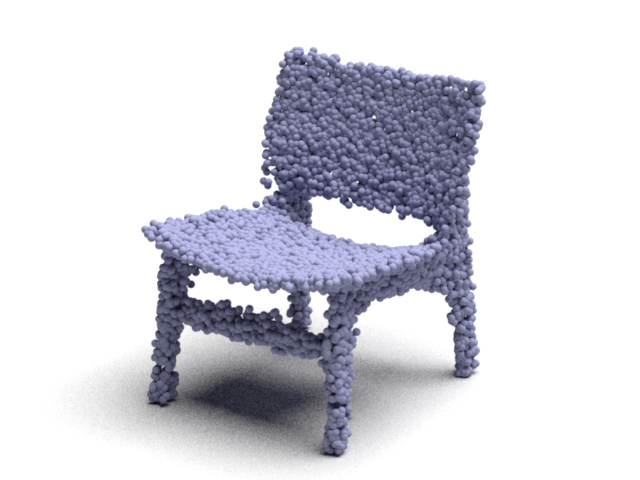}
    \includegraphics[width=\sizea, trim={\thl} {\thb} {\thr} {\tht},clip]{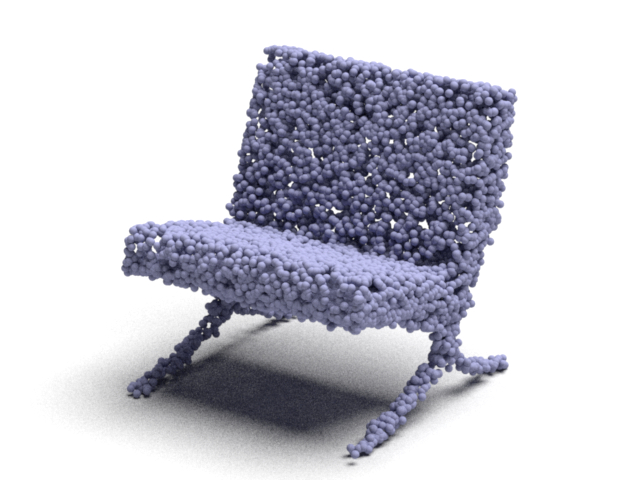}
    \includegraphics[width=\sizea, trim={\thl} {\thb} {\thr} {\tht},clip]{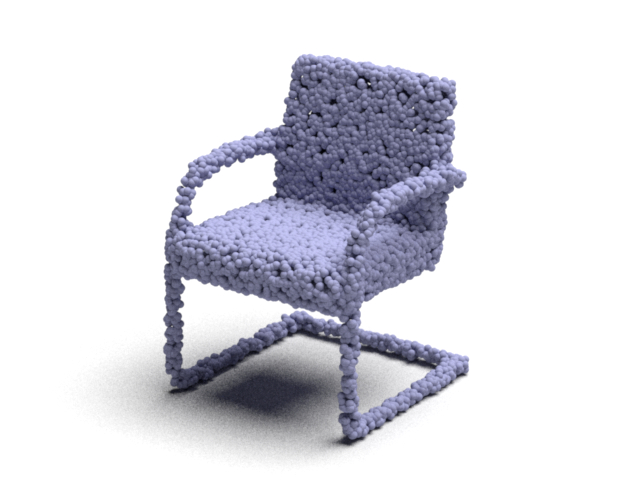}
    \\
    \end{center}
    \caption{Examples of reconstruction results on the Airplane, Car and Chair categories. For each category, the first row is the input point cloud, and the second row is the up-sampling output point cloud.
    }
    \label{fig:extra_recon}
\end{figure}

% <left> <lower> <right> <upper>}
\begin{figure}
    \begin{center}
    \newcommand{\sizea}{0.16\linewidth}
    \newcommand{\tal}{1.0cm}
    \newcommand{\tab}{1.5cm}
    \newcommand{\tar}{1.0cm}
    \newcommand{\tat}{2cm}
    \newcommand{\tcl}{1.0cm}
    \newcommand{\tcb}{1.5cm}
    \newcommand{\tcr}{1.0cm}
    \newcommand{\tct}{2cm}
    \newcommand{\thl}{2.0cm}
    \newcommand{\thb}{0.5cm}
    \newcommand{\thr}{2.0cm}
    \newcommand{\tht}{0.5cm}
    \setlength{\tabcolsep}{0pt}
    \renewcommand{\arraystretch}{0}
    \begin{tabular}{@{}cccccc@{}}
        \includegraphics[width=\sizea, trim={\tal} {\tab} {\tar} {\tat},clip]{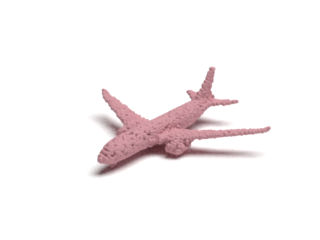}&
        \includegraphics[width=\sizea, trim={\tal} {\tab} {\tar} {\tat},clip]{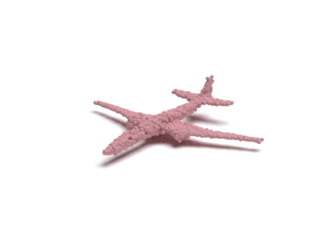}&
        \includegraphics[width=\sizea, trim={\tal} {\tab} {\tar} {\tat},clip]{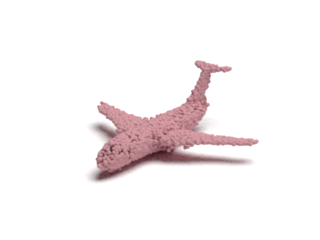}&
        \includegraphics[width=\sizea, trim={\tal} {\tab} {\tar} {\tat},clip]{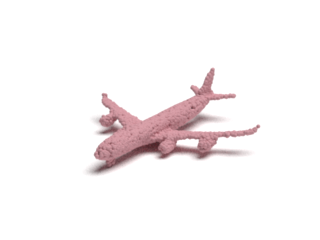}&
        \includegraphics[width=\sizea, trim={\tal} {\tab} {\tar} {\tat},clip]{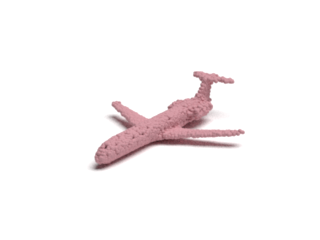}&
        \includegraphics[width=\sizea, trim={\tal} {\tab} {\tar} {\tat},clip]{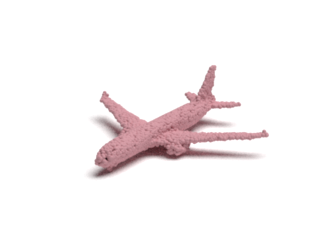}\\
        \includegraphics[width=\sizea, trim={\tal} {\tab} {\tar} {\tat},clip]{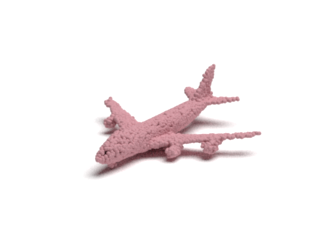}&
        \includegraphics[width=\sizea, trim={\tal} {\tab} {\tar} {\tat},clip]{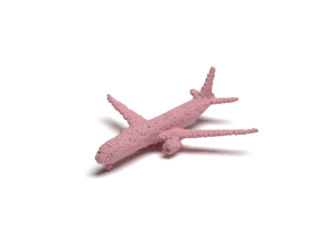}&
        \includegraphics[width=\sizea, trim={\tal} {\tab} {\tar} {\tat},clip]{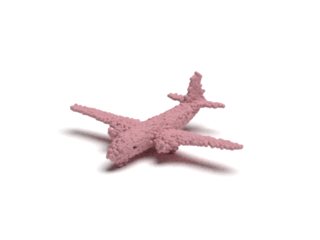}&
        \includegraphics[width=\sizea, trim={\tal} {\tab} {\tar} {\tat},clip]{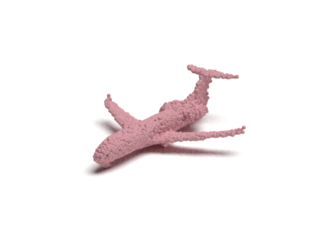}&
        \includegraphics[width=\sizea, trim={\tal} {\tab} {\tar} {\tat},clip]{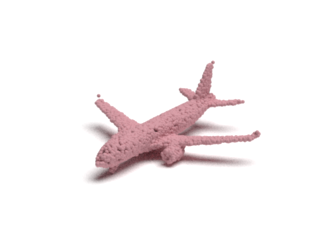}&
        \includegraphics[width=\sizea, trim={\tal} {\tab} {\tar} {\tat},clip]{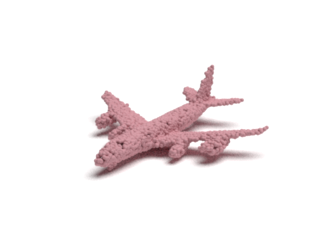}
        \\
        \includegraphics[width=\sizea, trim={\tcl} {\tcb} {\tcr} {\tct},clip]{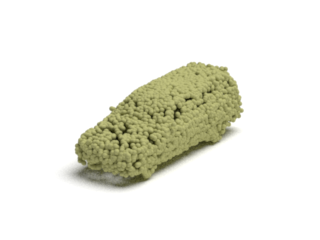} &
        \includegraphics[width=\sizea, trim={\tcl} {\tcb} {\tcr} {\tct},clip]{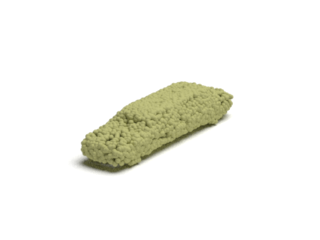} &
         \includegraphics[width=\sizea, trim={\tcl} {\tcb} {\tcr} {\tct},clip]{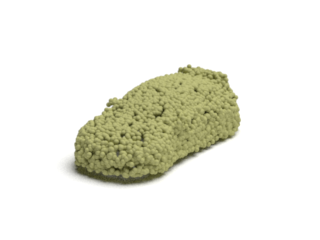}&
         \includegraphics[width=\sizea, trim={\tcl} {\tcb} {\tcr} {\tct},clip]{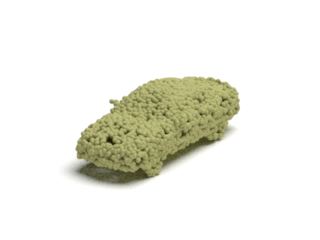}&
         \includegraphics[width=\sizea, trim={\tcl} {\tcb} {\tcr} {\tct},clip]{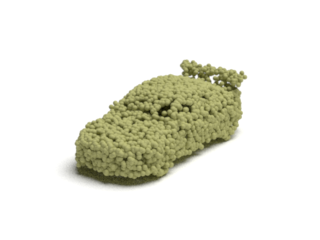}&
         \includegraphics[width=\sizea, trim={\tcl} {\tcb} {\tcr} {\tct},clip]{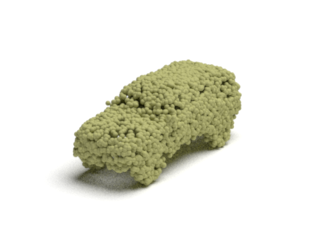}\\
         \includegraphics[width=\sizea, trim={\tcl} {\tcb} {\tcr} {\tct},clip]{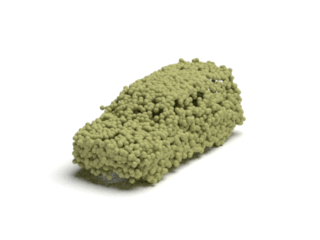}&
         \includegraphics[width=\sizea, trim={\tcl} {\tcb} {\tcr} {\tct},clip]{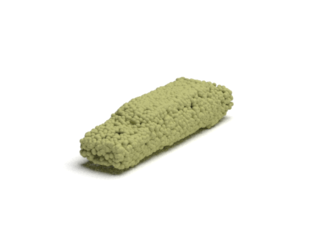}&
         \includegraphics[width=\sizea, trim={\tcl} {\tcb} {\tcr} {\tct},clip]{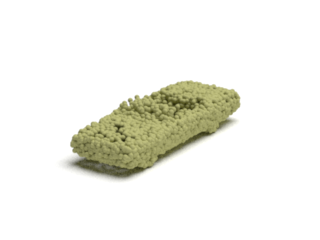}&
         \includegraphics[width=\sizea, trim={\tcl} {\tcb} {\tcr} {\tct},clip]{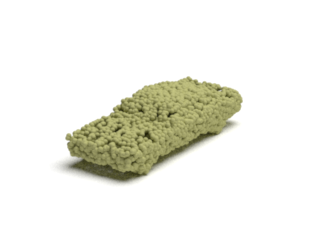}&
         \includegraphics[width=\sizea, trim={\tcl} {\tcb} {\tcr} {\tct},clip]{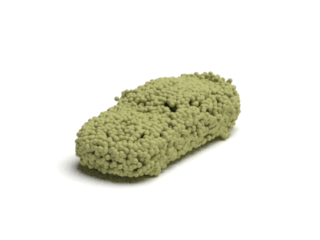}&
         \includegraphics[width=\sizea, trim={\tcl} {\tcb} {\tcr} {\tct},clip]{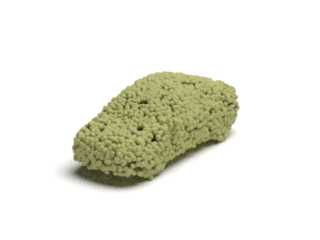}
         \\
        \includegraphics[width=\sizea, trim={\thl} {\thb} {\thr} {\tht},clip]{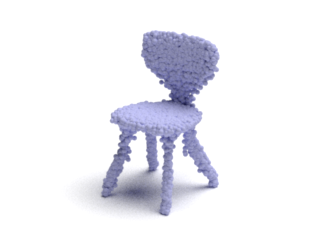}&
        \includegraphics[width=\sizea, trim={\thl} {\thb} {\thr} {\tht},clip]{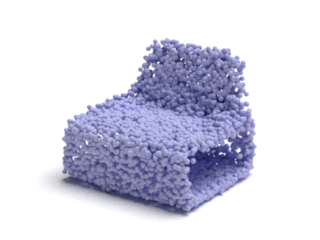}&
         \includegraphics[width=\sizea, trim={\thl} {\thb} {\thr} {\tht},clip]{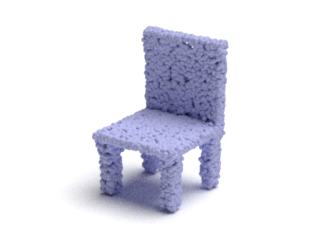}&
         \includegraphics[width=\sizea, trim={\thl} {\thb} {\thr} {\tht},clip]{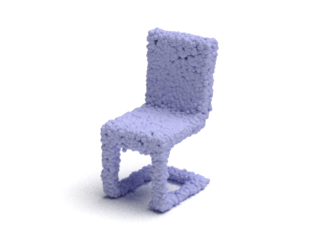}&
         \includegraphics[width=\sizea, trim={\thl} {\thb} {\thr} {\tht},clip]{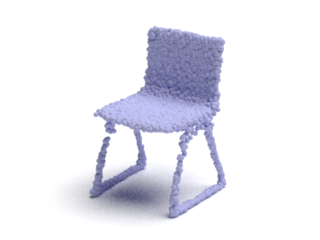}&
         \includegraphics[width=\sizea, trim={\thl} {\thb} {\thr} {\tht},clip]{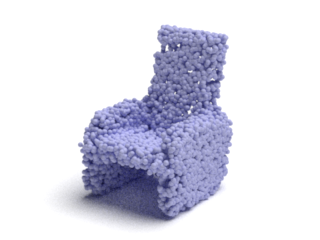}\\
         \includegraphics[width=\sizea, trim={\thl} {\thb} {\thr} {\tht},clip]{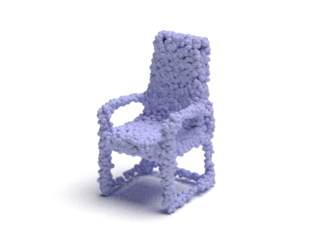}&
         \includegraphics[width=\sizea, trim={\thl} {\thb} {\thr} {\tht},clip]{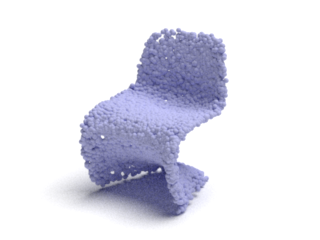}&
         \includegraphics[width=\sizea, trim={\thl} {\thb} {\thr} {\tht},clip]{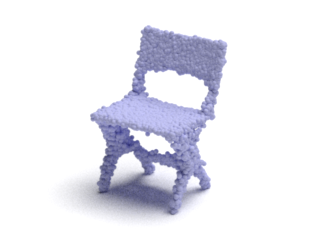}&
         \includegraphics[width=\sizea, trim={\thl} {\thb} {\thr} {\tht},clip]{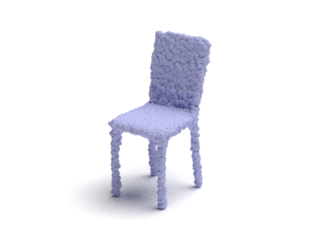}&
         \includegraphics[width=\sizea, trim={\thl} {\thb} {\thr} {\tht},clip]{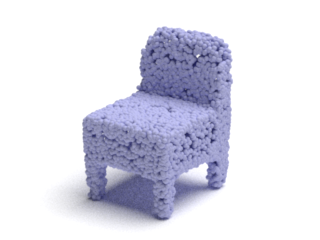}&
         \includegraphics[width=\sizea, trim={\thl} {\thb} {\thr} {\tht},clip]{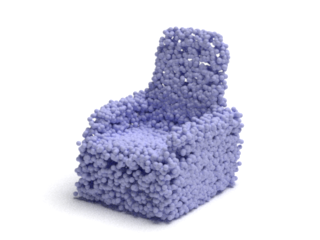}
        \\
    \end{tabular}

    \end{center}
    \caption{Examples of generation results on the Airplane, Car and Chair categories.
    }
    \label{fig:generation_supp}
\end{figure}

% <left> <lower> <right> <upper>}
\begin{figure}
    \begin{center}
    \newcommand{\sizea}{0.16\linewidth}
    \newcommand{\tal}{1.0cm}
    \newcommand{\tab}{1.5cm}
    \newcommand{\tar}{1.0cm}
    \newcommand{\tat}{2cm}
    \newcommand{\tcl}{1.0cm}
    \newcommand{\tcb}{1.5cm}
    \newcommand{\tcr}{1.0cm}
    \newcommand{\tct}{2cm}
    \newcommand{\thl}{2.0cm}
    \newcommand{\thb}{0.5cm}
    \newcommand{\thr}{2.0cm}
    \newcommand{\tht}{0.5cm}
    \setlength{\tabcolsep}{0pt}
    \renewcommand{\arraystretch}{0}
    \begin{tabular}{@{}cccccc@{}}
        \includegraphics[width=\sizea, trim={\tal} {\tab} {\tar} {\tat},clip]{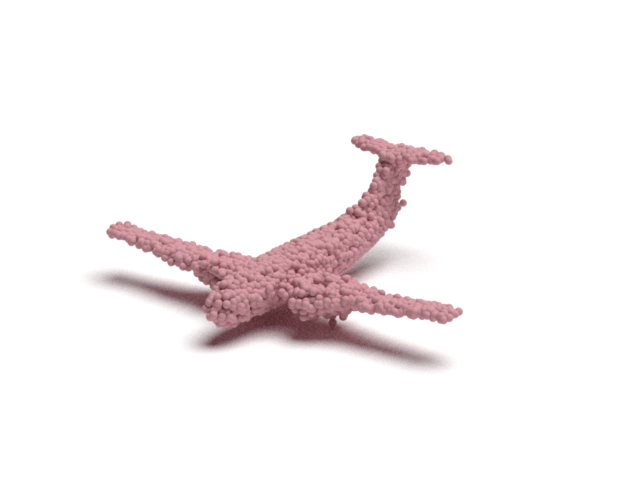}&
        \includegraphics[width=\sizea, trim={\tal} {\tab} {\tar} {\tat},clip]{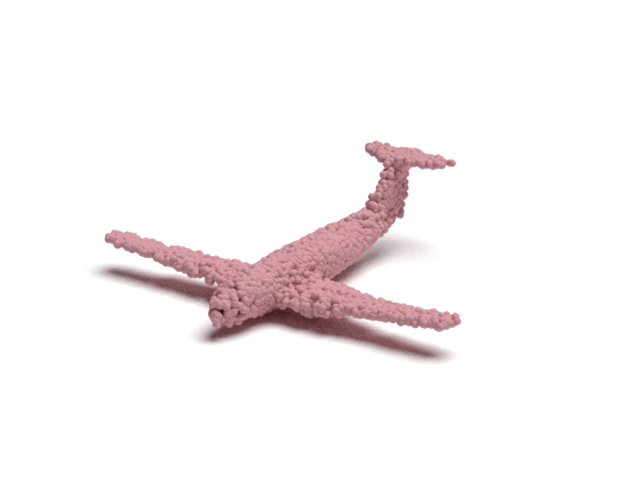}&
        \includegraphics[width=\sizea, trim={\tal} {\tab} {\tar} {\tat},clip]{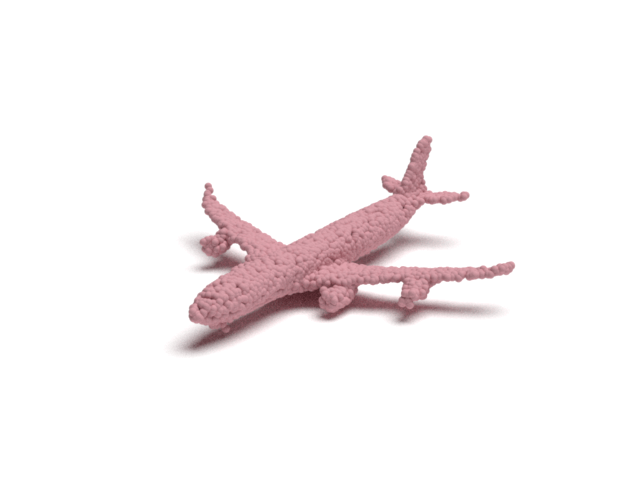}&
        \includegraphics[width=\sizea, trim={\tal} {\tab} {\tar} {\tat},clip]{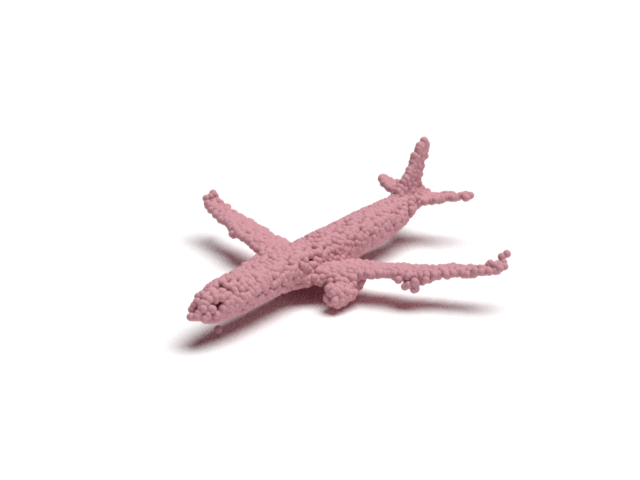}&
        \includegraphics[width=\sizea, trim={\tal} {\tab} {\tar} {\tat},clip]{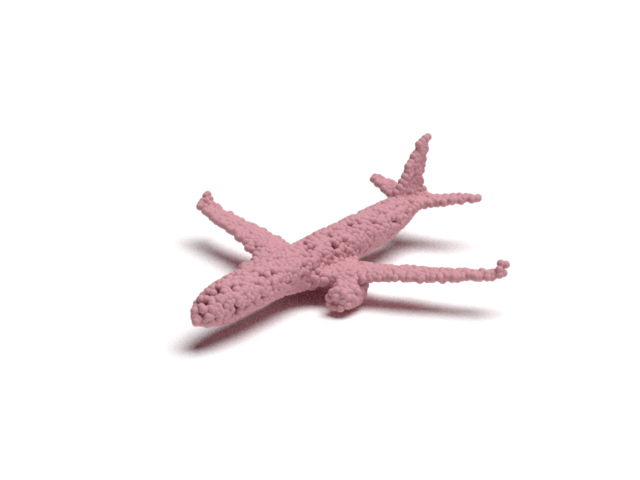}&
        \includegraphics[width=\sizea, trim={\tal} {\tab} {\tar} {\tat},clip]{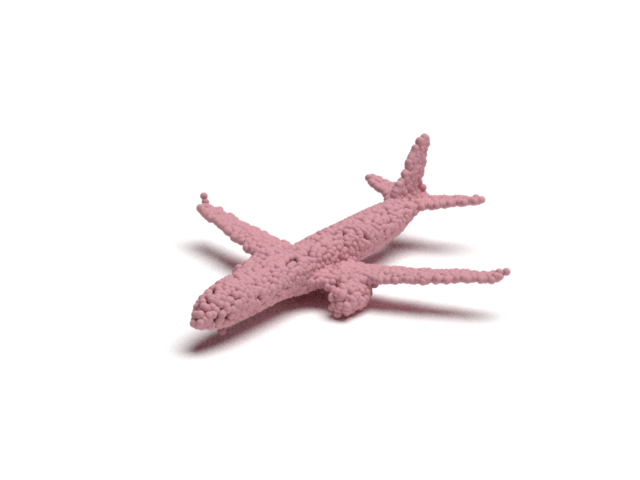}
        \\
        \includegraphics[width=\sizea, trim={\tcl} {\tcb} {\tcr} {\tct},clip]{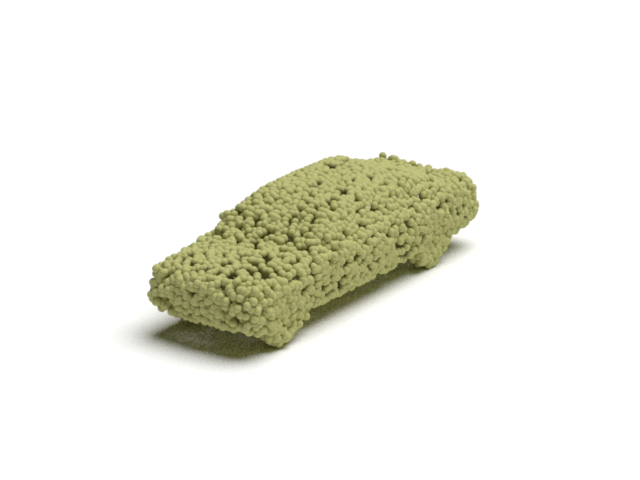} &
        \includegraphics[width=\sizea, trim={\tcl} {\tcb} {\tcr} {\tct},clip]{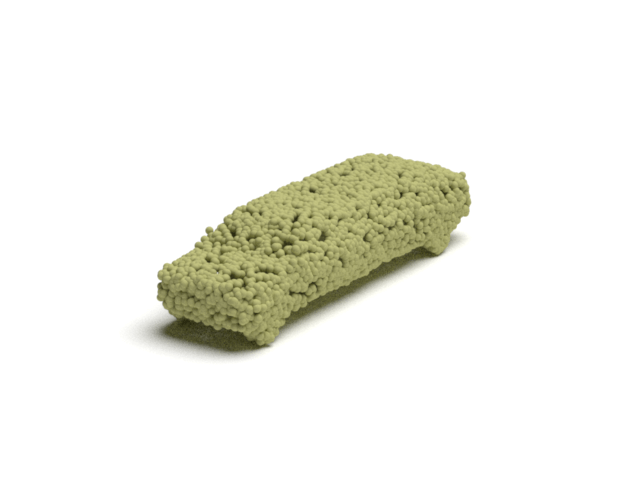} &
        \includegraphics[width=\sizea, trim={\tcl} {\tcb} {\tcr} {\tct},clip]{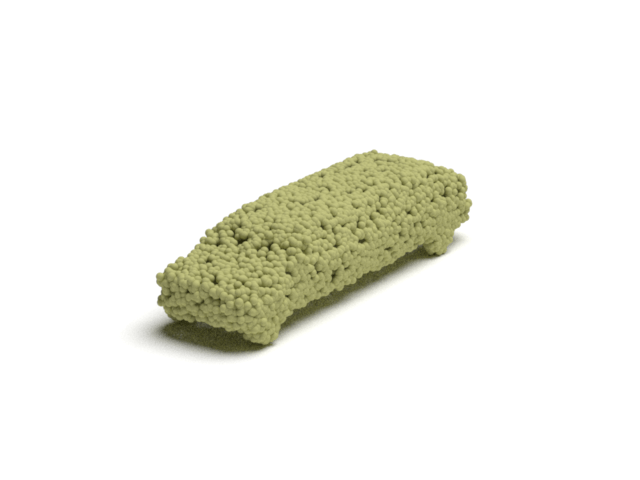} &
        \includegraphics[width=\sizea, trim={\tcl} {\tcb} {\tcr} {\tct},clip]{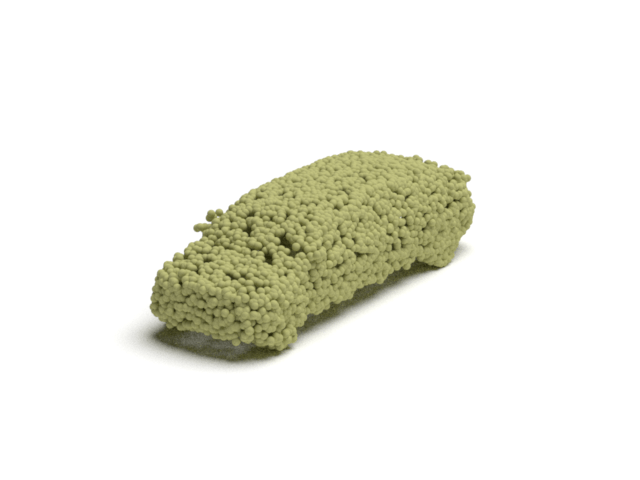} &
        \includegraphics[width=\sizea, trim={\tcl} {\tcb} {\tcr} {\tct},clip]{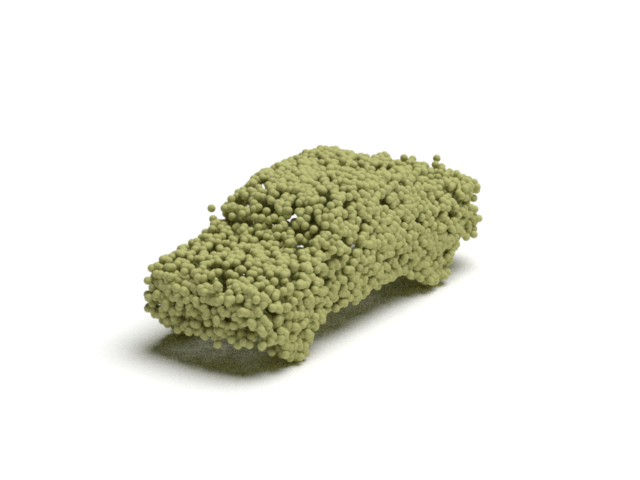} &
        \includegraphics[width=\sizea, trim={\tcl} {\tcb} {\tcr} {\tct},clip]{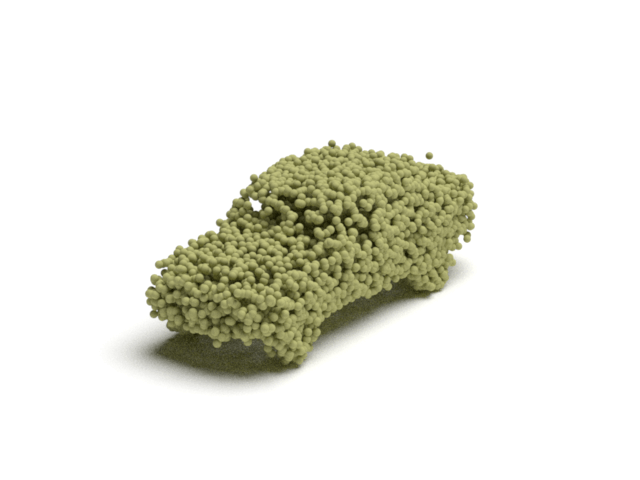} 
         \\
        \includegraphics[width=\sizea, trim={\thl} {\thb} {\thr} {\tht},clip]{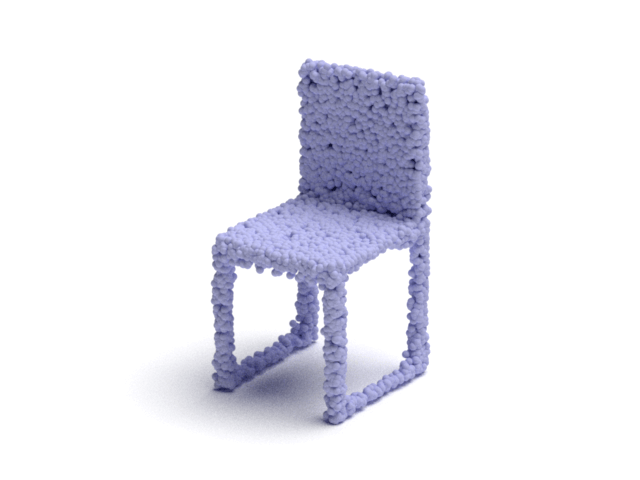}&
        \includegraphics[width=\sizea, trim={\thl} {\thb} {\thr} {\tht},clip]{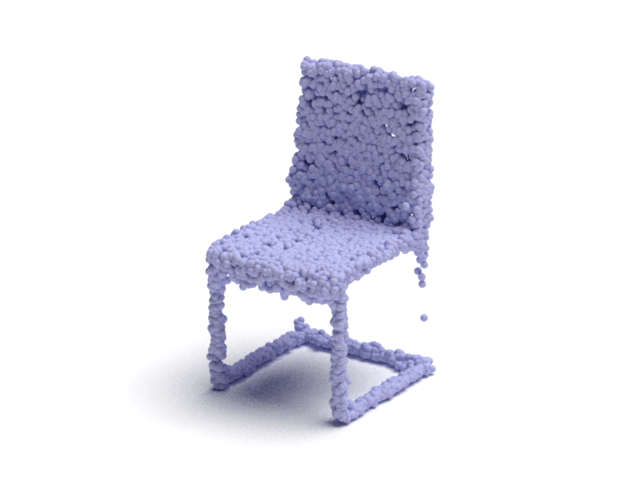}&
        \includegraphics[width=\sizea, trim={\thl} {\thb} {\thr} {\tht},clip]{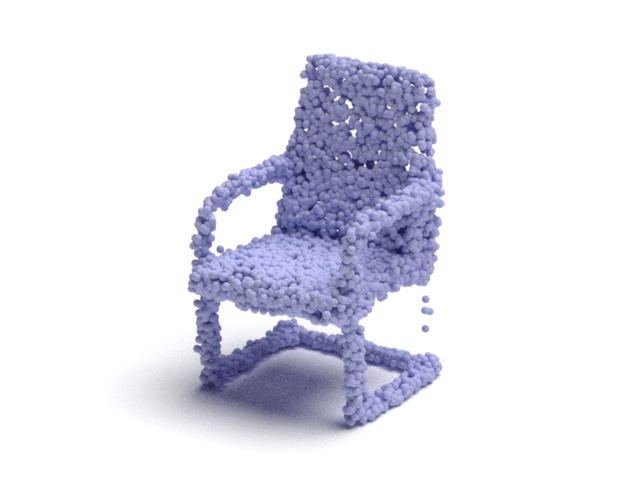}&
        \includegraphics[width=\sizea, trim={\thl} {\thb} {\thr} {\tht},clip]{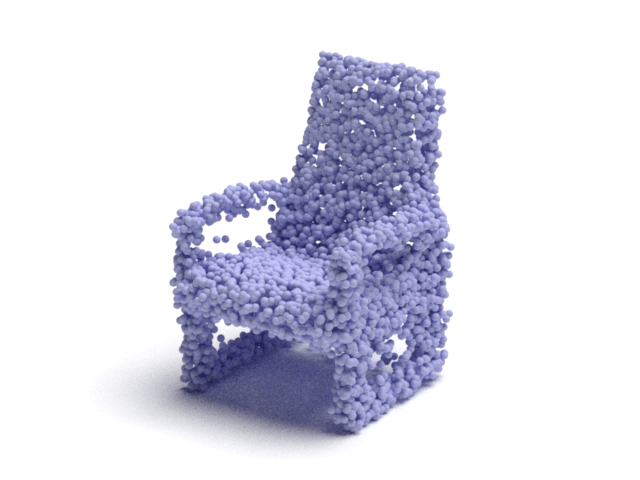}&
        \includegraphics[width=\sizea, trim={\thl} {\thb} {\thr} {\tht},clip]{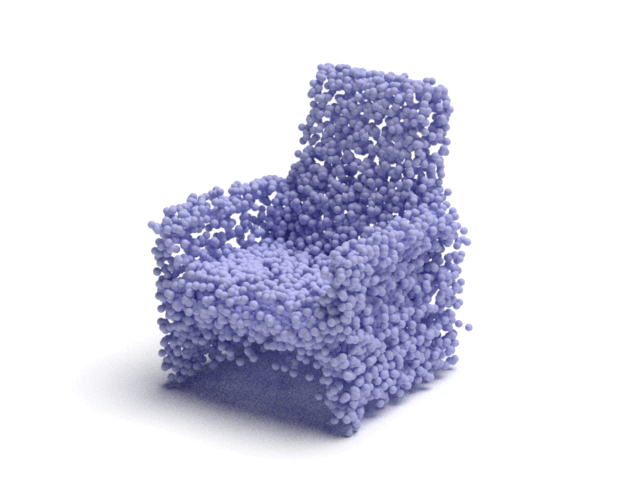}&
        \includegraphics[width=\sizea, trim={\thl} {\thb} {\thr} {\tht},clip]{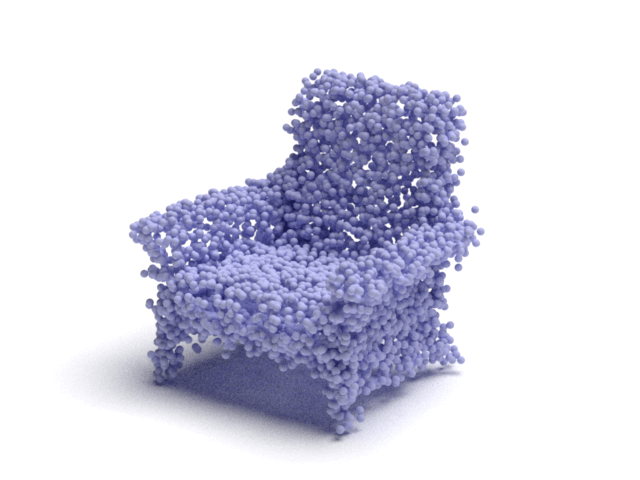}
        \\
        Sampled & Interp & Interp & Interp & Interp & Sampled
        \\
    \end{tabular}

    \end{center}
    \caption{Generation and interpolation results. Generated point clouds (Sampled) and the interpolated results between two generated shapes (Interp) are illustrated.
    }
    \label{fig:gen_interp}
\end{figure}
   % qualitative results
\end{appendix}

\end{document}